\pdfoutput=1
\documentclass[nohyperref]{article}

\usepackage[final]{neurips_2022}

\usepackage{microtype}
\usepackage{graphicx}
\usepackage[scriptsize]{subfigure}
\usepackage{booktabs} %
\usepackage{multirow}
\usepackage{enumitem}
\usepackage[pdftex,bookmarksnumbered,bookmarksopen,colorlinks,citecolor={blue!70!black},linkcolor={blue!70!black},urlcolor=blue]{hyperref}
\usepackage{algorithm}
\usepackage[noend]{algpseudocode}

\usepackage{amsmath}
\usepackage{amssymb,bbm}
\usepackage{mathtools}
\usepackage{amsthm}
\usepackage{nicefrac}
\usepackage{wrapfig}
\usepackage{tablefootnote}

\usepackage[capitalize,noabbrev]{cleveref}

\usepackage{scalefnt}
\usepackage{pgfplots}
\pgfplotsset{compat = newest}

\usepackage{xspace}
\usepackage{xcolor}

\usepackage{setspace}

\newcommand{\highlight}[2][blue!20]{\mathchoice%
  {\colorbox{#1}{$\displaystyle#2$}}%
  {\colorbox{#1}{$\textstyle#2$}}%
  {\colorbox{#1}{$\scriptstyle#2$}}%
  {\colorbox{#1}{$\scriptscriptstyle#2$}}}%
  
\usepackage[most]{tcolorbox}
\tcbset{on line, 
        boxsep=4pt, left=0pt,right=0pt,top=0pt,bottom=0pt,
        colframe=white,colback=blue!20,
        boxsep=3pt,code={\singlespacing},
        highlight math style={enhanced}
        }

\newcommand{\notes}[1]{{ \vspace{3pt} \footnotesize \begin{spacing}{1} #1 \end{spacing}}}

\theoremstyle{plain}
\newtheorem{theorem}{Theorem}[section]
\newtheorem{proposition}[theorem]{Proposition}
\newtheorem{lemma}[theorem]{Lemma}

\theoremstyle{definition}
\newtheorem{definition}[theorem]{Definition}

\theoremstyle{remark}

\DeclareMathOperator*{\E}{\mathbb{E}}
\newcommand{\define}{\smash{\triangleq}}
\newcommand{\newterm}{\textit}

\newcommand{\erm}{SGD\xspace}
\newcommand{\dpsgd}{DP-SGD\xspace}
\newcommand{\dpis}{DP-IS-SGD\xspace}
\newcommand{\iwerm}{IW-SGD\xspace}
\newcommand{\dpiw}{\iwerm}
\newcommand{\iserm}{IS-SGD\xspace}
\newcommand{\gdro}{gDRO\xspace}
\newcommand{\noisygdro}{gDRO-n\xspace}
\newcommand{\noisyiserm}{IS-SGD-n\xspace}
\newcommand{\noisyiwerm}{IW-SGD-n\xspace}

\newcommand{\1}{\mathbbm{1}}
\newcommand{\eps}{\epsilon}
\newcommand{\xy}{z}
\newcommand{\loss}{\ell}
\newcommand{\params}{\theta}
\newcommand{\train}{\mathcal{T}}
\newcommand{\test}{\phi}

\newcommand{\datagen}{\mathcal{D}}

\newcommand{\subtrain}{{\substack{\dataset \sim \datagen^n \\ \xy \sim \dataset}}}
\newcommand{\subtest}{{\substack{\dataset \sim \datagen^n \\ \xy \sim \datagen}}}
\newcommand{\ptrain}{P_1}
\newcommand{\ptest}{P_0}
\newcommand{\dataset}{S}
\newcommand{\model}{\train}

\newcommand{\property}{\pi}

\newcommand{\tv}{d_\mathsf{TV}}
\newcommand{\setgroups}{\mathcal{G}}
\newcommand{\group}{g}

\newcommand{\sX}{\mathbb{X}}
\newcommand{\sY}{\mathbb{Y}}
\newcommand{\sD}{\mathbb{D}}
\newcommand{\sR}{\mathbb{R}}

\newcommand{\vb}{{\vec{b}}}

\newcommand{\cgap}{\textsc{cgap}}
\newcommand{\wggap}{\textsc{wggap}}

\def\shownotes{1}  \ifnum\shownotes=0
\usepackage{todonotes}
\else
\usepackage[disable]{todonotes}
\fi

\newcommand{\bnote}[1]{\todo[linecolor=red,backgroundcolor=red!25,bordercolor=red]{#1 --BK}}
\newcommand{\jnote}[1]{\todo[linecolor=red,backgroundcolor=green!25,bordercolor=red]{#1 --JB}}

\newcommand{\fulltitle}{What You See is What You Get:\\
Principled Deep Learning via Distributional Generalization
}

\title{\fulltitle}
\renewcommand{\cite}{\citep}
\renewcommand{\paragraph}[1]{{\bf #1}}

\author{\hspace{-.6em}Bogdan Kulynych$^{1\star}$ \ Yao-Yuan Yang$^{2\star}$ \ Yaodong Yu$^3$ \ Jarosław Błasiok$^4$ \ Preetum Nakkiran$^2$ \\
$^1$EPFL \ $^2$UC San Diego \ $^3$UC Berkeley \ $^4$Columbia University\\
{\footnotesize$^{\star}$ denotes equal contribution.}
}

\begin{document}
\maketitle
\vspace{-0.2in}

\begin{abstract}
Having similar behavior at training time and test time---what we call a ``What You See Is What You Get'' (WYSIWYG) property---is desirable in machine learning. Models trained with standard stochastic gradient descent (SGD), however, do not necessarily have this property, as their complex behaviors such as robustness or subgroup performance can differ drastically between training and test time. In contrast, we show that Differentially-Private (DP) training provably ensures the high-level WYSIWYG property, which we quantify using a notion of distributional generalization. Applying this connection, we introduce new conceptual tools for designing deep-learning methods by reducing generalization concerns to optimization ones: to mitigate unwanted behavior at test time, it is provably sufficient to mitigate this behavior on the training data. By applying this novel design principle, which bypasses ``pathologies'' of SGD, we construct simple algorithms that are competitive with SOTA in several distributional-robustness applications, significantly improve the privacy vs. disparate impact trade-off of DP-SGD, and mitigate robust overfitting in adversarial training. Finally, we also improve on theoretical bounds relating DP, stability, and distributional generalization.
\end{abstract}

\section{Introduction}\label{sec:intro}
Much of machine learning (ML), both in theory and in practice,
operates under two assumptions. First, we have independent and identically distributed (i.i.d.) samples. Second, we care only about
a single averaged scalar metric (error, loss, risk).
Under these assumptions, we have mature methods and theory:
Modern learning methods excel
when trained on i.i.d. data to directly optimize a scalar loss, and there are
many theoretical tools for reasoning about \emph{generalization},
which explain when does optimization of a scalar on the training data translates to
similar values of this scalar at test time.

\begin{figure}[t!]
    \centering
    \includegraphics[width=.6\textwidth]{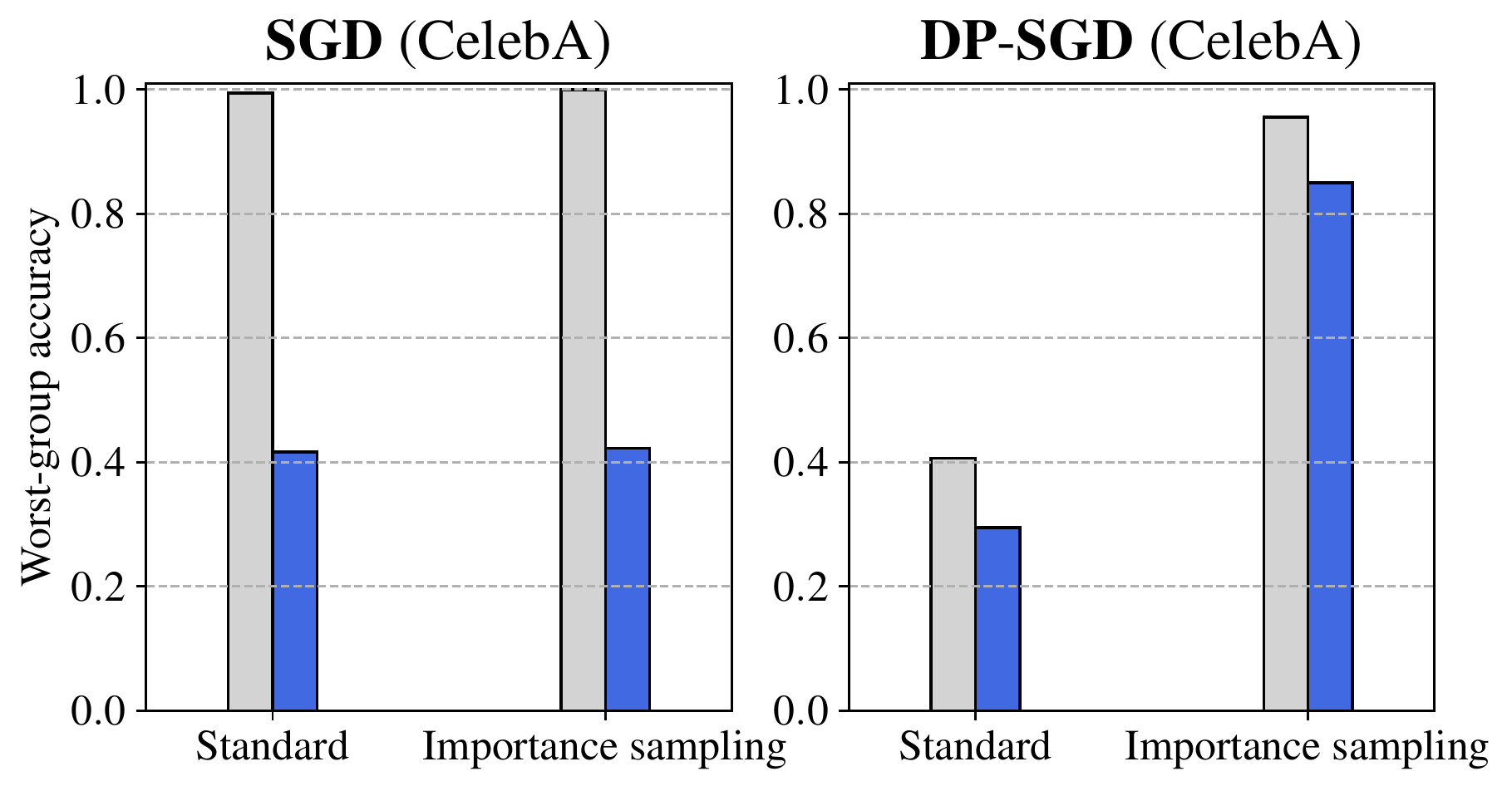}
    \vspace{-0.1in}
    \caption{
    {\bf Differential privacy ensures the desired behavior of importance sampling on test data.}
    The train and test accuracy of ResNets
    on CelebA, evaluated on
    the worst-performing (``male, blond'') subgroup.
    \emph{Left:}
    Standard SGD
    has a large generalization gap on this subgroup, and Importance Sampling (IS)
    has little effect.
    \emph{Right:}
    DP-SGD provably has small generalization gap on all subgroups,
    and IS improves subgroup performance as intended. See \cref{sec:experiments} for details.
    }
    \label{fig:CelebA}
    \vspace{-0.1in}
\end{figure}

The focus on scalar metrics such as average error, however, misses many
theoretically, practically, and socially relevant aspects of model performance.
For example, models with small \emph{average} error often have high error on salient minority subgroups~\citep{buolamwini2018gender, koenecke2020racial}.
In general, ML models are applied to the heterogeneous and long-tailed data distributions of the real world~\citep{van2017devil}.
Attempting to summarize their complex behavior with only 
a single scalar misses many rich and important aspects of learning.

These issues are compounded for modern overparameterized networks,
as their nuanced test-time behavior is not reflected at training time.
For example, consider the setting of \newterm{importance sampling}:
suppose we know that a certain subgroup of inputs is underrepresented in the training data
compared to the test distribution (breaking the i.i.d. assumption).
For underparameterized models, we can upsample this underrepresented group
to account for the distribution shift~\citep[see, e.g.,][]{gretton2009covariate}.
This approach, however, is known to empirically fail for overparameterized models~\citep{byrd2019effect}.
Because ``what you see'' (on the training data) is not ``what you get'' (at test time),
we cannot make principled train-time interventions to affect test-time behaviors.
This issue extends beyond importance sampling. 
For example, theoretically principled methods for
distributionally robust optimization (e.g. \citet{namkoong2016})
fail for overparameterized deep networks,
and require ad-hoc modifications %
~\citep{sagawa2019distributionally}.

We develop a theoretical framework which
sheds light on these existing issues,
and leads to improved practical methods in
privacy, fairness, and distributional robustness.
The core object in our framework is 
what we call the ``What You See Is What You Get'' (WYSIWYG) property.
A training procedure with the WYSIWYG property does \emph{not}
exhibit the ``pathologies'' of standard stochastic gradient descent (SGD):
all test-time behaviors will be expressed on the training data as well,
and there will be ``no surprises'' in generalization.

\paragraph{What You See Is What You Get (WYSIWYG) as a Design Principle.}
The WYSIWYG property is desirable for two reasons.
The first is diagnostic: as there are ``no surprises'' at test time, 
all properties of a model at test time
are already evident at the training stage.
It cannot be the case, for example, that a WYSIWYG model has
small disparate impact on the training data, but large disparate impact at test time.
The second reason is algorithmic:
to mitigate \emph{any} unwanted test-time behavior, it is sufficient to mitigate
this behavior on the training data.
This means that algorithm designers can be concerned only with achieving
desirable behavior at train time, as the WYSIWYG property guarantees it holds at test time too.
In practice, this enables the usage of many theoretically principled algorithms
which were developed in the underparameterized 
regime
to also apply in the modern overparameterized (deep learning) setting.
For example, we find that interventions such as importance sampling, or algorithms for distributionally robust optimization,
which fail without additional regularization, work exactly as intended with WYSIWYG (See \cref{fig:CelebA} for an illustration).

As WYSIWYG is a high-level conceptual property, we have to formalize it to use in computational practice. We do so using the notion of \newterm{Distributional Generalization} (DG), as introduced by \citet{nakkiran2020distributional,kulynych2022disparate}. If classical generalization ensures that the values of the model's loss on the training dataset and at test time are close on average~\cite{shalev2010learnability}, distributional generalization ensures that values of any other bounded test function---not only loss---are close on training and test time. We say that a model which satisfies an appropriately high level of distributional generalization exhibits the WYSIWYG property.

\paragraph{Achieving DG in Practice.}
Our key observation is that distributional generalization
is formally implied by
\emph{differential privacy} (DP)~\cite{dwork2006calibrating,dwork2014algorithmic}).
The spirit of this observation is not novel: 
DP training is known to satisfy much stronger notions of generalization (e.g., \emph{robust generalization}, see \cref{sec:related} for more details),
and stability than standard SGD \citep{dwork2015generalization, cummings2016adaptive, bassily2016algorithmic, steinke2020reasoning}.
We show that a similar connection holds for the notion of distributional generalization,
and prove (and improve) tight bounds relating DP, stability, and DG.
This guarantees the WYSIWYG property
for any method that is differentially-private,
including DP-SGD on deep neural networks \cite{abadi2016deep}.

We demonstrate how DG can be a useful design principle in three concrete settings.
First, we show that we can mitigate disparate impact of DP training~\cite{bagdasaryan2019differential, pujol2020fair} by leveraging importance sampling. Second, we study the setting of distributionally robust optimization~\cite[e.g.,][]{sagawa2019distributionally, hu2018does}.
We show how ideas from DP can be used to construct heuristic optimizers, which do not formally satisfy DP, yet empirically exhibit DG. 
Our heuristics lead to competitive results with SOTA algorithms in five datasets in the distributional robustness setting. Third, we show that the heuristic optimizer is also capable of reducing overfitting of adversarial loss in adversarial training~\cite{madry2017towards, zhang2019theoretically, rice2020overfitting}.

\paragraph{Our Contributions.}
We develop the theoretical connection between Differential Privacy (DP) and Distributional Generalization (DG),
and we leverage our theory to improve empirical performance
in privacy, fairness, and robustness applications.
Theoretically (\cref{sec:dg-theory,sec:applications,sec:algorithms}):
\begin{enumerate}[itemsep=2pt,parsep=2pt,topsep=0pt,partopsep=0pt]
    \item We provide tighter bounds than previously reported connecting DP and strong forms of generalization, and show that DP training methods satisfy DG, thus the WYSIWYG property. 
    \item We introduce \dpis, an importance-sampling version of \dpsgd,
    and show it satisfies DP and DG.
\end{enumerate}

Experimentally (\cref{sec:experiments}):
\begin{enumerate}[itemsep=2pt,parsep=2pt,topsep=0pt,partopsep=0pt]
    \item
    We use our framework to shed light on
    \emph{disparate impact}:
    The disparity in accuracy across groups
    at test time
    is provably reflected by the accuracy disparity \emph{on the train dataset}.
    \item We use our \dpis algorithm to largely mitigate the 
    disparate impact of DP using importance sampling.
    \item Based on our theoretical intuitions,
    we propose a DP-inspired heuristic: addition of gradient noise.
    We find this empirically achieves
    competitive and even improved results in several DRO settings, and reduces overfitting of adversarial loss in adversarial training.
\end{enumerate}

Taken together, our results emphasize the central role of the WYSIWYG property in
designing machine learning algorithms which avoid the ``pathologies''
of standard SGD. We also establish DP as a useful tool
for achieving WYSIWYG, thus extend its applications further beyond privacy.

\section{Theory of ``What You See is What You Get'' Generalization}\label{sec:dg-theory}

We first review the notion of distributional generalization and demonstrate why it captures the WYSIWYG property. 
Second, we show that strong stability notions imply distributional generalization. 
Finally, we improve on the known stability guarantees of differential privacy. 
As a result, we extend the connections between differential privacy, stability, and generalization to \emph{distributional} generalization, showing that stability and privacy imply the WYSIWYG property.

\paragraph{Notation.} We consider a learning task with a set of examples $\sX$ and labels $\sY$. We assume that a \newterm{source distribution} of labeled examples $\xy~\define~(x, y) \sim \datagen$ is defined over $\sD = \sX \times \sY$. Given an i.i.d.-sampled \newterm{dataset} $\dataset \sim \datagen^n$ of size $n$, we use a randomized \newterm{training algorithm} $\train(\dataset)$ that outputs a model's parameter vector $\params$ from the set $\Theta$. We denote by $f_\params(x)$ the resulting prediction function.

\subsection{Distributional Generalization and WYSIWYG}\label{sec:wysiwyg}
If on-average generalization~\cite{shalev2010learnability}
guarantees closeness only of loss values on train and test data, distributional generalization (DG) also guarantees closeness of values of all test functions $\test(\xy; \params) \in [0, 1]$ beyond only loss:
\begin{definition}[Based on \citet{nakkiran2020distributional}]
\label{def:dg}
An algorithm $\train(\dataset)$ satisfies $\delta$-distributional generalization if for all $\test: \sD \times \Theta \rightarrow [0, 1]$,
\begin{equation}\label{eq:dg-variational-defn}
    \Big|\E_\subtrain \test\big(\xy; \train(\dataset)\big) -
    \E_\subtest \test\big(\xy; \train(\dataset)\big)\Big| \leq \delta.
\end{equation}
By the variational characterization of the total-variation (TV) distance~\cite[see, e.g., ][Chapter 6.3]{polyanskiy2014lecture}, \cref{eq:dg-variational-defn} is equivalent to the bound
$
    \tv( \ptrain, \ptest) \leq \delta,
$
where $\ptrain$ and $\ptest$ are both distributions of $\big(\xy, \train(\dataset)\big)$ over the randomness of $\dataset \sim \datagen^n$ and the training algorithm $\train(\cdot)$, with the difference that $\xy \sim \dataset$ in the case of $\ptrain$ (train), and $\xy \sim \datagen$ in the case of $\ptest$ (test).
\end{definition}

It might seem that DG only ensures average closeness of bounded tests on train and test data. This is not, however, the full picture. Consider generalization in terms of a broader class of functions:
\begin{definition}[\citet{kulynych2022disparate}]
\label{def:dg-prop}
An algorithm $\train(\dataset)$ satisfies $(\delta, \property)$-distributional generalization if for a given property function $\property: \sD \times \Theta \rightarrow \sR^k$ it holds that $\tv( \property _\sharp \ptrain, \property _\sharp \ptest) \leq \delta,$
where $\property _\sharp P$ is the distribution of $\property(T)$ for $T \sim P$.
\end{definition}
Because TV distance is preserved under post-processing, we can see that $\delta$-distributional generalization implies $(\delta, \property)$-distributional generalization for \emph{all} property functions.
Informally, $\delta$-DG means that for \emph{all} numeric property functions $\property(\xy; \params)$ of a model,
the distributions of the property values are close on the train and test data, on average.
This fact captures the high-level idea of the \newterm{``What You See is What You Get''} (WYSIWYG) guarantee. 
Some example property functions:
\vspace{-.3em}
\begin{itemize}
    \setlength\itemsep{.1em}
    \item \textit{Subgroup loss:} $\property(\xy; \params) = \1\{\xy \in G\} \cdot \loss(\xy; \params)$, for some subgroup $G \subset \sD$.
    \item \textit{Counterfactual fairness:} $\property((x, y); \params) = f_\params(x') - f_\params(x)$, where $x'$ is a counterfactual version of $x$ had it had a different value of a sensitive attribute~\cite{kusner2017counterfactual}.
    \item \textit{Robustness to corruptions:} $\property(\xy; \params) = \loss(A(\xy); \params)$, where $A(x)$ is a possibly randomized transformation that distorts the example, e.g., by adding Gaussian noise.
    \item \textit{Adversarial robustness:} $\property(\xy; \params) = \loss(A_\params(\xy); \params)$, where $A_\params(\xy)$ is an adversarial example, e.g. generated using the PGD attack~\cite{madry2017towards}.

\end{itemize}
In the next sections, we show how a training algorithm can provably satisfy DG and therefore provide
WYSIWYG guarantees for all properties, including the ones above.

\subsection{Distributional Generalization from Stability and Differential Privacy}\label{sec:from-stability-to-dg}
The connections between privacy, stability, and generalization are well-known. In particular, stability of the learning algorithm---its non-sensitivity to limited changes in the training data---implies generalization~\cite{bousquet2002stability, shalev2010learnability}. In turn, differential privacy implies strong forms of stability, thus ensuring generalization through the chain Privacy $\Rightarrow$ Stability $\Rightarrow$ Generalization~\cite{raskhodnikova2008can, dwork2015preserving, dwork2015generalization, wang2016learning}.

Let us formally define differential privacy:
\begin{definition}[Differential Privacy~\cite{dwork2006calibrating, dwork2014algorithmic}]\label{def-dp} An algorithm $\train(\dataset)$ is $(\epsilon, \delta)$-differentially private (DP) if for any two \newterm{neighbouring datasets}---differing by one example---$\dataset$, $\dataset'$ of size $n$, for any subset $K \subseteq \Theta$ it holds that
$
    \Pr[ \train(\dataset) \in K ] \leq \exp(\epsilon) \Pr[ \train(\dataset') \in K] + \delta.
$
\end{definition}

DP mathematically encodes a notion of plausible deniability of the inclusion of an example in the dataset. However, it can also be thought as a strong form of stability~\cite{dwork2015preserving}. As such, DP implies other notions of stability. 
We consider the following notion, which has been studied in the literature under multiple names.
In the context of privacy, it is equivalent to $(0, \delta)$-differential privacy, and has been called additive differential privacy~\cite{geng2019optimal}, and total-variation privacy~\cite{barber2014privacy}. In the context of learning, it has been called total-variation (TV) stability~\cite{bassily2016algorithmic}. We take this last approach and refer to it as TV stability:
\begin{definition}[TV Stability]
An algorithm $\train(\dataset)$ is $\delta$-TV stable if for any two \newterm{neighbouring datasets} $\dataset$, $\dataset'$ of size $n$, for any subset $T \subseteq \Theta$ it holds that 
$
    \Pr[ \train(\dataset) \in K ] \leq \Pr[ \train(\dataset') \in K] + \delta.
$
\end{definition}
It is easy to see that $(\epsilon, \delta)$-DP immediately implies $\delta'$-TV stability with:
\begin{equation}\label{eq:dg-bound-loose}
    \delta' = \exp(\epsilon) - 1 + \delta.
\end{equation}

\paragraph{From Classical to Distributional Generalization.} Similarly to the classical generalization, one way to achieve distributional generalization is through strong stability:

\begin{theorem}\label{stmt:tv-to-dg}
Suppose that the training algorithm is $\delta$-TV stable. Then, the algorithm satisfies $\delta$-DG.
\end{theorem}

We refer to \cref{app:proofs} for the proofs of this and all other formal statements in the rest of the paper.

As DP implies TV-stability, by \cref{stmt:tv-to-dg} we have that DP also implies DG. We show that DP algorithms enjoy a significantly stronger stability guarantee than  previously known,
which means that the DG guarantee that one obtains from DP is also stronger.

\begin{proposition}\label{stmt:dp-to-tv-tight}
An algorithm which is $(\epsilon, \delta)$-DP is also $\delta'$-TV stable with:
\[
    \delta' = \frac{\exp(\epsilon) - 1 + 2\delta}{\exp(\epsilon) + 1}.
\]
\end{proposition}

\begin{wrapfigure}{r}{.45\textwidth}
    \centering
    \vspace{-0.2em}
    \resizebox{\linewidth}{!}{
\scalefont{1.35}
\begin{tikzpicture}
\begin{axis}[
    restrict y to domain = 0:5,
    xmin = 0, xmax = 5,
    ymin = 0, ymax = 1.5,
    xtick distance = 0.5,
    ytick distance = 0.5,
    grid = both,
    legend pos = south east,
    legend style = {draw=none},
    legend cell align={left},
    minor tick num = 1,
    major grid style = {lightgray},
    minor grid style = {lightgray!25},
    width = \textwidth,
    height = 0.4\textwidth,
    xlabel = {$\epsilon$},
    ylabel = {DG},
    ]

\addplot[
    domain = 0:5,
    samples = 200,
    smooth,
    ultra thick,
    blue,
] {exp(x) - 1};
\addlegendentry{\cref{eq:dg-bound-loose}}

\addplot[
    domain = 0:5,
    samples = 200,
    smooth,
    ultra thick,
    cyan,
] {x};
\addlegendentry{\citet{steinke2020reasoning}}

\addplot[
    domain = 0:5,
    samples = 200,
    smooth,
    ultra thick,
    red,
] {(exp(x) - 1)/(exp(x) + 1)};
\addlegendentry{\cref{stmt:dp-to-tv-tight} (Ours)}
 
\end{axis}
\end{tikzpicture}}
\vspace{-1.6em}
\caption{DG bound from $(\epsilon, 0)$-DP.}
\label{fig:bound}
\end{wrapfigure}
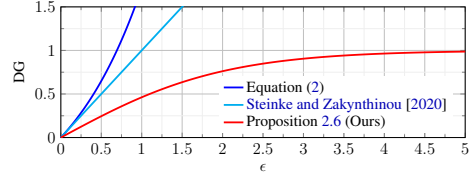

In \cref{app:it-bounds}, we discuss the relationship of this result to other works in the literature on information-theoretic generalization. In particular, to \citet{steinke2020reasoning} whose results can also be used to relate DP and DG. \cref{fig:bound} shows that the known bounds quickly become vacuous unlike the bound in \cref{stmt:dp-to-tv-tight}. In fact, we show that our bound is tight in \cref{app:proofs}.

{\bf Stronger Distributional Generalization Guarantees.}
Although DG immediately implies generalization for all bounded properties,
it is possible to obtain tighter bounds from TV stability.
For example,
directly applying DG to the \emph{subgroup loss} property yields a bound
that decays with the size of the subgroup:
accuracy on very small subgroups is not guaranteed to generalize well.
In \cref{app:subgroup} we show that TV stability implies ``subgroup DG'',
which guarantees that the accuracy on even small subgroups generalizes well in expectation.
As another example, in \cref{app:calibration} we show that TV stability also ensures generalization of calibration properties of the learning algorithm.

\vspace{-0.1in}
\section{Example Applications}\label{sec:applications}
To demonstrate that WYSIWYG is a useful property in algorithm design, in the remainder of this paper
we use it to construct
simple and high-performing algorithms
for three example applications: mitigation of disparate impact of DP, ensuring group-distributional robustness, and mitigation of robust overfitting in adversarial training.

\paragraph{Mitigating Disparate Impact of DP.}
First, we consider applications in which learning presents privacy concerns, e.g., in the case that the training data contains sensitive information. Using training procedures that satisfy DP is a standard way to guarantee privacy in such settings.
Training with DP, however, is known to incur \emph{disparate impact} on the model accuracy:
some subgroups of inputs can have worse test accuracy than others.
For example, \citet{bagdasaryan2019differential} show that using DP-SGD---a standard algorithm for satisfying DP~\cite{abadi2016deep}---in place of regular SGD
causes a significant accuracy drop on
``darker skin'' faces in models trained on the CelebA dataset of celebrity faces~\cite{liu2015faceattributes},
but a less severe drop on ``lighter skin'' faces.
Our goal is to mitigate such disparate impact. This issue---a quality-of-service harm~\cite{madaio2020co}---is but one of many possible harms due to ML systems. We do not aim to mitigate any other broad fairness-related issues, nor claim this is possible within our framework.

Formally, we assume the data distribution $\datagen$ is a mixture of $m$ groups indexed by set $\setgroups = \{1, \ldots, m\}$, such that $\datagen = \sum_{i = 1}^{m} q_{i} \datagen_{i}$. The vector $(q_i, \ldots, q_m) \in [0, 1]^m$ represents the group probabilities, with $\sum_{i = 1}^{m} q_{i} = 1$.
For given parameters $(\eps, \delta)$, we
want to learn a model $\params$ that simultaneously satisfies
$(\eps, \delta)$-DP, has high overall accuracy, and incurs small \newterm{loss disparity}:
\vspace{-0.07in}
\begin{equation}\label{eq:fair-acc}
    \max_{\group, \group' \in \setgroups} \left| \E_{\xy \sim \datagen_{\group}}[\ell(\xy; \params)] - \E_{\xy \sim \datagen_{\group'}}[\ell(\xy; \params)]\right|.
\end{equation}
\paragraph{Group-Distributional Robustness.}
Next, we consider a setting of \newterm{group-distributionally robust optimization} \citep[e.g.,][]{sagawa2019distributionally,hu2018does}. If in the standard learning approach we want to train a model that minimizes \emph{average} loss, in this setting, we want to minimize the \emph{worst-case (highest) group loss}. This objective can be used to mitigate fairness concerns such as those discussed previously, as well as to avoid learning spurious correlations~\cite{sagawa2019distributionally}. 

Formally, we want to learn a model $\params$ that minimizes the \newterm{worst-case group loss}:
\begin{equation}
\label{eq:dro-obj}
\max_{\group \in \setgroups}\E_{\xy \sim \datagen_\group}\left[\ell(\xy; \params)\right].
\end{equation}
Unlike the previous application, in this setting, we do not require privacy of the training data. We use training with DP as a \emph{tool} to ensure the generalization of the worst-case group loss.

\paragraph{Mitigating Robust Overfitting.} 
Finally, we consider the setting of robustness to test-time adversarial examples through adversarial training~\cite{goodfellow2014explaining}. A common way to train robust models in this sense in image domains is to minimize \newterm{robust (adversarial) loss}~\cite{madry2017towards}:
\vspace{-.2em}
\begin{equation}
\label{eq:adv-obj}
\E_{(x, y) \sim \datagen} \left[ \max_{ \|x - x_\mathsf{adv}\|_p \leq \gamma }
\ell((x_\mathsf{adv}, y); \params) \right],
\end{equation}
where $\gamma > 0$ is a small constant. \citet{rice2020overfitting} observed that adversarially trained models exhibit ``robust overfitting'': higher generalization gap of robust loss than that of the regular loss. In this application, we similarly aim to use a relaxed version of training with DP as a tool to ensure generalization of robust loss, thus mitigate robust overfitting.

\section{Algorithms which Distributionally Generalize}\label{sec:algorithms}
\vspace{-0.05in}

In this section, we construct algorithms for the applications in \cref{sec:applications}.
Our approach follows the blueprint:
First, we apply a principled algorithmic intervention that ensures desired behavior on \emph{the training data} (e.g., importance sampling). Second, we modify the resulting algorithm to additionally ensure DG, which guarantees that the desired behavior generalizes to \emph{test time}.

\subsection{DP Training with Importance Sampling}
\label{sec:dpis}

\newcommand{\samplerate}{\bar p}
\newcommand{\weightfunc}{p}

Our first algorithm, \dpis (\cref{alg:dp_is}), is a version of DP-SGD~\cite{abadi2016deep} which performs importance sampling.
\dpis is designed to mitigate disparate impact while retaining DP guarantees. 
The standard \dpsgd samples data batches using \newterm{uniform Poisson subsampling}\emph{:} Each example in the training set is chosen into the batch according to the outcome of a Bernoulli trial with probability $\samplerate \in [0, 1]$. To correct for unequal representation and the resulting disparate impact, we use \newterm{non-uniform Poisson subsampling}: Each example $\xy \in \dataset$ has a possibly different probability $\weightfunc(\xy)$ of being selected into the batch, where $\weightfunc(\xy)$ does not depend on the dataset $\dataset$ otherwise, and is bounded: $0 \leq \weightfunc(\xy) \leq \weightfunc^* \leq 1$. We denote this subsampling procedure as $\mathsf{Sample}_{\weightfunc(\cdot)}(\dataset)$.

We assume that we know to which group any $\xy = (x, y)$ belongs, denoted as $g(\xy)$, e.g., the group is one of the features in $x$. We choose $\weightfunc(\xy)$ to satisfy two properties. First, to increase the sampling probability for examples in minority groups: $\weightfunc(\xy) \propto \nicefrac{1}{q_{g(\xy)}}$. Second, to keep the average batch size equal to $\samplerate \cdot n$ as in standard \dpsgd. In the rest of the paper, we assume that the group probabilities $(q_1, \ldots, q_m)$  are known, but it is possible to estimate them in a private way using standard methods~\cite{nelson2019sok}.
We present \dpis in \cref{alg:dp_is}, along with its differences to the standard \dpsgd.

\begin{figure}[t]
\begin{algorithm}[H]
  \caption{\dpis (DP Importance Sampling SGD)}\label{alg:dp_is}
  \renewcommand{\algorithmicrequire}{\textbf{Input:}}
  \begin{algorithmic}
  \Require Dataset $\dataset$, loss $\loss(\xy; \params)$, initial parameters $\params_0$, learning rate $\eta$, maximal gradient norm~$C$, noise parameter~$\sigma$, number of epochs $T$, sampling rate $\samplerate$,
  $\highlight{\text{group probabilities } (q_1, \ldots, q_m)}.$
  \For{$t = 1, \ldots, T$}
  \State Sample batch $S_{t} \gets \highlight{\smash{\mathsf{Sample}_{\weightfunc(\cdot)}}(\dataset)\text{, with sampling probabilities } \weightfunc(\xy) ~\define~ \nicefrac{\samplerate}{m \cdot q_{g(\xy)}}}$
  \State $\tilde g_t \gets
  \frac{1}{|S_t|}  \sum_{\xy \in S_t}
  \underbrace{\nicefrac{ 1 }{\max \{1,~C^{-1} \cdot \| \nabla_\params \loss(\xy; \params)\|_2\}}}_{\text{Gradient clipping}}
  \cdot \nabla_\params \loss(\xy; \params)
  + \underbrace{\sigma C \cdot \mathcal{N}(0, I) }_{\text{Gradient noise}}$
  \State $\params_{t} \gets \params_{t-1} + \eta \cdot \tilde g_t$
  \EndFor
  \end{algorithmic}
\end{algorithm}
\vspace{-0.1in}
\notes{
\setlength\fboxsep{0pt}
The $\highlight{\text{highlighted}}$ parts indicate the differences with respect to \dpsgd. We obtain \dpsgd as a special case when we have a single group with $q = 1$ (implying $\weightfunc(\xy) = \samplerate$).}
\vspace{-0.1in}
\end{figure}

\paragraph{DP Properties of \dpis.}
Uniform Poisson subsampling is well-known to amplify the privacy guarantees of an algorithm~\cite{chaudhuri2006random, li2012sampling}.
For example, \citet{li2012sampling} show that if an algorithm $\train(\dataset)$ satisfies $(\epsilon, \delta)$-DP, then $\train \circ \smash{\mathsf{Sample}_{\samplerate}}(\dataset)$ provides approximately $(O(\samplerate \epsilon), \samplerate \delta)$-DP for small values of $\epsilon$. We show in \cref{app:proofs} that non-uniform Poisson subsampling provides the same amplification guarantee with $\samplerate = \weightfunc^*$, where $\weightfunc^*$ is the maximum value of $\weightfunc(\cdot)$.

As this guarantee is independent of the internal workings of $\train(\dataset)$, it is loose. For DP-SGD, one way of computing tight privacy guarantees of subsampling is using the notion of \newterm{Gaussian differential privacy} (GDP)~\cite{dong2019gaussian}. GDP is parameterized by a single parameter $\mu$. If an algorithm $\train(\dataset)$ satisfies $\mu$-GDP, one can efficiently compute a set of $(\epsilon, \delta)$-DP guarantees also satisfied by $\train(\dataset)$~\cite{dong2019gaussian}.
We show that we can use any GDP-based mechanism for computing the privacy guarantee of \dpsgd to obtain the privacy guarantees of \dpis in a black-box manner:
\begin{proposition}\label{stmt:dpiw-blackbox}
Let us denote by $\mu(\samplerate, \sigma, C, T)$ (see \cref{alg:dp_is}) a function that returns a $\mu$-GDP guarantee of \dpsgd. Then, \dpis satisfies a GDP guarantee $\mu(\weightfunc^*, \sigma, C, T)$.
\end{proposition}

\subsection{Gaussian Gradient Noise}\label{sec:algo_dr}
We showed that \dpis enjoys theoretical guarantees for both DP and DG.
DP models, however, often have lower test accuracy compared to standard training~\cite{chaudhuri2011differentially}. 
This can be an unnecessary disadvantage in settings where privacy is not required,
such as in our robustness applications.
Thus, we explore training algorithms which are inspired by our theory 
yet do not come with generic theoretical guarantees of DG.

Note that \dpsgd uses gradient \newterm{clipping} and \newterm{noise} (see \cref{alg:dp_is}).
Individually, these are used as \emph{regularization methods}
for improving stability and generalization~\cite{hardt2016train, neelakantan2015adding}.
Following this, we relax \dpis to only use gradient noise.
This sacrifices privacy guarantees in exchange for practical performance.
Specifically, we apply gradient noise to three standard algorithms for achieving group-distributional robustness: importance sampling (\iserm), importance weighting (\iwerm)~\cite{gretton2009covariate}, and \gdro~\cite{sagawa2019distributionally}. This results in the following variations: \noisyiserm, \noisyiwerm, \noisygdro, respectively. Similarly, we apply gradient noise to standard PGD adversarial training~\cite{madry2017towards}. See Appendix~\ref{app:algo} for details.

\section{Experiments}
\label{sec:experiments}
We empirically study the distributional generalization in real-world applications. 
The code for the experiments is available at  \url{https://github.com/yangarbiter/dp-dg}.

\paragraph{Datasets.}
We use the following datasets with group annotations:
CelebA~\citep{liu2015faceattributes},
UTKFace~\citep{zhifei2017cvpr},
iNaturalist2017 (iNat)~\cite{van2018inaturalist},
CivilComments~\cite{borkan2019nuanced},
MultiNLI~\cite{N18-1101, sagawa2019distributionally},
and ADULT~\cite{kohavi1996scaling}.
For every dataset, each example belongs to one group (e.g., CelebA) or multiple groups (e.g., CivilComments).
For example, in the CelebA dataset, there are four groups: ``blond male'', ``male with other hair color'', ``blond female'', and ``female with other hair color''. Additionally, we use the CIFAR-10~\cite{krizhevsky2009learning} dataset for the adversarial-overfitting application. We present more details on the datasets, their groups, and used model architectures in \cref{app:experiment}.

\subsection{Enforcing DG in Practice}
\label{sec:dg-in-practice}

We empirically confirm that a training procedure 
with DP guarantees also has a bounded DG gap.

\begin{wrapfigure}{r}{.33\linewidth}
\vspace{-1.4em}
\centering
\includegraphics[width=.33\textwidth]{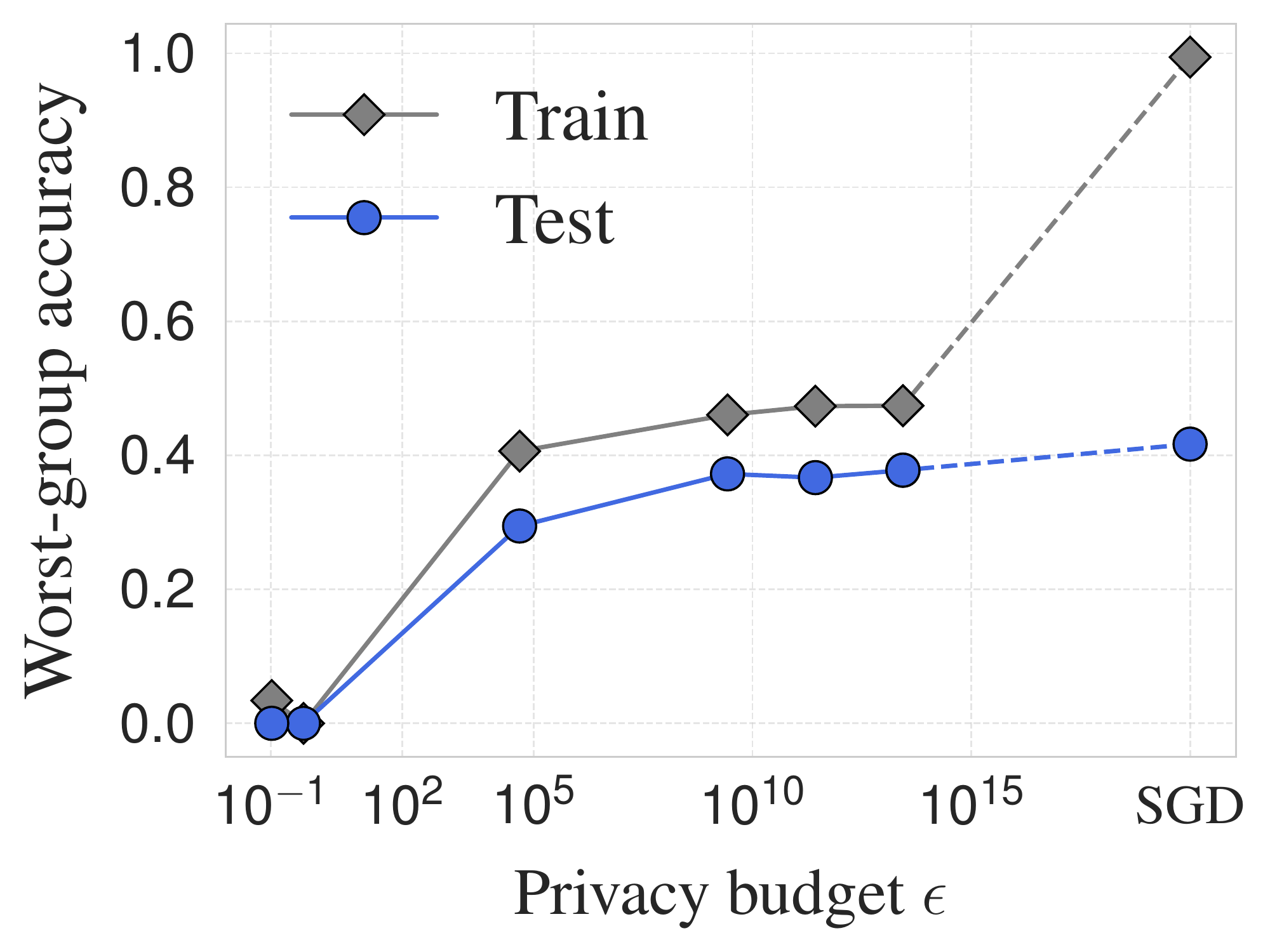}
\vspace{-0.2in}
\caption{\textbf{Privacy induces DG.} Train/test worst-case group accuracies as a function of privacy parameter $\epsilon$ of \dpsgd on CelebA (x axis). Increasing privacy reduces the generalization gap.}\label{fig:dpdg_gap_privacy}
\vspace{-2em}
\end{wrapfigure}

In practice, it is not possible to compute the exact DG gap. As a proxy in applications which concern subgroup performance in this section, and \cref{sec:exp-disparate,sec:exp-dro}, we use the difference between train-time and test-time worst-group accuracy. This (a) follows the empirical approach by \citet{nakkiran2020distributional} which proposes to estimate the gap in \cref{eq:dg-variational-defn} using a finite set of test functions, and (b) measures the aspect of distributional generalization that is relevant to our applications. We provide more details on this choice of the proxy measure in~\cref{app:empirical-dg}.

We train a model on CelebA using \dpsgd for varying levels of $\eps$. \cref{fig:dpdg_gap_privacy} shows that 
the gap between training and testing worst-group accuracy increases as 
the level of privacy decreases, which is consistent with our theoretical bounds. 
In \cref{app:dg-regularization-basic} we also explore how regularization methods which do not necessarily formally imply DG, can empirically improve DG.

\subsection{Disparate Impact of Differentially Private Models}\label{sec:exp-disparate}
We evaluate \dpis (\cref{alg:dp_is}),
and demonstrate that it can mitigate the disparate impact in realistic settings
where both privacy and fairness are required.

\begin{figure*}[t!]
    \centering
    \subfigure[Accuracy disparity (lower is better)]{
      \includegraphics[width=.3\textwidth]{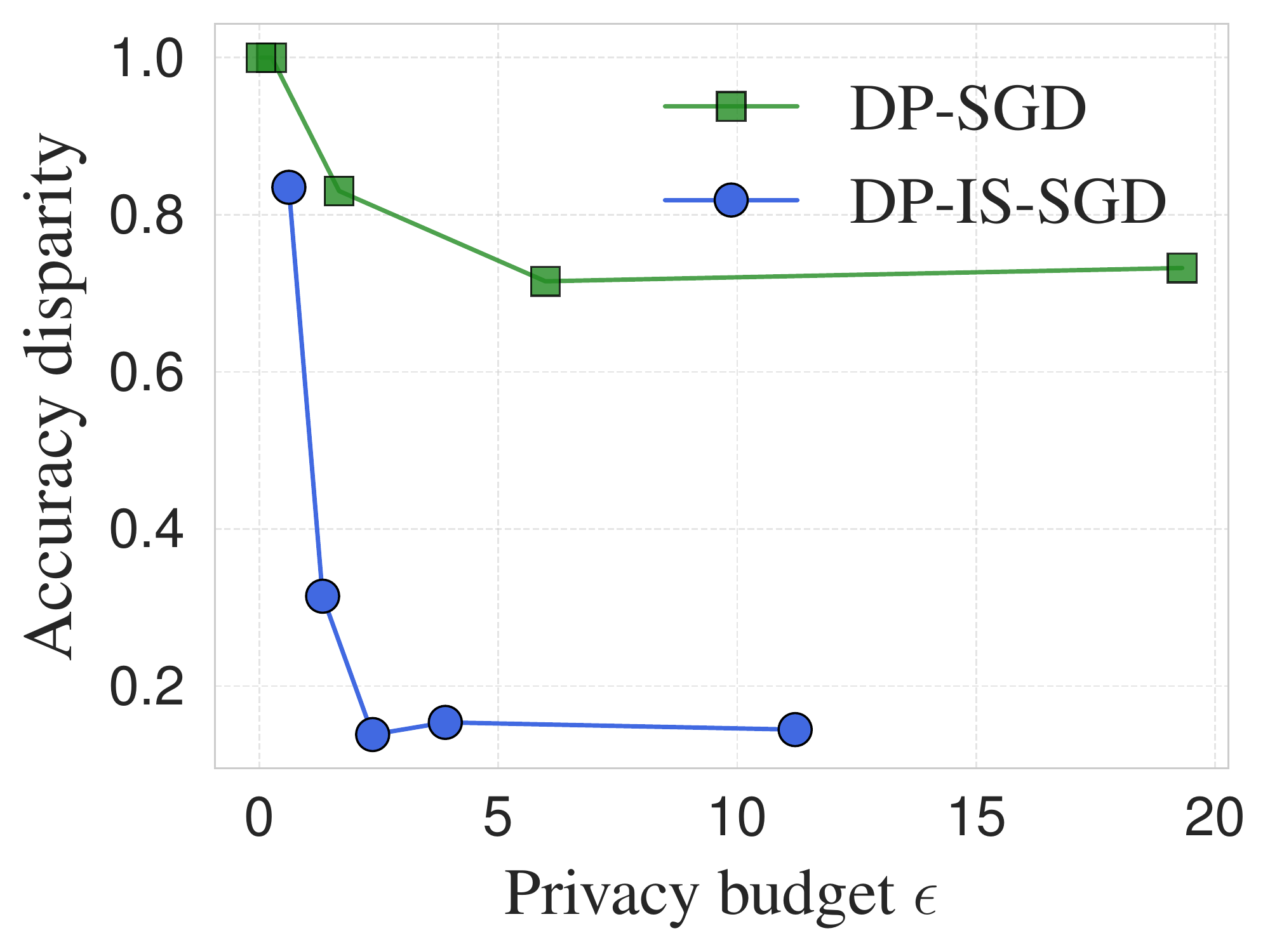}}
    \subfigure[Worst-group accuracy (higher is better)]{
      \includegraphics[width=.3\textwidth]{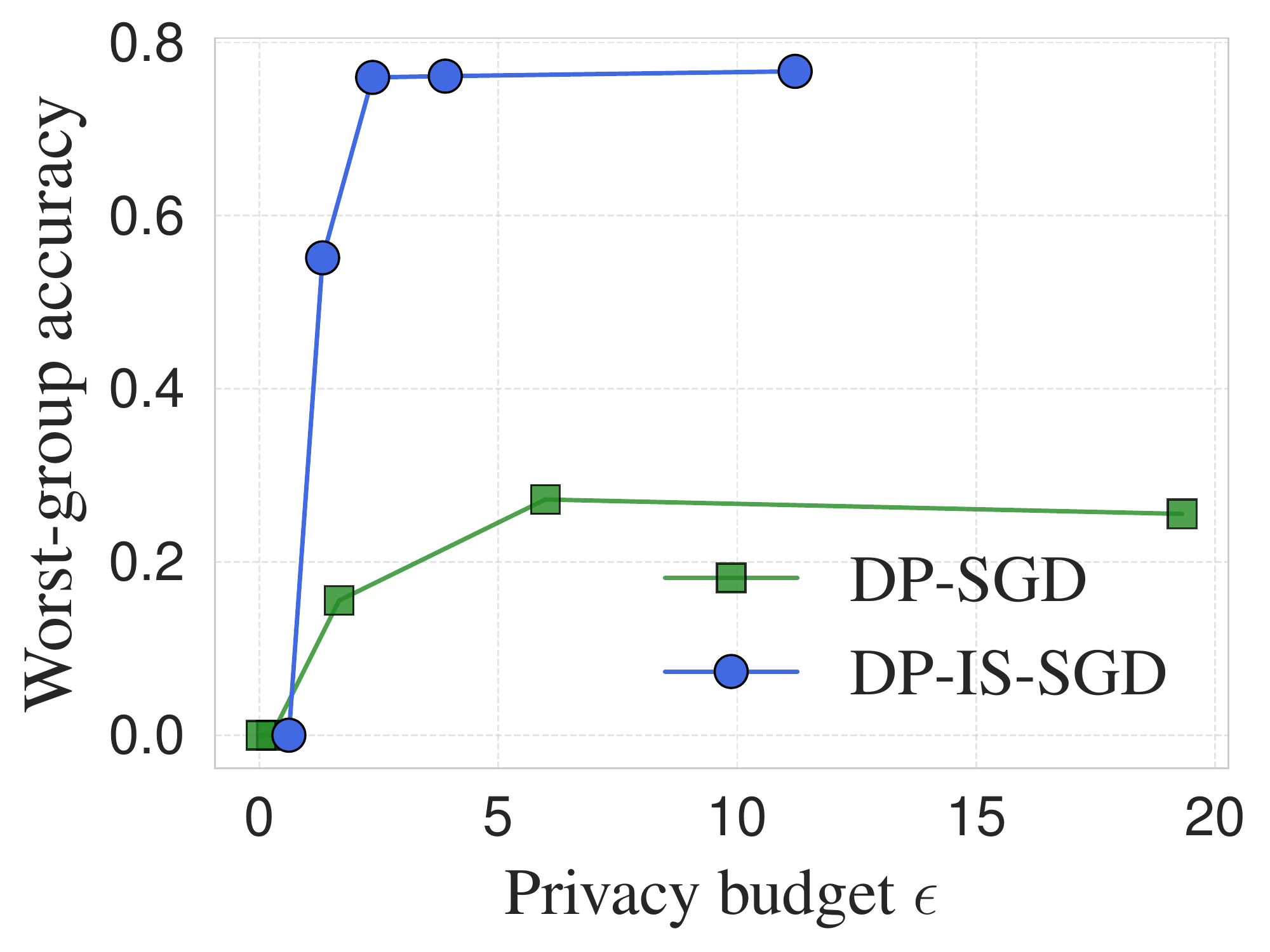}}
    \subfigure[Test accuracy (higher is better)]{
      \includegraphics[width=.3\textwidth]{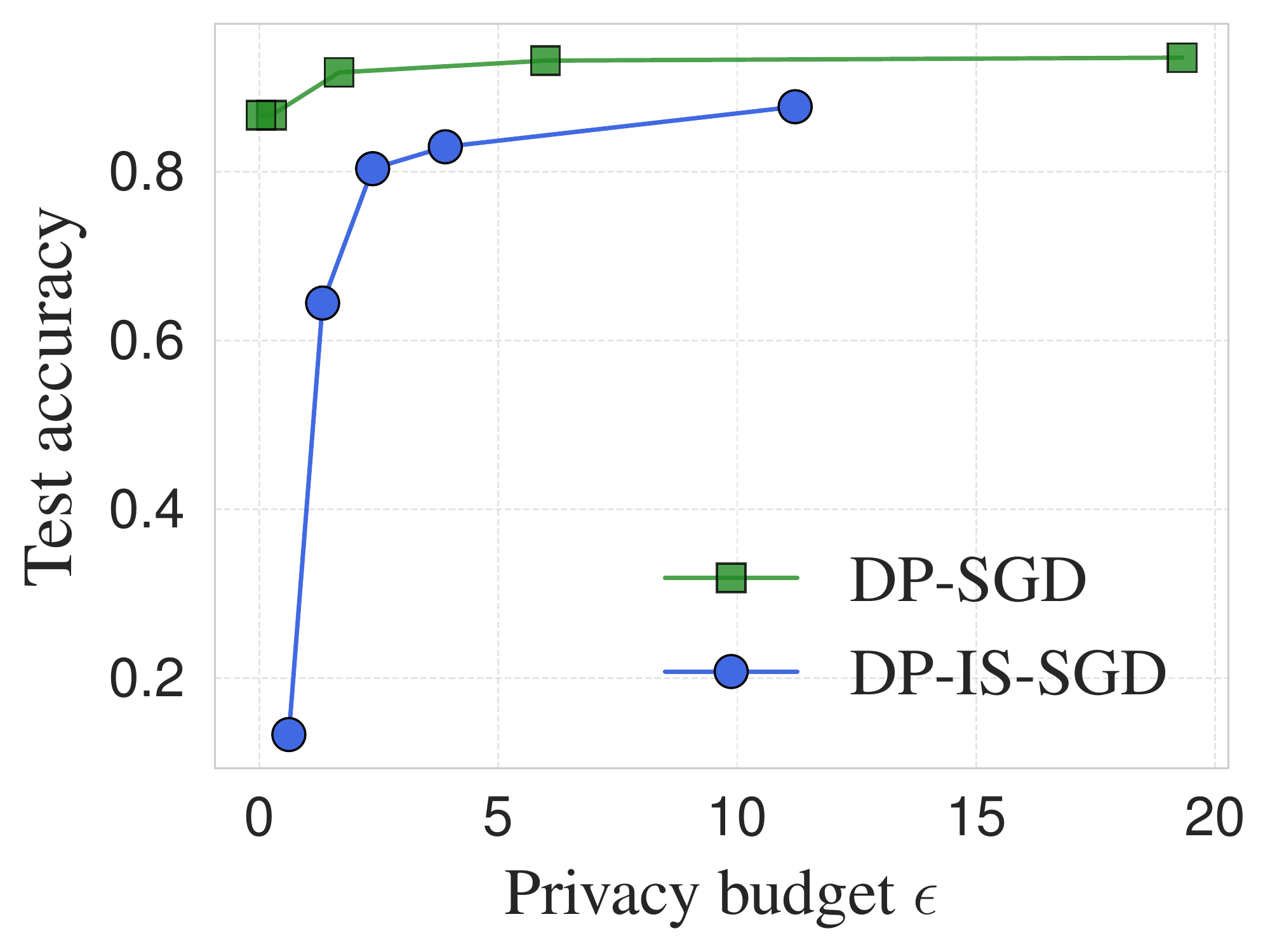}}
    \vspace{-0.05in}
    \caption{
    {\bf Importance Sampling Improves Disparate Impact of DP-SGD.}
    The accuracy disparity of the models trained with \dpsgd and \dpis on CelebA.
    Adding importance sampling (IS) improves disparate impact at most privacy budgets in this setting.
    We set $\delta = \nicefrac{1}{2n}$, where $n$ is the number of training examples.
    We use GDP accountant to compute the privacy budget $\varepsilon$.
    }
    \label{fig:disparity}
    \vspace{-0.11in}
\end{figure*}

\cref{fig:disparity} shows the accuracy disparity, test accuracy, and worst-case group accuracy, computed as in \cref{eq:fair-acc}, as a function of the privacy budget $\epsilon$. The models are trained with \dpsgd and \dpis.
When comparing \dpsgd and \dpis with the same or similar $\epsilon$, we observe that \dpis achieves lower disparity on all datasets.
However, this comes with a drop in average accuracy.
On CelebA, with $\epsilon \in [2, 12]$, \dpis has around 8 p.p. lower test accuracy than \dpsgd. 
At the same time, the disparity drop ranges from 40~p.p. to 60~p.p., which is significantly higher than  the accuracy drop.
We observe similar results on UTKFace. On iNat, however, although DP-IS-SGD decreases disparity, the overall test accuracy suffers a significant hit. This is likely because the minority subgroup is very small, which results in high maximum sampling probability~$p^*$, thus deteriorating the privacy guarantee. Details for UTKFace and iNat are in \cref{app:exp-disparate-extra}.

In summary, we find that \dpis can achieve lower disparity
at the same privacy budget compared to standard DP-SGD,
with mild impact on test accuracy.

\paragraph{Comparison to DP-SGD-F~\cite{xu2021removing}.} \dpsgd-F is a variant of \dpsgd which dynamically adapts gradient-clipping bounds for different groups to reduce the disparate impact. We did not manage to achieve good overall performance of \dpsgd-F on the datasets above. In \cref{app:exp-disparate-extra}, we compare it to \dpis on the ADULT dataset (used by \citet{xu2021removing}), finding that \dpis obtains lower disparity for the same privacy level, yet also lower overall accuracy.

\subsection{Group-Distributionally Robust Optimization}\label{sec:exp-dro}

We investigate whether our proposed versions of standard algorithms with Gaussian gradient noise (\cref{sec:algo_dr}) can improve group-distributional robustness.
To do so, we evaluate empirical DG using worst-group accuracy as a proxy for DG gap as in \cref{sec:dg-in-practice},
following the evaluation criteria in prior work~\citep{sagawa2019distributionally,idrissi2022simple}.
State-of-the-art (SOTA) methods apply $\ell_2$ regularization and early-stopping to achieve the best performance.
We compare three baselines with $\ell_2$ regularization, \iserm-$\ell_2$, \iwerm-$\ell_2$, and \gdro-$\ell_2$ to our noisy-gradient variations as well as \dpis.
We use the validation set to select the best-performing regularization parameter and epoch (for early stopping) for each method.
See \cref{app:exp-dro-extra} for details on the experimental setup.

\cref{tab:sota_table} shows the worst-group accuracy of each algorithm on five datasets.
When comparing \iserm, \iwerm, and \gdro with their noisy counterparts,
we observe that the noisy versions in general have similar or slightly better performance compared to non-noisy counterparts. For instance, IS-SGD-n improves the SOTA results on CivilComments dataset. 
This showcases that in terms of learning distributionally robust models, \textit{noisy gradient can be a more effective regularizer than the currently standard $\ell_2$ regularizer}. 
We also find that \dpis improves on baseline methods or even achieves SOTA-competetitive performance on several datasets. For instance, on CelebA and MNLI, \dpis achieves better performance than IS-SGD-$\ell_{2}$.
This is surprising, as DP tends to deteriorate performance. This suggests that distributional robustness and privacy
are not incompatible goals.
Moreover, DP can be a useful tool even when privacy is not required.

\begin{table}[t!]
    \centering
    \caption{
    \textbf{Our noisy-gradient algorithms produce competitive results 
    compared to counterparts with $\ell_2$ regularization.}
    The table shows the worst-group accuracy of each algorithm.
    Baselines are in the top rows; our algorithms are in the bottom.
    For gDRO-$\ell_2$-SOTA, we show avg. $\pm$ std. over five runs from \citet{idrissi2022simple}.
    For CelebA, we show avg. $\pm$ std. over three random splits.
    }
    \label{tab:sota_table}
    \resizebox{.57\linewidth}{!}{
    \setlength{\tabcolsep}{3pt}
    \begin{tabular}{l|ccccc}
    \toprule
           &  CelebA & UTKFace & iNat. & Civil. & MNLI \\
    \midrule
      \erm-$\ell_2$       & 73.0 $\pm$ 2.2  & 86.3  & 41.8  & 57.4 & 67.9 \\
      \iserm-$\ell_2$     & 82.4 $\pm$ 0.5  & 85.8  & 70.6  & 64.3 & 70.4 \\
      \iwerm-$\ell_2$\tablefootnote{\iwerm numbers are different from Fig.~\ref{fig:CelebA}, as in the figure we do not apply regularization.}
                          & \textbf{89.0} $\pm$ 0.9 & 86.5  & 67.6 & 65.7 & 68.1 \\
      \gdro-$\ell_2$      & 84.5 $\pm$ 0.8 & 85.2  & 67.3 & 67.3 & 75.9 \\
      \gdro-$\ell_2$-SOTA & 86.9 $\pm$ 0.5 & --- & --- & 69.9 $\pm$ 0.5 & \textbf{78.0} $\pm$ 0.3 \\
    \midrule
      \dpis         & 86.0 $\pm$ 0.8 &  82.5          & 51.4 & 70.4 & 72.3 \\
      \noisyiserm   & 84.9 $\pm$ 1.0 &  85.5    & \textbf{71.0} & \textbf{71.9} & 70.8 \\
      \noisyiwerm   & \textbf{88.5} $\pm$ 0.4   & \textbf{88.5} & 70.9 & 69.9 & 69.7 \\
      \noisygdro    & 83.3 $\pm$ 0.5  & 87.5   & 56.4 & 71.3 & \textbf{78.0}  \\
    \bottomrule
    \end{tabular}
    } %
    
\end{table}

\subsection{Mitigating Robust Overfitting}
\label{sec:exp-rob-overfitting}

\begin{figure}[t]
    \centering
    \subfigure[Gen. gap of robust accuracy (lower is better)]{\includegraphics[width=.36\linewidth]{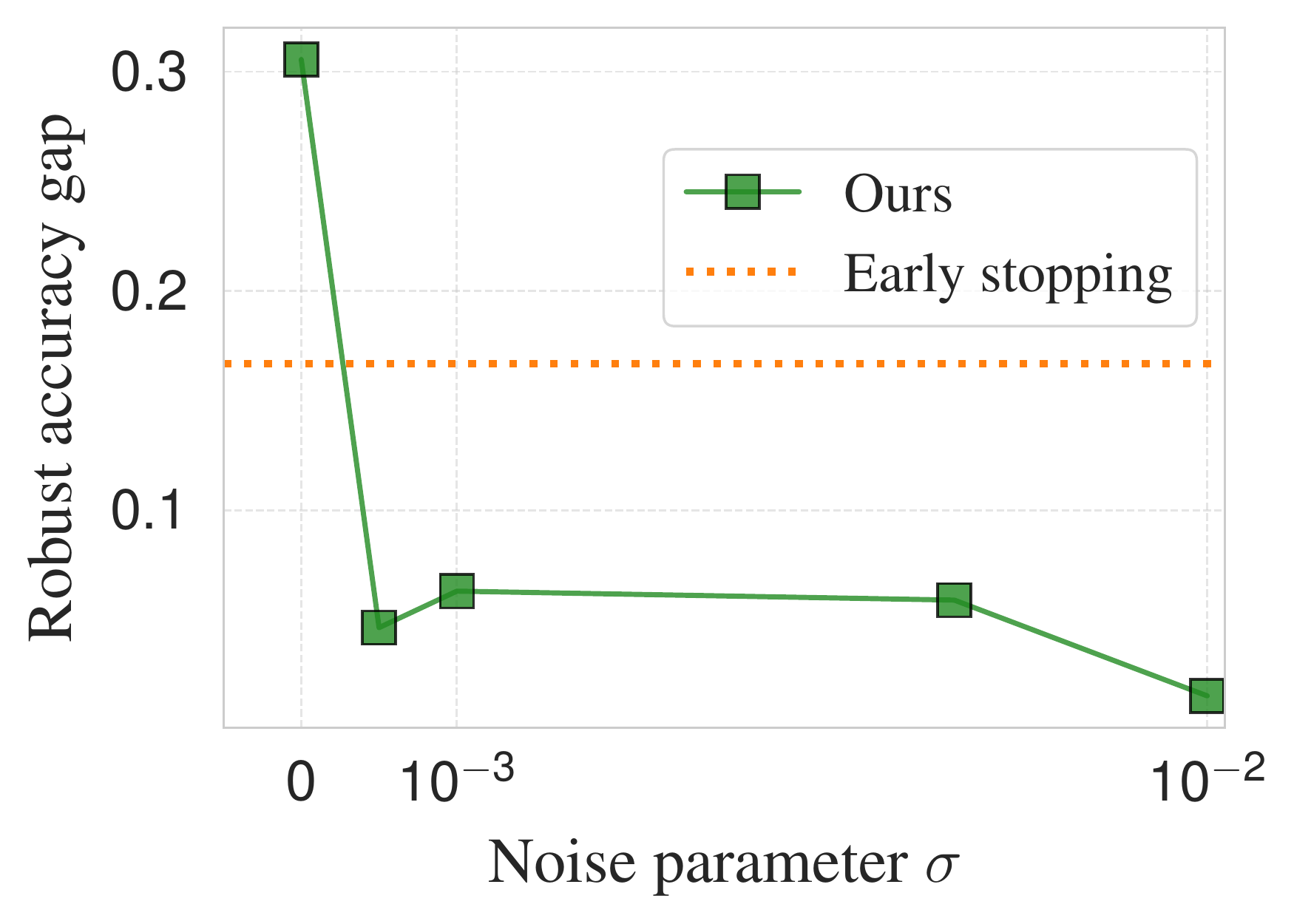}}~
        \subfigure[Robust accuracy (higher is better)]{\includegraphics[width=.36\linewidth]{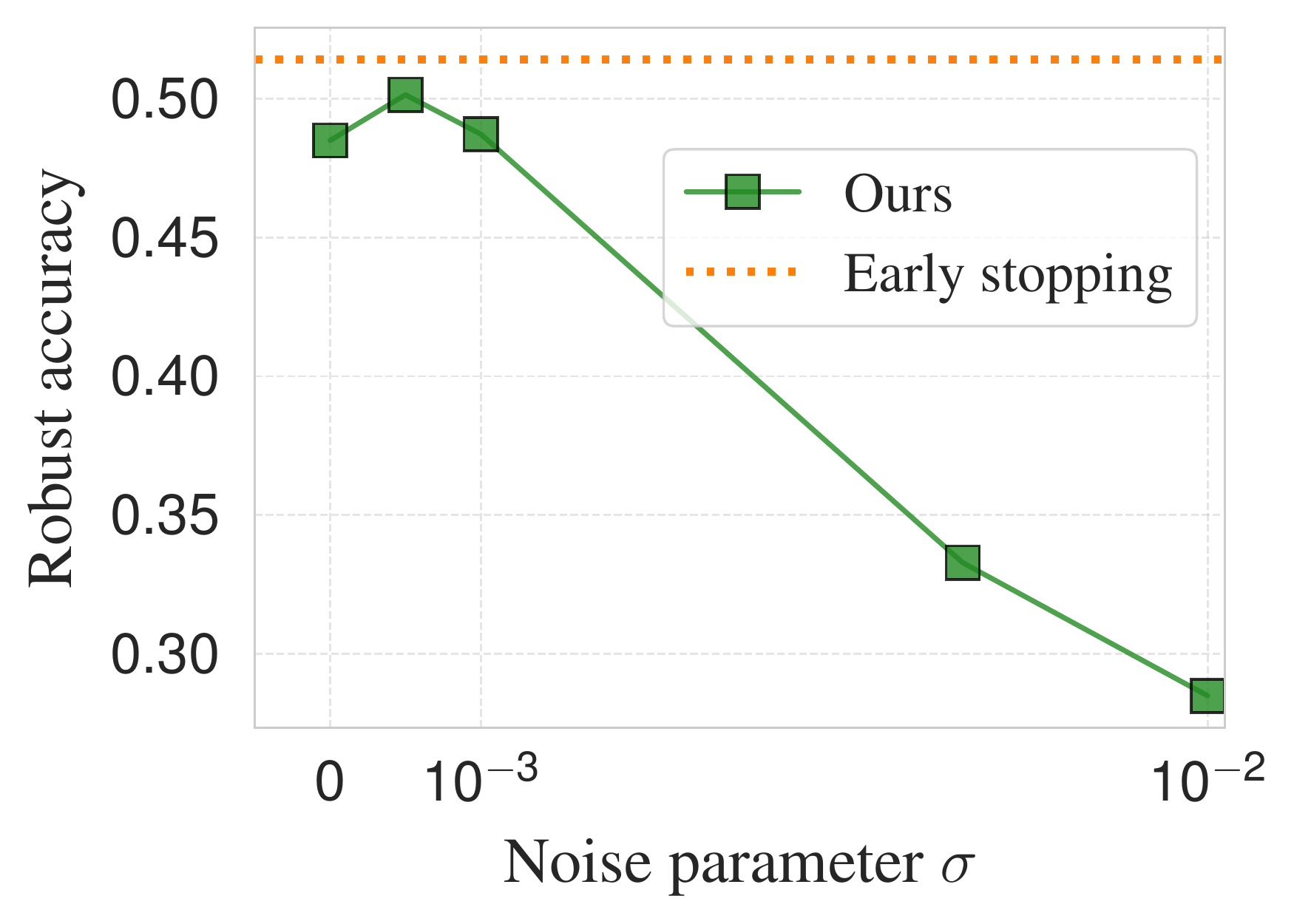}}
    \caption{\textbf{Noisy gradient reduces overfitting in adversarial training.} We show the generalization gap of robust accuracy (left), and test-time robust accuracy (right) of adversarially trained models with different levels of noise magnitude. The dash orange lines represent the performance of adversarial training with early stopping. The model trained without noise exhibits ``robust overfitting'' of about 30 p.p. Gradient noise reduces the generalization gap by more than 3$\times$ for all values of the noise parameter at a cost of decreased robust accuracy as the noise gets larger.}
    \label{fig:adv}
    \vspace{-0.1in}
\end{figure}

As in the previous section, we expect that a modification of a standard projected gradient descent (PGD) method for adversarial training~\cite{madry2017towards} with added Gaussian gradient noise (\cref{sec:algo_dr}) improves the generalization behavior of adversarial training.

To verify this, we adversarially train models on the CIFAR-10 dataset with varying levels of the noise magnitude. We provide more details on the setup in \cref{app:exp-rob-overfitting-extra}. \cref{fig:adv} shows that in standard adversarial training without noise the gap between robust training accuracy and robust test accuracy is large at approximately 30 p.p., which is consistent with the prior observations of \citet{rice2020overfitting}. By injecting noise into the gradient, our proposed approach decreases the generalization gap of robust accuracy by more than 3$\times$ to less than 10 p.p. Surprisingly, in our experiments, training with gradient noise achieves both a small adversarial accuracy gap \textit{and} better adversarial test accuracy compared to standard adversarial training, when using a small noise magnitude ($\sigma=5 \times 10^{-4}$).
In terms of resulting robust accuracy, the method's performance is comparable to early stopping, identified as the most effective way to prevent robust overfitting by \citet{rice2020overfitting}.
These experimental results demonstrate how WYSIWYG can be a useful design principle in practice. 

\section{Related Work}
\label{sec:related}

{\bf DP and Strong Generalization.}
DP is known to imply a stronger than standard
notion of generalization, called \emph{robust generalization}\footnote{Unrelated to ``robust overfitting'' in adversarial training.}
~\citep{cummings2016adaptive, bassily2016algorithmic}.
Robust generalization can be thought as a
high-probability counterpart of DG: generalization holds with high probability
over the train dataset, not only on average over datasets.
We focus on our notion of DG for both conceptual and theoretical simplicity.
A more comprehensive discussion of relations to robust generalization is in Appendix~\ref{app:rob-gen}.
It can also be useful to consider intermediary definitions, varying in strength from DG to robust generalization. In~\cref{app:strong-dg},
we introduce such an notion (``strong DG'') and show its connections to DP. 
Other than robust generalization, our connections in \cref{sec:dg-theory} can also be derived from weaker generalization bounds that rely on information-theoretic measures~\cite{steinke2020reasoning}. We detail this in~\cref{app:it-bounds}. 

{\bf Disparate Impact of DP.}
\citet{bagdasaryan2019differential, pujol2020fair} have shown that ensuring DP in algorithmic systems can cause error disparity across population groups. \citet{xu2021removing} proposed a variant of \dpsgd for reducing disparate impact. We compare our method to DP-SGD-F in \cref{app:exp-disparate-extra}.
In another line of related work, \citet{sanyal2022unfair,cummings2019compatibility} show fundamental trade-offs between model's loss and DP training. As our theoretical results concern generalization, not loss per se, our results do not contradict these theoretical trade-offs. We discuss the relationship in detail in \cref{app:fairness-dp-tensions}. 

{\bf Group-Distributional Robustness.}
Group-distributional robustness aims to improve the worst-case group performance.
Existing approaches include using worst-case group loss~\citep{mohri2019agnostic, sagawa2019distributionally, zhang2020coping}, balancing majority and minority groups by reweighting or subsampling~\citep{byrd2019effect, sagawa2020investigation, idrissi2022simple}, leveraging generative models~\citep{goel2020model}, and applying various regularization techniques~\citep{sagawa2019distributionally, cao2019learning}.
Although some work~\citep{sagawa2019distributionally, cao2019learning} discusses the importance of
regularization in distributional robustness, they have not explored potential reasons for this (e.g. via the connection to generalization).
Another line of work studies how to improve group performance without group annotations~\citep{duchi2021statistics, liu2021just, creager2021environment}, which is a different setting from ours as we assume the group annotations are known.

{\bf Robust Overfitting.} \citet{rice2020overfitting,yu2022understanding} have shown that adversarially trained models tend to overfit in terms of robust loss. \citet{rice2020overfitting} proposed to use regularization to mitigate overfitting, but the noisy gradient has not been explored for this. We showed that the WYSIWYG framework can serve as an alternative direction for mitigating and explaining this issue.

\section{Conclusions and Future Work}
We argue that a ``What You See is What You Get'' property, which we formalize through the notion of distributional generalization (DG), can be desirable for learning algorithms, as it enables principled algorithm design in settings including deep learning. We show that this property is possible to achieve with DP training.
This enables us to leverage advances in DP to enforce DG in many applications.

We propose enforcing DG as a general design principle,
and we use it to construct simple yet effective algorithms in three settings.
In certain fairness settings,
we largely mitigate the disparate impact of differential privacy
by using importance sampling and enforcing DG in our new algorithm \dpis.
In our analysis, however, the privacy and DG guarantees of DP-IS-SGD deteriorate
in the presence of very small groups.
Future work could explore individual-level accounting~\cite{feldman2021individual} for a tighter analysis.
In certain worst-case generalization settings, inspired by \dpsgd, we propose using
a noisy-gradient regularizer.
Compared to SOTA algorithms in DRO, noisy gradient 
achieves competitive results
across many standard benchmarks.
In certain adversarial-robustness settings, our proposed noisy-gradient regularizer significantly reduces robust overfitting. An interesting direction for future work would be to explore its effectiveness in large-scale settings, e.g., ImageNet~\citep{croce2022interplay}. 
We hope future work can explore extending
this design principle to ensure generalization of other properties, such as calibration and counterfactual fairness.

\subsubsection*{Acknowledgements}
We thank Kamalika Chaudhuri, Benjamin L. Edelman,
Saurabh Garg, Gautam Kamath, Maksym Andriushchenko, and Cassidy Laidlaw
for useful feedback on an early draft.
BK acknowledges support from the Swiss National Science Foundation with grant 200021-188824.
Yao-Yuan Yang thanks NSF under CNS 1804829 and ARO MURI W911NF2110317 for the research support.
Yaodong Yu acknowledges support from the joint Simons Foundation-NSF DMS grant \#2031899.
PN is grateful for support of
the NSF and the
Simons Foundation for the Collaboration on the Theoretical
Foundations of Deep Learning\footnote{\url{https://deepfoundations.ai}}
through awards DMS-2031883 and \#814639.

\clearpage

\bibliographystyle{plainnat}
\bibliography{main}

\clearpage
\section*{Checklist}

\begin{enumerate}

\item For all authors...
\begin{enumerate}
  \item Do the main claims made in the abstract and introduction accurately reflect the paper's contributions and scope?
    \answerYes{}
  \item Did you describe the limitations of your work?
    \answerYes{We address the limitations where appropriate, e.g., high-level limitations in \cref{sec:applications}, and the fact that some methods we use in the empirical evaluations deviate from theory in \cref{sec:algorithms,sec:experiments}.}
  \item Did you discuss any potential negative societal impacts of your work?
    \answerYes{In \cref{sec:applications} we caution against using WYSIWYG for fairness-washing~\cite{bietti2020ethics} through technical solutionism. Other than this, our potential negative societal impacts are the same as with any other academic work which aims to improve the performance of deep learning: we risk improving performance of deep learning in downstream applications that are harmful, e.g., privacy-invasive, or violating other civil liberties. It is beyond the scope of our work to discuss such broad issues.}
  \item Have you read the ethics review guidelines and ensured that your paper conforms to them?
    \answerYes{}
\end{enumerate}

\item If you are including theoretical results...
\begin{enumerate}
  \item Did you state the full set of assumptions of all theoretical results?
    \answerYes{}
        \item Did you include complete proofs of all theoretical results?
    \answerYes{See \cref{app:proofs}.}
\end{enumerate}

\item If you ran experiments...
\begin{enumerate}
  \item Did you include the code, data, and instructions needed to reproduce the main experimental results (either in the supplemental material or as a URL)?
    \answerYes{We included the experiment code as supplemental material.}
  \item Did you specify all the training details (e.g., data splits, hyperparameters, how they were chosen)?
    \answerYes{See \cref{app:experiment}.}
        \item Did you report error bars (e.g., with respect to the random seed after running experiments multiple times)?
    \answerYes{For experimental settings in which it was feasible for us to repeat experiments multiple times, we reported the standard errors (\cref{tab:sota_table}).}
        \item Did you include the total amount of compute and the type of resources used (e.g., type of GPUs, internal cluster, or cloud provider)?
    \answerYes{See \cref{app:technical}.}
\end{enumerate}

\item If you are using existing assets (e.g., code, data, models) or curating/releasing new assets...
\begin{enumerate}
  \item If your work uses existing assets, did you cite the creators?
    \answerYes{See \cref{app:technical}.}
  \item Did you mention the license of the assets?
    \answerYes{See \cref{app:technical}.}
  \item Did you include any new assets either in the supplemental material or as a URL?
    \answerNA{}
  \item Did you discuss whether and how consent was obtained from people whose data you're using/curating?
    \answerYes{See \cref{app:technical}.}
  \item Did you discuss whether the data you are using/curating contains personally identifiable information or offensive content?
    \answerYes{See \cref{app:technical}.}
\end{enumerate}

\item If you used crowdsourcing or conducted research with human subjects...
\begin{enumerate}
  \item Did you include the full text of instructions given to participants and screenshots, if applicable?
    \answerNA{}
  \item Did you describe any potential participant risks, with links to Institutional Review Board (IRB) approvals, if applicable?
    \answerNA{}
  \item Did you include the estimated hourly wage paid to participants and the total amount spent on participant compensation?
    \answerNA{}
\end{enumerate}

\end{enumerate}

\newpage
\appendix
\onecolumn

\section{Related Work Details}
\label{app:related}

\subsection{Differential Privacy and Robust Generalization}
\label{app:rob-gen}
DP is known to imply a stronger 
notion of generalization, called \emph{robust generalization}, which
is a ``tail bound'' version of DG \citep{dwork2015generalization, cummings2016adaptive, bassily2016algorithmic}.
The original motivations for robust generalization are slightly different, 
but in our notation, a training procedure $\train$ is said to satisfy $(\gamma, \eta)$-Robust Generalization if and only if for any
test $\test: \sD \times \Theta \rightarrow [0, 1]$, we have
\begin{equation*}
    \Pr_{\substack{\dataset \sim \mathcal{D}^n, \train}} \left( \big|\E_{\xy \sim \dataset} \test\big(\xy; \train(\dataset)\big) -
    \E_{\xy \sim \mathcal{D}} \test\big(\xy; \train(\dataset)\big)\big| > \gamma \right) \leq \eta.
\end{equation*}
Comparing this to the definition of DG (Definition~\ref{def:dg}), it is immediate\bnote{Can you elaborate?} that any procedure that satisfies $(\gamma, \eta)$-Robust Generalization, also satisfies $(\gamma + \eta)$-DG.
Any training method satisfying $(\epsilon, \delta)$-DP also satisfies $(\mathcal{O}(\epsilon), \mathcal{O}(\delta))$-robust generalization, as long as the sample size $n$ is of size $\Omega(\log(1/\delta)/\epsilon^2)$ \citep[Theorem 7.2]{bassily2016algorithmic}, therefore it satisfies $\mathcal{O}(\epsilon + \delta)$-DG by the previous implication. Thus, it is possible to recover the result that DP implies DG as a consequence of these previous works, although with
looser bounds.

The difference between Distributional Generalization and Robust Generalization
is that DG considers all quantities in expectation,
while robust generalization considers tail bounds with respect to the train dataset.
We focus on DG for two reasons:
First, we believe DG is conceptually simpler, as
it can be seen as simply the TV distance between two natural distributions,
and does not involve additional parameters.
This simplicity is conceptually useful to the algorithm designer,
but also enables us to prove simpler tight theoretical bounds
that are independent of sample size.
Second, it is often possible to lift results about DG to
the stronger setting of robust generalization, with additional bookkeeping.
Thus, we focus on DG in this paper, with the understanding
that stronger guarantees
can be obtained for these methods if desired.

\subsection{Information-Theoretic Generalization Bounds}
\label{app:it-bounds}
Other than robust generalization, it is also possible to obtain bounds on generalization of arbitrary test (loss) functions using information-theoretic measures~\cite{xu2017information, russo2019much, steinke2020reasoning}. Recently, \citeauthor{steinke2020reasoning} showed that one such information-theoretic measure---\newterm{conditional mutual information} (CMI) between training algorithm outputs and the training dataset---is bounded if the training algorithm is DP or TV stable. Thus, we could relate stability and DG as in \cref{sec:from-stability-to-dg} using CMI as an intermediate tool.

In particular, \citeauthor{steinke2020reasoning} show that for a given $\phi: \sD \times \Theta \rightarrow [0, 1]$:
\begin{equation}\label{eq:dp-cmi}
    \Big|\E_\subtrain \test\big(\xy; \train(\dataset)\big) -
    \E_\subtest \test\big(\xy; \train(\dataset)\big)\Big| \leq \sqrt{ \frac{2}{n} \cdot \mathrm{CMI}_\datagen(\train)},
\end{equation}
where $\mathrm{CMI_\datagen}(\train)$ is the conditional mutual information of the training algorithm with respect to the data distribution. If the training algorithm satisfies $(\epsilon, 0)$-DP, they also show that $\mathrm{CMI_\datagen}(\train) \leq \frac{n}{2} \epsilon^2$, where $n$ is the dataset size. Plugging this into \cref{eq:dp-cmi}, we can see that the generalization upper bound (right-hand side) is~$\epsilon$. This is significantly looser than our bound in \cref{sec:from-stability-to-dg}, as illustrated in \cref{fig:bound}.

A recent line of work on information-theoretic bounds explores sharper generalization bounds using individual-level measures~\cite{bu2020tightening,haghifam2020sharpened}. Analogously, as a direction for future work, it could also be possible to obtain tighter bounds on DG using per-instance notions of stability~\cite{wang2017per, feldman2021individual}.

\subsection{Tension between Differential Privacy and Algorithmic Fairness}
\label{app:fairness-dp-tensions}
Beyond empirical observations that training with DP results in disparate impact on performance across subgroups~\cite{bagdasaryan2019differential, pujol2020fair}, \citet{cummings2019compatibility} and, more recently, \citet{sanyal2022unfair} theoretically analyze the inherent tensions between DP and algorithmic fairness.

It might appear that this trade-off contradicts our results in which we claim that using DP or similar noise-adding algorithms with additional train-time interventions can reduce disparate impact. However, \citet{cummings2019compatibility} and \citet{pujol2020fair} discuss the relationship between privacy and disparate performance (accuracy or false-positive/false-negative rates), whereas we discuss the relationship between privacy and generalization. Even if a DP model has to incur at least a certain error on small subgroups on average \cite[Lemma 1]{sanyal2022unfair}, this error is guaranteed to be similar at train time and test time (from our theoretical results in \cref{sec:from-stability-to-dg}). 

In terms of empirical results, the lower bound on subgroup error in Lemma 1 from \citet{sanyal2022unfair} vanishes for subgroups of size greater than 100 even for small values of epsilon (e.g., 0.1). The subgroups and values of epsilon in our experiments in \cref{sec:experiments} are all larger than this, thus in our regime we can achieve meaningful subgroup performance using the \dpis algorithm despite the fundamental trade-off.
\section{Proofs Omitted in the Main Body}\label{app:proofs}

\subsection{TV-Stability implies Distributional Generalization}
\begin{proof}[Proof of \cref{stmt:tv-to-dg}]
First, observe that the following distributions are equivalent as the dataset is an i.i.d. sample:
\begin{equation}\label{eq:dist-equiv}
    \begin{aligned}
        \Pr_{\substack{\dataset \sim \datagen^n \\ \xy \sim \dataset}}[\test\big(\xy; \model(\dataset)\big)] &\equiv
        \Pr_{\substack{\dataset \sim \datagen^{n - 1} \\ \xy \sim \datagen}}[\test\big(\xy; \model({\dataset \cup \{\xy\}})\big)], \\
        \Pr_{\substack{\dataset \sim \datagen^n \\ \xy \sim \datagen}}[\test\big(\xy; \model(\dataset)\big)] &\equiv
        \Pr_{\substack{\dataset \sim \datagen^{n - 1} \\ \xy \sim \datagen \\ \xy' \sim \datagen}}[\test\big(\xy'; \model({\dataset \cup \{\xy\}})\big)]. \\
    \end{aligned}
\end{equation}
It is thus sufficient to analyze the equivalent distributions instead. By the post-processing property of differential privacy, for any dataset $\dataset \in \sD^{n - 1}$, any two examples $\xy, \xy' \in \sD$, and any set $K \subseteq \{0, 1\}$:
\[
    \Pr[ \test\big(\xy; \model({\dataset \cup \{\xy\}})\big) \in K] \leq \Pr[ \test\big(\xy; \model({\dataset \cup \{\xy'\}})\big) \in K] + \delta,
\]
as datasets $\dataset \cup \{\xy\}$ and $\dataset \cup \{\xy'\}$ are neighbouring. Taking the expectation of both sides over $\xy, \xy' \sim \datagen$ and $\dataset \sim \datagen^{n - 1}$, we get:
\begin{equation}\label{eq:dp-bound-1}
\begin{aligned}
    \Pr_{\substack{\dataset \sim \datagen^{n - 1} \\ \xy \sim \datagen}}[ \test\big(\xy; \model({\dataset \cup \{\xy\}})\big) \in K]
    &\leq \Pr_{\substack{\dataset \sim \datagen^{n - 1} \\ \xy \sim \datagen\\ \xy' \sim \datagen}}[ \test\big(\xy; \model({\dataset \cup \{\xy'\}})\big) \in K] + \delta \\
    &= \Pr_{\substack{\dataset \sim \datagen^{n - 1} \\ \xy \sim \datagen\\ \xy' \sim \datagen}}[ \test\big(\xy', \model({\dataset \cup \{\xy\}})\big) \in K] + \delta,
\end{aligned}
\end{equation}
where the last equality is simply renaming of the variables for convenience. Note that analogously we also can obtain a symmetric bound:
\begin{equation}\label{eq:dp-bound-2}
\begin{aligned}
    \Pr_{\substack{\dataset \sim \datagen^{n - 1} \\ \xy \sim \datagen \\ \xy' \sim \datagen}}[ \test\big(\xy', \model({\dataset \cup \{\xy\}})\big) \in K]
    &\leq \Pr_{\substack{\dataset \sim \datagen^{n - 1} \\ \xy \sim \datagen}}[ \test\big(\xy; \model({\dataset \cup \{\xy\}})\big) \in K] + \delta,
\end{aligned}
\end{equation}
The total variation between these two distributions is bounded:
\[
    \begin{aligned}
    & \tv \Big(\Pr_{\substack{\dataset \sim \datagen^{n - 1} \\ \xy \sim \datagen}}[ \test\big(\xy; \model({\dataset \cup \{\xy\}})\big)],
    \Pr_{\substack{\dataset \sim \datagen^{n - 1} \\ \xy \sim \datagen\\\xy' \sim \datagen}}[ \test\big(\xy', \model({\dataset \cup \{\xy\}})\big)] \Big) \\
    = &\sup_{K \subseteq \mathsf{range}(\test)} \Big| \Pr_{\substack{\dataset \sim \datagen^{n - 1} \\ \xy \sim \datagen}}[ \test\big(\xy; \model({\dataset \cup \{\xy\}})\big) \in K]
    - \Pr_{\substack{\dataset \sim \datagen^{n - 1} \\ \xy \sim \datagen \\ \xy' \sim \datagen}}[ \test\big(\xy', \model({\dataset  \cup \{\xy\}})\big) \in K] \Big| \leq \delta,
    \end{aligned}
\]
where the last inequality is by \cref{eq:dp-bound-2}. Using the equivalences in \cref{eq:dist-equiv} we can see that:
\[
    \begin{aligned}
    \tv \Big(\Pr_{\substack{\dataset \sim \datagen^{n} \\ \xy \sim \dataset}}[ \test\big(\xy; \model(\dataset)\big)],
    \Pr_{\substack{\dataset \sim \datagen^{n} \\ \xy \sim \datagen}}[ \test\big(\xy; \model(\dataset)\big)] \Big)
    = \big|\E_\subtrain[\test(\xy; \model(\dataset)] - \E_\subtest[\test(\xy; \model(\dataset)]\big|
    &\leq \delta,
    \end{aligned}
\]
which is the sought result.
\end{proof}

\subsection{Tight Bound on TV-Stability from DP}\label{sec:tv-to-dp}
To prove \cref{stmt:dp-to-tv-tight}, we make use of the hypothesis-testing interpretation of DP~\cite{wasserman2010statistical}. Let us define the hypothesis-testing setup and the two types of errors in hypothesis testing. For any two probability distributions $P$ and $Q$ defined over $\sD$, let $\test: \sD \rightarrow \{0, 1\}$ be a \newterm{hypothesis-testing decision rule} that aims to tell whether a given observation from the domain $\sD$ comes from $P$ or $Q$.

\begin{definition}[Hypothesis-testing FPR and FNR] Without loss of generality, the \newterm{false-positive error rate} $\alpha_\test$ (FPR, or type I error rate), and the \newterm{false-negative error rate} $\beta_\test$ (FNR, or type II error rate) of the decision rule $\test: \sD \rightarrow [0, 1]$ are defined as the following probabilities:
\begin{equation}\label{eq:ht-errors}
\begin{aligned}
    \alpha_\test & \define \Pr_{\xy \sim P}[\test(\xy) = 1] = \E_P[\test], \\
    \beta_\test & \define \Pr_{\xy \sim Q}[\test(\xy) = 0] = 1 - \E_Q[\test].
\end{aligned}
\end{equation}
\end{definition}

A well-known result due to Le Cam provides the following relationship between the trade-off between the two types of errors and the total variation between the probability distributions:
\begin{equation}\label{eq:le-cam}
    \alpha_\test + \beta_\test \geq 1 - \tv(P, Q).
\end{equation}
DP is known to provide the following relationship between FPR and FNR of any decision rule:
\begin{proposition}[\citet{kairouz2015composition}]
Suppose that an algorithm $\train(\dataset)$ satisfies $(\epsilon, \delta)$-DP. Then, for any decision rule $\test: \sD \rightarrow [0, 1]$:
\begin{equation}
\begin{aligned}\label{eq:dp-ht}
    \alpha_\test + \exp(\epsilon) \, \beta_\test \geq 1 - \delta, \\
    \exp(\epsilon) \, \alpha_\test + \beta_\test \geq 1 - \delta.
\end{aligned}
\end{equation}
\end{proposition}
We can now prove \cref{stmt:dp-to-tv-tight}: \begin{proof}[Proof]
Consider a hypothesis-testing setup in which we want to distinguish between the distributions $\train(\dataset)$ and $\train(\dataset')$. Let us sum the two bounds in \cref{eq:dp-ht}:
\begin{equation}\label{eq:combined-kairouz}
\begin{aligned}
    (\exp(\epsilon) + 1) (\alpha_\test + \beta_\test) \geq 2(1 - \delta) \implies \alpha_\test + \beta_\test \geq \frac{2 - 2\delta}{\exp(\epsilon) + 1}.
\end{aligned}
\end{equation}
Let us take the optimal decision rule $\test^*$. In this case, the bound in \cref{eq:le-cam} holds exactly:
\[
    \tv(\train(\dataset), \train(\dataset')) = 1 - (\alpha_{\test^*} + \beta_{\test^*}).
\]
Combining this with \cref{eq:combined-kairouz}, we get:
\[
    \tv(\train(\dataset), \train(\dataset')) \leq 1 - \frac{2 - 2\delta}{\exp(\epsilon) + 1} = \frac{\exp(\epsilon) - 1 + 2\delta}{\exp(\epsilon) + 1}.
\]
\end{proof}

Next, we show that the upper bound is tight:
\begin{proposition}
\label{stmt:dp-to-tv-lb}
There is an algorithm $\train(\dataset)$ satisfying $(\varepsilon, \delta)$-DP, such that $\tv(\train(\dataset), \train(\dataset')) = \frac{\exp(\varepsilon) - 1 + 2\delta}{\exp(\varepsilon) + 1}$ for any two neighbouring datasets $S$ and $S'$.
\end{proposition}
\begin{proof}
    We use the construction of the reduced mechanism by \cite{kairouz2015composition}. Consider a mechanism $\train: \{0, 1\} \rightarrow \{0, 1, 2, 3\}$, defined as follows:
    
    \[
    \begin{matrix*}[l]
        P(\train(0) = 0) = 0 & P(\train(1) = 0) = \delta \\
        P(\train(0) = 1) = (1 - \delta) \cdot \frac{\exp(\epsilon)}{\exp(\epsilon) + 1}
        & P(\train(1) = 1) = (1 - \delta) \cdot \frac{1}{\exp(\epsilon) + 1} \\
        P(\train(0) = 2) = (1 - \delta) \cdot \frac{1}{\exp(\epsilon) + 1} & P(\train(1) = 1) = (1 - \delta) \cdot \frac{\exp(\epsilon)}{\exp(\epsilon) + 1} \\
        P(\train(0) = 3) = \delta & P(\train(1) = 0) = 0 \\
    \end{matrix*}
    \]
    
    Observe that this mechanism satisfies $(\varepsilon, \delta)$-DP, and $\tv(\train(0), \train(1)) = \frac{\exp(\varepsilon) - 1 + 2\delta}{\exp(\varepsilon) + 1}$.
\end{proof}

\subsection{Privacy Analysis of \dpis}
First, we present a loose analysis of the privacy guarantees of non-uniform Poisson subsampling.
\begin{lemma}\label{stmt:dp-dpis}
Suppose that $\train(\dataset)$ satisfies $(\epsilon, \delta)$-DP and $\mathsf{Sample}(\dataset)$ is a Poisson sampling procedure where each of the sampling probabilities $p_i$ depend on the element $\xy_i$ (but do not depend on the set $S$ otherwise) and is guaranteed to satisfy $p_i \leq p^*$. Then $\train \circ \mathsf{Sample}$ satisfies $(\ln(1 - p^* + p^* e^\epsilon), p^* \delta)$-DP. For small $\epsilon$ this can be bounded by $(\mathcal{O}(p^* \epsilon), p^* \delta)$-DP.
\end{lemma}

\begin{proof}[Proof of \cref{stmt:dp-dpis}]
Consider two neighboring datasets $S$ and $S' = S \cup \{\xy_0\}$ for some $\xy_0 \not\in S$. We wish to show that for any set $K$, we have
\begin{equation*}
    \Pr(\train(\mathsf{Sample}(S')) \in K) \leq (1 - p + pe^{\epsilon}) \Pr(\train(\mathsf{Sample}(S)) \in K) + p \delta
\end{equation*}
and symmetrically for $S$ and $S'$. We will only prove first of those inequalities, as the second is analogous.

Note that with probability $p_0 \leq p$ the element $\xy_0$ is included in $\mathsf{Sample}(S')$ and we have $\mathsf{Sample}(S') = \{ \xy_0 \} \cup \mathsf{Sample}(S)$, otherwise the element $\xy_0$ is not included, and conditioned on $\xy_0$ not being included $\mathsf{Sample}(S')$ has the same distribution as $\mathsf{Sample}(S)$. Therefore,
\begin{equation}
\label{eq:sampling}
    \Pr(\train(\mathsf{Sample}(S')) \in K) = p_0 \Pr(\train(\{\xy_0\} \cup \mathsf{Sample}(S)) \in K) + (1 - p_0) \Pr(\train(\mathsf{Sample}(S)) \in K).
\end{equation}
Now for each realization $\mathsf{Sample}(S) = \tilde{S}$, we have $\Pr(\train(\{\xy_0\} \cup \tilde{S}) \in K) \leq e^{\epsilon} \Pr(\train(\tilde{S}) \in K) + \delta$ by the assumed DP guarantee of the algorithm $\train(S)$. We can average over all possible subsets $\tilde{S}$ to get
\begin{align*}
    \Pr(\train(\{\xy_0\} \cup \mathsf{Sample}(S)) \in K) & = \sum_{\tilde{S}} \Pr(\mathsf{Sample}(S) = \tilde{S}) \Pr(\train(\{\xy_0\} \cup \tilde{S}) \in K) \\
    & \leq \sum_{\tilde{S}} \Pr(\mathsf{Sample}(S) = \tilde{S}) (e^{\epsilon} \Pr(\train(\tilde{S}) \in K) + \delta)\\
    & = e^{\epsilon} \Pr(\train(\mathsf{Sample}(S)) \in K) + \delta.
\end{align*}
Plugging this back to the inequality \eqref{eq:sampling}, we get
\begin{align*}
    \Pr(\train(\mathsf{Sample}(S')) \in K) & \leq p_0 (e^{\epsilon} \Pr(\train(\mathsf{Sample}(S)) \in K) + \delta) + (1 - p_0) \Pr(\train(\mathsf{Sample}(S)) \in K) \\
    & \leq (1 - p^* + p^* e^{\epsilon}) \Pr(\train(\mathsf{Sample}(S)) \in K) + p^* \delta.
\end{align*}
Finally, when $\epsilon \leq 1$ we have $e^{\epsilon} \leq (1 + 2 \epsilon)$, and therefore $(1 - p^* + p^* e^{\epsilon}) \leq 1 + 2 \epsilon p^* \leq e^{2 \epsilon p^*}$.
\end{proof}

For the tight privacy analysis of non-uniform Poisson subsampling, we make use of the notion of $f$-privacy:
\begin{definition}[$f$-Privacy~\citet{dong2019gaussian}]
An algorithm $\train(\dataset)$ satisfies $f$-privacy if for any two neighbouring datasets $\dataset, \dataset'$ the following holds:
\[
    \tau(\train(\dataset), \train(\dataset')) \geq f,
\]
where $\tau(P, Q)$ is a trade-off function between the FPR and FNR of distinguishing tests (see \cref{sec:tv-to-dp}):
\begin{equation}
    \tau(P, Q)(\alpha) = \inf_{\phi:~\sD \rightarrow [0, 1]} \{\beta_\phi: \alpha_\phi \leq \alpha\},
\end{equation}
and $f(\alpha) \in [0, 1]$ is a convex, continuous, non-increasing function.
\end{definition}

\newcommand{\Id}{\mathsf{Id}}

\citet{bu2020deep} show that uniform Poisson subsampling (see \cref{sec:dpis}) provides the following privacy amplification:
\begin{proposition}[\citet{bu2020deep}]\label{stmt:unif-amp}
Suppose that $\train(\dataset)$ satisfies $f$-privacy, and $\mathsf{Sample}(\dataset)$ is a uniform Poisson sampling procedure with sampling probability $\samplerate$. The composition $\train \circ \mathsf{Sample}(\dataset)$ satisfies $f'$-privacy with $f' = \samplerate f + (1 - \samplerate) \Id$, where $\Id(\alpha) = 1 - \alpha$ is the trade-off function that corresponds to perfect privacy.
\end{proposition}

We show that a similar result holds for non-uniform Poisson subsampling:
\begin{lemma}\label{stmt:non-unif-amp}
Suppose that $\train(\dataset)$ satisfies $f$-privacy, and $\mathsf{Sample}(\dataset)$ is a non-uniform Poisson sampling procedure, where the sampling probabilities $p_i$ depend on the element $\xy_i$ (but do not depend on the set $S$ otherwise) and each is guaranteed to satisfy $p_i \leq p^*$. The composition $\train \circ \mathsf{Sample}(\dataset)$ satisfies $f'$-privacy with $f' = p^* + (1 - p^*) \Id$.
\end{lemma}

To show this, we adapt the proof \cref{stmt:unif-amp}, and make use of the following lemma:
\begin{lemma}[\citet{bu2020deep}]\label{stmt:amp-lemma}
Let $\{P_i\}_{i \in I}$ and $\{Q_i\}_{i \in I}$ be two collections of probability distributions on the same sample space for some index set $I$. Let $(\lambda_i)_{i \in I} \in [0, 1]^{|I|}$ be a collection of numbers such that $\sum_{i \in I} \lambda_i = 1$. If $\tau(P_i, Q_i) \geq f$ for all $i \in I$, then for any $p \in [0, 1]$:
\[
    \tau\left(\sum_{i} \lambda_i \cdot P_i,\ \sum_{i} (1 - p) \cdot \lambda_i \cdot P_i + \sum_{i} p \cdot \lambda_i \cdot Q_i \right)
    \geq p f + (1 - p) \Id.
\]
\end{lemma}

\begin{proof}[Proof of \cref{stmt:non-unif-amp}] We can think of the result of the subsampling procedure as outputting a binary vector $\vb = (b_1, \ldots, b_n) \in \{0, 1\}^n$, where each bit $b_i$ indicates whether an example $\xy_i \in \dataset$ was chosen in the subsample or not. We denote the resulting subsample as $\dataset_\vb \subseteq \dataset$. By definition of Poisson subsampling, each bit $b_i$ is an independent sample $b_i \sim \mathsf{Bern}(p_i)$. Let us denote by $\lambda_\vb$ the joint probability of $\vb$. The composition $\train(\dataset) \circ \mathsf{Sample}(\dataset)$ can be expressed as a mixture distribution:
\[
    \train(\dataset) \circ \mathsf{Sample}(\dataset) = \sum_{\vb \in \{0, 1\}^n} \lambda_{\vb} \cdot \train(\dataset).
\]
Analogously, for a neighbouring dataset $\dataset'~\define~ \dataset \cup \{\xy_0\}$, with the sampling probability $p_0$ corresponding to $\xy_0$, we have:
\[
    \train(\dataset) \circ \mathsf{Sample}(\dataset) = \sum_{\vb \in \{0, 1\}^n} p_0 \cdot \lambda_{\vb} \cdot \train(\dataset'_{\vb} \cup \{\xy_0\}) + \sum_{\vb \in \{0, 1\}^n} (1 - p_0) \cdot \lambda_{\vb} \cdot \train(\dataset_{\vb}).
\]
Applying \cref{stmt:amp-lemma}, we get $f_0$-privacy with $f_0 = p_0 f + (1 - p_0) \Id$. Applying to an arbitrary other $\xy_0 \in \sD$, we potentially get the worst-case privacy guarantee for the highest sampling probability, i.e., $f = p^* f + (1 - p^*) \Id$. 
\bnote{Formally, probably need to convexify the collection of $f_i$}
\end{proof}

\cref{stmt:dpiw-blackbox} is immediate from \cref{stmt:non-unif-amp} by the fact that GDP is a special case of $f$-privacy.

\section{Stronger Formalizations of Distributional Generalization}\label{app:strong-dg}
\jnote{I'm writing the math so far, will work on wording --- it's far from final... In particular, we can probably get rid of repeating the definitions of Robust Generalization and $\delta$-DG... Although it's nice to have them next to each other to compare those.}

The main notion discussed in this paper is that of $\delta$-Distributional Generalization. To recap, a learning algorithm $\train(S)$ satisfies it if for any bounded test function $\phi$ we have
\begin{equation*}
\Big|\E_{\substack{\dataset \sim \mathcal{D}^n, \train \\ \xy \sim \dataset}} \test\big(\xy; \train(\dataset)\big) -
    \E_{\substack{\dataset \sim \mathcal{D}^n, \train \\ \xy \sim \dataset}} \test\big(\xy; \train(\dataset)\big)\Big| \leq \delta.
\end{equation*}

In Appendix~\ref{app:related} we have compared it to a stronger property called $(\gamma, \eta)$-Robust Generalization; namely learning algorithm $\train(S)$ satisfies it if for any bounded $\phi$ we have:
\begin{equation*}
    \Pr_{\substack{\dataset \sim \mathcal{D}^n, \train}}\left( \big|\E_{\xy \sim \dataset} \test\big(\xy; \train(\dataset)\big) -
    \E_{\xy \sim \mathcal{D}} \test\big(\xy; \train(\dataset)\big)\big| > \gamma\right) \leq \eta.
\end{equation*}

A natural middle ground between these two is the following notion. 
\begin{definition}
\label{def:strong-dg}
We say that a learning algorithm $\train(\dataset)$ satisfies \emph{$\delta$-strong Distributional Generalization} if and only if for all $\test: \sD \times \Theta \rightarrow [0, 1]$:
\begin{equation*}
\E_{\dataset \sim \mathcal{D}^n, \train} \big|\E_{\xy \sim \dataset} \test\big(\xy; \train(\dataset)\big) -
    \E_{\xy \sim \mathcal{D}} \test\big(\xy; \train(\dataset)\big)\big| \leq \delta.
\end{equation*}
\end{definition}
If DG as defined in \cref{sec:dg-theory} is a (supremum over all test functions) version of expected generalization error~\cite{shalev2010learnability, xu2017information, russo2019much}, strong DG is an analogous variant of \newterm{expected absolute error}~\cite{steinke2020reasoning}.

Clearly any algorithm satisfying $(\eta, \gamma)$-Robust Generalization satisfies also $(\eta + \gamma)$-strong DG, and any algorithm satisfying $\delta$-strong-DG satisfies $\delta$-DG. Thus, $\delta$-strong DG can be thought of as a strengthening of Distributional Generalization as well as Robust Generalization.

One of the motivations behind the definition of $\delta$-strong DG is the following issue of the Distributional Generalization itself. As it turns out, it is possible to artificially construct a mechanism for binary classification which satisfies $0$-DG, always outputs a classifier with test error $\frac{1}{2} + o(1)$ (almost completely useless), but with probability $1/2$ outputs a classifier with training error $0$ (and with probability $1/2$ the training error is $1$). That is problematic: it is natural to expect from a training procedure with strong generalization guarantee, that if we used it to train a model up to a small training error, we should have high confidence that the test error is also small --- but this example shows that our confidence cannot be larger than $1/2$.

The $\delta$-strong DG solves this problem: The following proposition is a formal statement of the desired property, and can be easily deduced from Markov inequality.
\begin{proposition}
If the learning procedure $\train(S)$ satisfies $\delta$-strong-DG, then for any bounded  $\phi(z; \params)$ we have:
\[
    \Pr_{S, \train} \left(|\E_{z\sim S} \phi(z, \train(S)) - \E_{z\sim\mathcal{D}} \phi(z, \train(S))| > \delta \lambda \right) < \lambda^{-1}.
\]

Equivalently, if a learning procedure satisfies $\delta$-strong-DG it also satisfies $(\delta \lambda, \lambda^{-1})$-Robust Generalization for any $\lambda > 1$.
\end{proposition}

The $\delta$-strong DG defined this way is also a more direct strengthening of the classical notions of generalization as expected absolute difference, which are usually defined as $\E_{S\sim\mathcal{D}^n, \train} | \E_{z\sim S} \ell(z, \train(S)) - \E_{z\sim \mathcal{D}} \ell(z, \train(S))|$ where $\ell$ is the loss function.

We can follow the same proof strategy as in~\citep[Theorem 7.2]{bassily2016algorithmic}, simplifying it significantly, to show that TV stability (and in turn Differential Privacy) implies strong Distributional Generalization, as soon as the training set is large enough.
\begin{lemma}
If a learning algorithm $\train(S)$ satisfies $\delta$-TV stability, and the sample size $|S| \geq \frac{2}{\delta^2}$, then $\train(S)$ satisfies $4 \delta$-strong DG.
\end{lemma}
\begin{proof}
Let us assume first that we have a $\delta$-TV stable mechanism $\mathcal{M} : \sD^n \to \mathcal{Q}$ where $\mathcal{Q}$ is a set of bounded queries $q : \sD \to \sR$ satisfying $|\sup_x(q(x)) - \inf_x(q(x))| \leq 1$. 

We wish to find another $2 \delta$-TV stable mechanism $\tilde{\mathcal{M}} : \sD^n \to \mathcal{Q}$ such that if we use notation $q_S ~\define~ \mathcal{M}(S)$ and $\tilde{q}_S ~\define~ \tilde{\mathcal{M}}(S)$, we have
\begin{equation}\E_{S, M} \left|\E_{z\sim \mathcal{D}} q_S(z) - \E_{z \sim S} q_S(z)|\right| \leq 
    2 \left(\E_{S, \tilde{M}, z \sim S} \tilde{q}_S(z) - \E_{S, \tilde{M}, z \sim \mathcal{D}} \tilde{q}_S(z)\right). \label{eq:dg-leq-sdg}
\end{equation}

Consider the following mechanism $\tilde{\mathcal{M}}(S)$: For an input dataset $S$, obtain $q_S ~\define~ \mathcal{M}(S)$, and evaluate $\tilde{\mu} ~\define~ \E_{z \sim S} q_S(z) + \frac{1}{\delta n} \mathbf{U}$, where $\mathbf{U}$ is a uniform random variable in $[-\nicefrac{1}{2}, \nicefrac{1}{2}]$. Finally, if $\tilde{\mu} \geq \E_{z \sim D} q_S(z)$, then output $q_S$, otherwise output $-q_S + 2 \E_{z \sim \mathcal{D}} q_S(z)$.

As $\tv(\eta \mathbf{U}, \gamma + \eta \mathbf{U}) \leq \frac{\gamma}{\eta}$, a mechanism $\mu_q(S) ~\define~ \E_{z \sim S} q(z) + \frac{1}{\delta n} \mathbf{U}$ is $\delta$-stable. Hence, by composition and post-processing, the mechanism $\tilde{\mathcal{M}}$ is $2\delta$-TV stable.

For $n \geq 2\delta^{-2}$ the inequality~\eqref{eq:dg-leq-sdg} follows from Lemma~\ref{lem:flip-small-values}. Indeed, let $\hat{Z} ~\define~ \E_{z \sim S} q_S(z) - \E_{z \sim \mathcal{D}} q_S(z)$, $Z = |\hat{Z}|$, and $\sigma = \mathrm{sign}(\hat{Z} + \frac{1}{\delta n} \mathbf{U})$. Note that, expanding the chosen notation, we have $\E_{z \sim S} \tilde{q}_S(z) - \E_{z\sim \mathcal{D}}\tilde{q}_S(z) = \sigma Z$, and $\E [\sigma | Z] = \max \{\frac{2Z}{\delta}, 1 \}$. The conclusion of the Lemma~\ref{lem:flip-small-values} is $\E Z \leq 2 \E [Z \sigma]$, which is exactly~\eqref{eq:dg-leq-sdg}.

As the mechanism $\tilde{\mathcal{M}}$ satisfies $2\delta$-TV stability, by Theorem~\ref{stmt:tv-to-dg} applied to the mechanism $\tilde{\mathcal{M}}$, we can bound the right hand side of the inequality~\eqref{eq:dg-leq-sdg} by $4\delta$.

Finally, for an arbitrary learning algorithm $\train(S)$ satisfying $\delta$-TV stability, and any bounded test function $\phi(z; \params)$, we can describe mechanism $\tilde{\mathcal{M}}$ which, given a sample $S \sim \mathcal{D}^n$, outputs a function $q_S(z) ~\define~ \phi(z; \train(S))$. By post-processing, such a mechanism $\mathcal{M}$ also satisfies $\delta$-TV stability, and the inequality~\eqref{eq:dg-leq-sdg} becomes
\begin{equation*}
    \E_{S, M} | \E_{z\sim S} \phi(z; \train(S)) - \E_{z\sim \mathcal{D}} \phi(z; \train(S))| \leq 4 \delta.
\end{equation*}
As $\phi(z; \params)$ was chosen in an arbitrary way, this proves that mechanism $\train(S)$ indeed satisfies the condition of $4\delta$-strong-DG.
\end{proof}

\begin{lemma}
\label{lem:flip-small-values}
Let $Z$ be a non-negative random variable satisfying $\E Z \geq \delta$ and $\sigma$ be a $\{\pm 1\}$ valued random variable with $\E[\sigma | Z] = \max \{ \frac{2 Z}{\delta}, 1 \}$. Then $\E Z \leq 2 \E [Z \sigma]$.
\end{lemma}
\begin{proof}
This is just a calculation.
\begin{equation*}
    \E Z \sigma = \E \1_{Z\geq \delta/2} Z + \E \1_{Z < \delta/2} \frac{2 Z^2}{\delta} \geq \E \1_{Z\geq \delta/2} Z = \E Z - \E \1_{Z < \delta/2} Z \geq \E Z - \delta/2 \geq \frac{1}{2} \E Z.
\end{equation*}
\end{proof}

\subsection{Subgroup-level Distributional Generalization from TV Stability}
\label{app:subgroup}
TV-stability implies a more granular, subgroup-level notion of distributional generalization:
\begin{definition}
    Suppose that the data distribution $\datagen$ is a mixture of group-specific distributions $\datagen_\group$, for $\group \in \setgroups$. We define $(\delta, \setgroups)$-subgroup-DG similarly to $\delta$-DG as follows:
    \[
        \forall \test, \group \in \setgroups: \quad \Big|\E_{\substack{\dataset \sim \datagen^n \\ \xy \sim \dataset_\group}}[\test(\xy, \train(\dataset)) \mid |\dataset_\group| > 0]
                        - \E_{\substack{\dataset \sim \datagen^n \\ \xy \sim \datagen_\group}}[\test(\xy, \train(\dataset))]\Big| \leq \delta,
    \]
    where $\dataset_\group$ denotes a subset of examples in the dataset $\dataset$ that belong to the group $\group$. 
\end{definition}
Subgroup DG is a stronger notion of DG which says that the model's behavior on examples from each group in $\setgroups$ distributionally generalizes in expectation, as long as the model encounters at least one representative of the group in training. In its definition, we explicitly prevent the case when the training dataset does not contain any group examples to avoid undefined behavior (what is the group accuracy on the training dataset if there are no group representatives in the dataset?).

We now show that TV-stability implies this granular notion of DG:

\begin{proposition}
    $\delta$-TV stability implies $(\delta, \setgroups)$-subgroup-DG for any group partitioning $\setgroups$.
\end{proposition}

\begin{proof}
Observe that the following distributions are equivalent:
\begin{equation}\label{eq:dist-equiv-subgroup}
    \begin{aligned}
        \Pr_{\substack{\dataset \sim \datagen^n \\ \xy \sim \dataset_\group}}[\test\big(\xy; \model(\dataset)\big) \mid |\dataset_\group| > 0] &\equiv
        \Pr_{\substack{\dataset \sim \datagen^{n - 1} \\ \xy \sim \datagen_\group}}[\test\big(\xy; \model({\dataset \cup \{\xy\}})\big)], \\
        \Pr_{\substack{\dataset \sim \datagen^n \\ \xy \sim \datagen_\group}}[\test\big(\xy; \model(\dataset)\big)] &\equiv
        \Pr_{\substack{\dataset \sim \datagen^{n - 1} \\ \xy \sim \datagen \\ \xy' \sim \datagen_\group}}[\test\big(\xy'; \model({\dataset \cup \{\xy\}})\big)]. \\
    \end{aligned}
\end{equation}
The statement follows by applying each step from the proof of \cref{stmt:tv-to-dg} to the equivalent distributions in \cref{eq:dist-equiv-subgroup}. Thus, the difference with the proof of \cref{stmt:tv-to-dg} is that $\xy, \xy' \sim \datagen_\group$, not $\xy, \xy' \sim \datagen$.
\end{proof}

\subsection{Calibration Generalization via Strong Distributional Generalization}
\label{app:calibration}
We show that strong Distributional Generalization (Definition \ref{def:strong-dg}) implies generalization of calibration properties. 
For simplicity we consider binary classifiers, 
with $\sY = \{0, 1\}$ and $f_\params:  \sX \to [0, 1]$.
Let us define the notion of calibration we consider:
\begin{definition}
The \emph{calibration gap}, or Expected Calibration Error (ECE), of a classifier $f_\params$ is
defined as follows:
\begin{align*}
\cgap(f_\params)
&\define
\E\left[\left\vert \E[y \mid f_\params(x) ] - f_\params(x) \right\vert\right]\\
&=
\int_{0}^1
\left\vert
\E\left[
y - p
\mid f_\params(x) = p
\right]
\right\vert
d F(p)
\tag{where $F(p)$ is the law of $f_\params(x)$}, \\
\end{align*}
where all expectations are taken w.r.t. $(x, y) \sim \datagen$.
\end{definition}

To estimate this quantity from finite samples, we will discretize the interval $[0,1]$. The \emph{$\tau$-binned calibration gap} of a classifier $f_\params$ is:
\[
\begin{aligned}
\cgap_\tau(f_\params; \datagen) &\define
\sum_{p \in \{0, \tau, 2\tau, \dots, 1\}}
\left\vert
\E_{(x, y) \sim \datagen}\left[
\1\{f_\params(x) \in (p, p+\tau)\}
\cdot
(y - p)
\right]
\right\vert
\end{aligned}
\]
So that $\cgap_\tau \to \cgap$ as the bin width $\tau \to 0$.
The empirical version of this quantity, for train set $S$, is:

\begin{align*}
\cgap_\tau(f_\params; S) &\define
\sum_{p \in \{0, \tau, 2\tau, \dots, 1\}}
\left\vert
\E_{(x, y) \sim S}\left[
\1\{f_\params(x) \in (p, p+\tau)\}
\cdot
(y - p)
\right].
\right\vert
\end{align*}

Next, we show strong DG implies that the expected binned calibration gap is similar between train and test.

\begin{theorem}
If the training method $\train(\dataset)$ satisfies $\delta$-strong DG, then
\begin{equation}
\left| \E_{\train, S\sim\mathcal{D}^n} \cgap_{\tau}(f_{\train(S)}; S) - \cgap_{\tau}(f_{\train(S)}; \mathcal{D})\right| \leq \frac{\delta}{\tau}.
\end{equation}
\end{theorem}
\begin{proof}
First, let us define the family of tests:
\[
\begin{aligned}
\phi_p((x,y); \params)
~\define~
\1\{f_\params(x) \in (p, p+\tau)\}
\cdot
(y - p)
\end{aligned}
\]
The assumed $\delta$-strong DG implies that each of these tests $\phi_p$ have similar expectations between
train and test distributions:
\begin{align*}
\forall p:\quad
\left|
\E_{S \sim \datagen^n; (x, y) \sim S}[
\phi_p((x,y); \train(S))
]
-
\E_{S \sim \datagen^n; (x, y) \sim \datagen}[
\phi_p((x,y); \train(S))
]
\right|
\leq \delta
\end{align*}

Thus, the difference in expected calibration gaps between train and test is:
\begin{align*}
&
|\E_{S \sim \datagen^n}[
\cgap_\tau(f_\params; S)
-
\cgap_\tau(f_\params; \datagen) 
]|\\
&=
\left|
\E_{S \sim \datagen^n}
\sum_{p \in \{0, \tau, 2\tau, \dots, 1\}}
\left|\E_{x, y \sim S}[\phi_p((x, y); \params)]\right|
-
\left|\E_{x, y \sim \datagen}[\phi_p((x, y); \params)]\right|
\right|\\
&\leq
\sum_{p \in \{0, \tau, 2\tau, \dots, 1\}}
\E_{S \sim \datagen^n}
\left|\E_{x, y \sim S}[\phi_p((x, y); \params)
-
\E_{x, y \sim \datagen}[\phi_p((x, y); \params)]
\right|\\
&\leq
\frac{\delta}{\tau},
\end{align*}
where the last inequality follows from the definition of $\delta$-strong DG applied to functions $\phi_p$.
\end{proof}

\section{Additional Details on Algorithms}\label{app:algo}

We define $q_g$ as the probability of group $g$, and $m$ as the number of groups.

\paragraph{\iserm.}
    The weight for group $g$ is $w_g = \nicefrac{1}{m \cdot q_g}$.
    Let $g_i$ be the group that the $i$-th example belongs to.
    We then sample (with replacement) from the training set with the $i$-th example having a $w_{g_i}$ chance of being sampled until we have $b$ examples, where $b$ is the batch size.
    Finally, for each mini-batch, we optimize the standard cross-entropy loss with the sampled examples.

\paragraph{\iwerm.}
    The weight for group $g$ is $w_g = \nicefrac{1}{m \cdot q_g}$.
    We optimize the following loss function:
    \[
        w_{g} \cdot \loss(\xy; \params),
    \] where $\ell(\xy, \params)$ is the cross-entropy loss and $\xy \in \dataset$ drawn uniformly random drawn from the dataset, and $g$ is the group to which $\xy$ belongs.

\section{Additional Experiment Details}\label{app:experiment}

\subsection{Details on Datasets, Software, and Model Training}
\label{app:technical}

\begin{table}[h!]
    \footnotesize
    \centering
    \caption{The number of examples in each subgroup for CelebA.}
    \label{tab:CelebA_detail}
    \begin{tabular}{lrrr}
    \toprule
                        & training & validation & testing \\
    \midrule
    not blond, female   & 71629    & 8535       & 9767    \\
    not blond, male     & 66874    & 8276       & 7535    \\
    blond, female       & 22880    & 2874       & 2480    \\
    blond, male       & 1387     & 182        & 180    \\
    \bottomrule
    \end{tabular}
\end{table}

\begin{table}[h!]
    \footnotesize
    \centering
    \caption{The number of examples in each subgroup for UTKFace.}
    \label{tab:UTKFace_detail}
    \begin{tabular}{lrrr}
    \toprule
                     & training & validation & testing \\
    \midrule
    male, White      & 3919    & 454       & 1105    \\
    male, Black      & 1700    & 181       & 437    \\
    male, Asian      & 1115    & 157       & 303    \\
    male, Indian     & 1594    & 190       & 477    \\
    male, Others     & 563     & 61        & 136    \\
    female, White    & 3316    & 384       & 902    \\
    female, Black    & 1606    & 188       & 414    \\
    female, Asian    & 1302    & 158       & 399    \\
    female, Indian   & 1230    & 152       & 333    \\
    female, Others   & 655     & 75        & 202    \\
    \bottomrule
    \end{tabular}
\end{table}

\begin{table}[h!]
    \footnotesize
    \centering
    \caption{The number of examples in each subgroup for iNat.}
    \label{tab:iNaturalist_detail}
    \begin{tabular}{lrrr}
    \toprule
                     & training & validation & testing \\
    \midrule
    Actinopterygii   & 2112    & 195       & 312     \\
    Amphibia         & 14531   & 1242      & 1930    \\
    Animalia         & 5362    & 491       & 737     \\
    Arachnida        & 4838    & 461       & 660     \\
    Aves             & 191773  & 17497     & 26251   \\
    Chromista        & 435     & 52        & 55      \\
    Fungi            & 6148    & 575       & 883     \\
    Insecta          & 96894   & 8648      & 13013   \\
    Mammalia         & 26724   & 2475      & 3624    \\
    Mollusca         & 7627    & 693       & 1057    \\
    Plantae          & 159843  & 14653     & 22117   \\
    Protozoa         & 309     & 25        & 37      \\
    Reptilia         & 33404   & 2983      & 4494    \\
    \bottomrule
    \end{tabular}
\end{table}

\begin{table}[h!]
    \footnotesize
    \centering
    \caption{The number of examples in each subgroup for CivilComments.}
    \label{tab:civil_detail}
    \begin{tabular}{lrrr}
    \toprule
    & training & validation & testing \\
    \midrule
    Non-toxic, Identity  &  94895  &   15759  &  46185          \\
    Non-toxic, Other   & 143628  &  24366   & 72373          \\
    Toxic, Identity &   18575  &   3088  &    9161        \\
    Toxic, Other  &  11940  &   1967  &   6063       \\
    \bottomrule
    \end{tabular}
\end{table}

\begin{table}[h!]
    \footnotesize
    \centering
    \caption{The number of examples in each subgroup for MNLI.}
    \label{tab:mnli_detail}
    \begin{tabular}{lrrr}
    \toprule
    & training & validation & testing \\
    \midrule
    Contradiction, No negation  & 57498  & 22814   &  34597     \\
    Contradiction, Negation   &  11158 &  4634  &  6655      \\
    Entailment, No negation   & 67376  &  26949  &   40496    \\
    Entailment, Negation     & 1521  & 613   &    886   \\
    Neutral,  No negation & 66630  &  26655  &  39930     \\
    Neutral,  Negation  & 1992  &  797  &  1148     \\
    \bottomrule
    \end{tabular}
\end{table}

\begin{table}[h!]
    \footnotesize
    \centering
    \caption{The number of examples in each subgroup for ADULT.}
    \label{tab:adult_detail}
    \begin{tabular}{lrrr}
    \toprule
    & training & validation & testing \\
    \midrule
    Female, income$\leq$50k  & 11763  & 911   &  1749     \\
    Male, income$\leq$50k   &  18700 &  1373  &  2659      \\
    Female, income$>$50k   & 1444  &  105  &  220    \\
    Male, income$>$50k     & 8093  & 611   &  1214   \\
    \bottomrule
    \end{tabular}
\end{table}

\begin{table}[t]
    \footnotesize
    \centering
    \caption{The accuracy for each subgroup on CelebA. These results are acquired without any regularization or early stopping (trained on full $50$ epochs).}
    \label{tab:CelebA}
    \begin{tabular}{llcccc}
\toprule
          &      & \multicolumn{2}{c}{blond} & \multicolumn{2}{c}{not blond} \\
          &      & female & male &    female & male \\
\midrule
\multirow{2}{*}{\erm} & train &   1.00 & 0.99 &      1.00 & 1.00 \\
          & test &   0.80 & 0.42 &      0.97 & 1.00 \\
\cline{1-6}
\multirow{2}{*}{\iwerm} & train &   0.98 & 0.99 &      0.98 & 0.99 \\
          & test &   0.87 & 0.49 &      0.95 & 0.98 \\
\cline{1-6}
\multirow{2}{*}{\iserm} & train &   1.00 & 1.00 &      1.00 & 1.00 \\
          & test &   0.83 & 0.38 &      0.96 & 0.99 \\
\cline{1-6}
\multirow{2}{*}{\dpsgd} & train &   0.80 & 0.41 &      0.96 & 0.99 \\
          & test &   0.74 & 0.29 &      0.98 & 1.00 \\
\cline{1-6}
\multirow{2}{*}{\dpis} & train &   0.94 & 0.96 &      0.88 & 0.90 \\
          & test &   0.92 & 0.85 &      0.91 & 0.92 \\
\bottomrule
\end{tabular}
\end{table}

All algorithms are implemented in \texttt{PyTorch}\footnote{Code and license can be found in \url{https://github.com/pytorch/pytorch}.}~\citep{paszke2019pytorch}.
For DP-related utilities, we use \texttt{opacus}\footnote{Code and license can be found in \url{https://github.com/pytorch/opacus}.}~\citep{opacus}.
Other packages, including
\texttt{numpy}~\footnote{Code and license can be found in \url{https://github.com/numpy/numpy}}~\citep{harris2020array},
\texttt{scipy}~\footnote{Code and license can be found in \url{https://github.com/scipy/scipy}}~\citep{2020SciPy-NMeth},
\texttt{tqdm}~\footnote{Code and license can be found in \url{https://github.com/tqdm/tqdm}}, and
\texttt{pandas}~\footnote{Code and license can be found in \url{https://github.com/pandas-dev/pandas}}~\citep{reback2020pandas},
are also used.
For \gdro~\citep{sagawa2019distributionally}, we use the implementation from \texttt{wilds}~\citep{koh2021wilds}.
We use Nvidia 2080ti, 3080, and A100 GPUs.
Our experiments required approximately 400 hours of GPU time.

\paragraph{Datasets.}
For CelebA and CivilComments, we follow the training/validation/testing split in \citet{koh2021wilds}.
For UTKFace and iNat, we randomly split the data into 17000/2000/4708 and 550000/50000/75170 for training/validation/testing.
For MNLI, we use the same training/validation/testing split in \citet{sagawa2019distributionally}.
For Adult~\cite{kohavi1996scaling}, we randomly split the data into 35000/3000/5842 for training/validation/testing.
\cref{tab:CelebA_detail,tab:UTKFace_detail,tab:iNaturalist_detail,tab:civil_detail,tab:mnli_detail,tab:adult_detail} show the dataset statistics on each group.

All the datasets are publicly available for non-commercial use. In our work, we adhere to additional rules regulating the use of each dataset.
All datasets other than iNat could potentially contain personally identifiable information, and are likely collected without consent, to the best of our knowledge. They are all, however, collected from manifestly public sources, such as public posts on social media. Thus, we consider the associated privacy risks low.

The data also contain offensive material (e.g., explicitly in the case of CivilComments dataset). We consider the associated risks of reproducing the offensive behavior low, as we use the datasets only to evaluate our theoretical and theoretically-inspired results.

\paragraph{Models.}
Similar to previous work~\citep{sagawa2019distributionally}, we use the ImageNet-1k pretrained ResNet50~\cite{he2016deep} from \texttt{torchvision} for CelebA, UTKFace, and iNat, and use the pretrained BERT-Base~\cite{devlin2018bert} from \texttt{huggingface}~\citep{wolf2019huggingface} for CivilComments and MNLI.

For ADULT, we follow the setup in \cite{xu2021removing} and use logistic regression with standard optimization, and DP-based training methods.
We fix the batch size to $256$ (for SGD), weight decay to $0.01$, and number of epochs to $20$. For the DP algorithms, we use gradient norm clipping to $0.5$, and sampling rate of $0.005$.
For all training algorithms, we train five model times with different random seeds and we record the mean and standard error of the mean of our metrics. 
The noise parameter $\sigma$ for DP-SGD-F and DP-SGD is set to $1.0$, and we set the $\sigma$ for DP-IS-SGD to $5.0$ to achieve similar privacy budget $\epsilon \approx 0.7$.
The additional noise parameter for DP-SGD-F $\sigma_2$ is set to $10\sigma$ as in \citet{xu2021removing}.

\paragraph{Hyperparameters.}
We run $50$ epochs for CelebA, $100$ epochs for UTKFace, $20$ epochs for iNat, and $5$ epochs for CivilComments and MNLI.
For image datasets (CelebA, UTKFace, and iNat), we use the SGD optimizer and for NLP datasets (CivilComments and MNLI), we use the AdamW~\citep{loshchilov2017decoupled} optimizer.
We use \texttt{opacus}'s~\cite{opacus} implementation of \dpsgd and DP-AdamW to achieve DP guarantees.

We fix the batch size for none-DP algorithms to $64$ for CelebA and UTKFace, $256$ for iNat, $16$ for CivilComments, and $32$ for  MNLI.
For \dpsgd and \dpis, we set the sample rate to $0.0001$ for CelebA and iNat, $0.001$ for UTKFace, and 0.00005 for CivilComments and MNLI.

\subsection{Generalization of Worst-Case Group Accuracy as a Proxy for the DG Gap}
\label{app:empirical-dg}

Although generalization of worst-case group accuracy is not explicitly implied by DG, in our experiments it is practically equivalent to using the generalization gap of subgroup accuracy, which is bounded by TV stability. Let us first concretely define the generalization gap of the worst-case group accuracy:
\begin{definition}
The on-average generalization gap of the worst-case accuracy is defined as the following difference:
\begin{equation}\label{eq:worst-group-gen}
    \wggap~ \define \E_{\dataset \sim \datagen^n}\left[\max_{\group \in \setgroups} \E_{\xy \sim \dataset_\group}[\ell(\xy, \theta(\dataset))] ~\Big|~ |\dataset_{\group}| > 0\right] - \E_{\dataset \sim \datagen^n}\left[\max_{\group \in \setgroups} \E_{\xy \sim \datagen_\group} [\ell(\xy, \theta(\dataset))]\right],
\end{equation}
where we take $\ell((x, y), \theta)~\define~\1[f_\theta(x) = y]$ to be the 0/1 loss. In this definition we explicitly restrict the datasets to include elements of each group $\group \in \setgroups$, which is a technicality needed in order to avoid undefined behavior. 
\end{definition}

In all our experimental results, the worst-performing groups (the maximizers in \cref{eq:worst-group-gen}) are always the same on the training and test data. As long as this holds---the worst-performing group is the same on the train and test data---the generalization gap above simplifies to:
\begin{equation}\label{eq:worst-group-gen-as-gen-of-worst-group}
\begin{aligned}
    \wggap = 
    \E_{\substack{\dataset \sim \datagen^n \\ \xy \sim \dataset_{\group^*}}}[\ell(\xy, \train(\dataset)) \mid |\dataset_{\group^*}| > 0] - \E_{\substack{\dataset \sim \datagen^n \\ \xy \sim \datagen_{\group^*}}}[\ell(\xy, \train(\dataset))],
\end{aligned}
\end{equation}
where $\group^* \in \setgroups$ is the worst-performing group. In \cref{app:subgroup} we show that this simplified gap from \cref{eq:worst-group-gen-as-gen-of-worst-group} is bounded by TV stability.

Therefore, in practice the generalization gap in \cref{eq:worst-group-gen} offers a lower bound on the DG gap in \cref{eq:dg-variational-defn}. Using it as a proxy for DG gap follows the spirit of the estimation approach by \citet{nakkiran2020distributional} which proposes to estimate the DG gap by taking the maximum of empirical generalization gaps for a finite set of relevant test functions (here, per-group accuracies).

\paragraph{Other Approaches to Estimate the DG Gap.}
The generalization gap of worst-case group accuracy can be loose as a proxy.
Finding the worst-case test function is an object of study in the literature on \newterm{membership inference attacks}~\cite{shokri2017membership}, because DG and the accuracy of such attacks in their standard formalization are equivalent, as showed by~\citet{kulynych2022disparate}. In this work, we opt for a simpler and direct approach described above.

\subsection{Additional Details for \cref{sec:dg-in-practice}}
\label{app:dg-regularization-basic}

As mentioned in \cref{sec:algorithms}, many regularization methods can be used to improve different generalization gaps.
For example, \citet{sagawa2019distributionally} show that strong $\ell_2$ regularization helps with improving group-distributional generalization, and \citet{YY2020roblocallip} show that dropout helps with adversarial-robustness generalization.
However, these works do not have theoretical justification.

Our framework suggests a unifying reason why strong regularization
is helpful in distributional robustness: because it enforces DG.
Following this theoretically-inspired intuition, other regularization methods
beyond a combination of gradient noise and clipping (\dpsgd) can imply DG in practice.
We verify this hypothesis empirically.

\begin{figure*}[t!]
    \centering
    \subfigure[Differential privacy as regularization]{
      \includegraphics[width=.3\textwidth]{images/dpdg/dpdg_wgacc_celebA-nohead.pdf}}
    \subfigure[$\ell_2$ regularization]{
      \includegraphics[width=.3\textwidth]{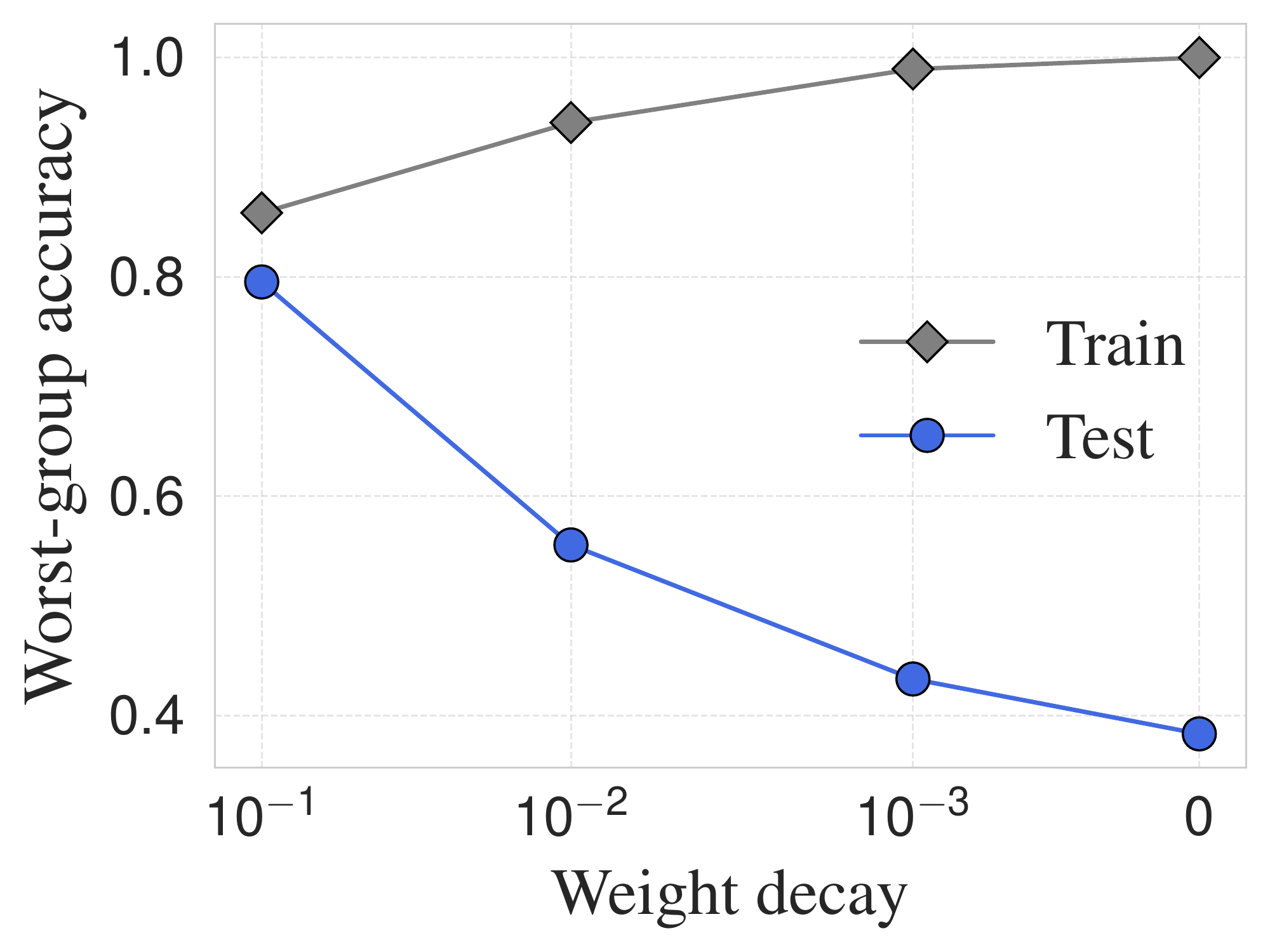}}
    \subfigure[Early-stopping as regularization]{
      \includegraphics[width=.3\textwidth]{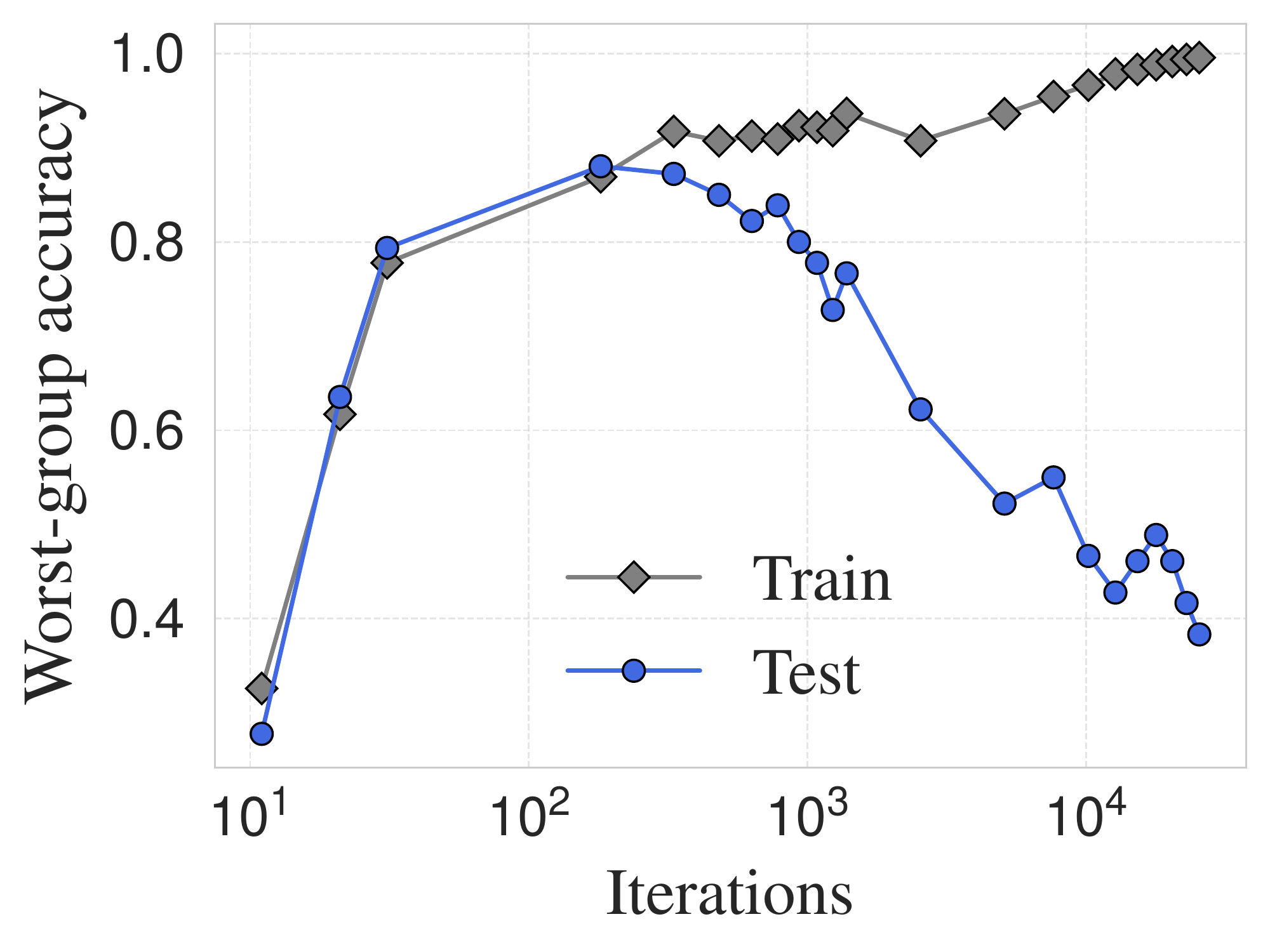}}

    \caption{{\bf Regularization induces DG.}
    The figure shows train/test worst-group accuracies as a function of regularization strength for SGD on CelebA,
    with different types of regularizers: differential privacy budget $\eps$, weight decay, and train time.
    For DP-SGD, $\epsilon=\infty$ represents standard SGD.
    For all types of regularizers, increasing the strength (left on x-axis) corresponds to a smaller generalization gap
    in worst-group accuracy.
    }
    \label{fig:dpdg_gap}
\end{figure*}

\paragraph{Privacy, $\ell_2$ Regularization, and Early Stopping.} In \cref{fig:dpdg_gap}, we train a neural network on CelebA using \dpsgd,
and decrease the ``regularization strength'' in several different ways:
by increasing privacy budget $\eps$ (\cref{fig:dpdg_gap}a),
decreasing the $\ell_2$ regularization (\cref{fig:dpdg_gap}b),
or increasing the number of training iterations (\cref{fig:dpdg_gap}c).\footnote{Train time can be considered a regularizer,
as its decrease induces stability (e.g. \citet{hardt2016train}).}
We then measure the gap in worst-group accuracy on train vs. test (\cref{app:empirical-dg}).
We observe that for all regularizers,
the gap between training and testing worst-group accuracy increases as
the regularization is weakened.

\paragraph{Investigating $\ell_2$ Regularization in Depth.}
In \cref{fig:other_dg_methods}, we show the training and testing worst-group accuracy with different strength of $\ell_2$ regularization and on different epochs (w/ and w/o $\ell_2$ regularization).
We have three observations:
(1) with properly tuned regularization parameter, the gap between training and testing worst-group accuracy can be narrowed,
(2) the gap can start widening in very early stage of training, and
(3) the testing worst-group accuracy can fluctuate largely, which highlights the importance of using validation set for early stopping in this task.

\begin{figure}[h!]
    \centering
    \subfigure[CelebA]{\includegraphics[width=.3\textwidth]{images/regularizations/l2reg_celebA.pdf}}
    \subfigure[UTKFace]{\includegraphics[width=.3\textwidth]{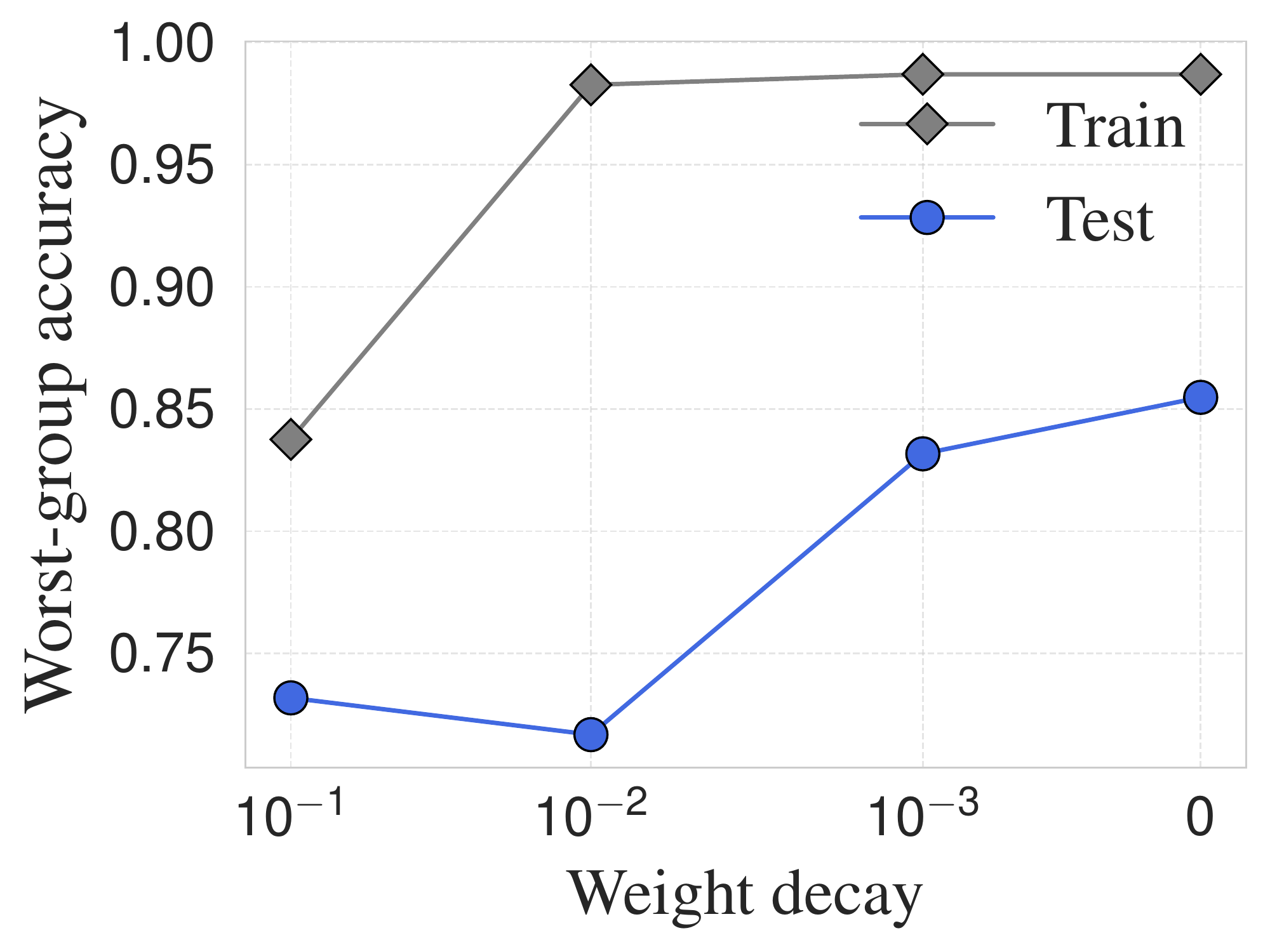}}
    \subfigure[iNat]{\includegraphics[width=.3\textwidth]{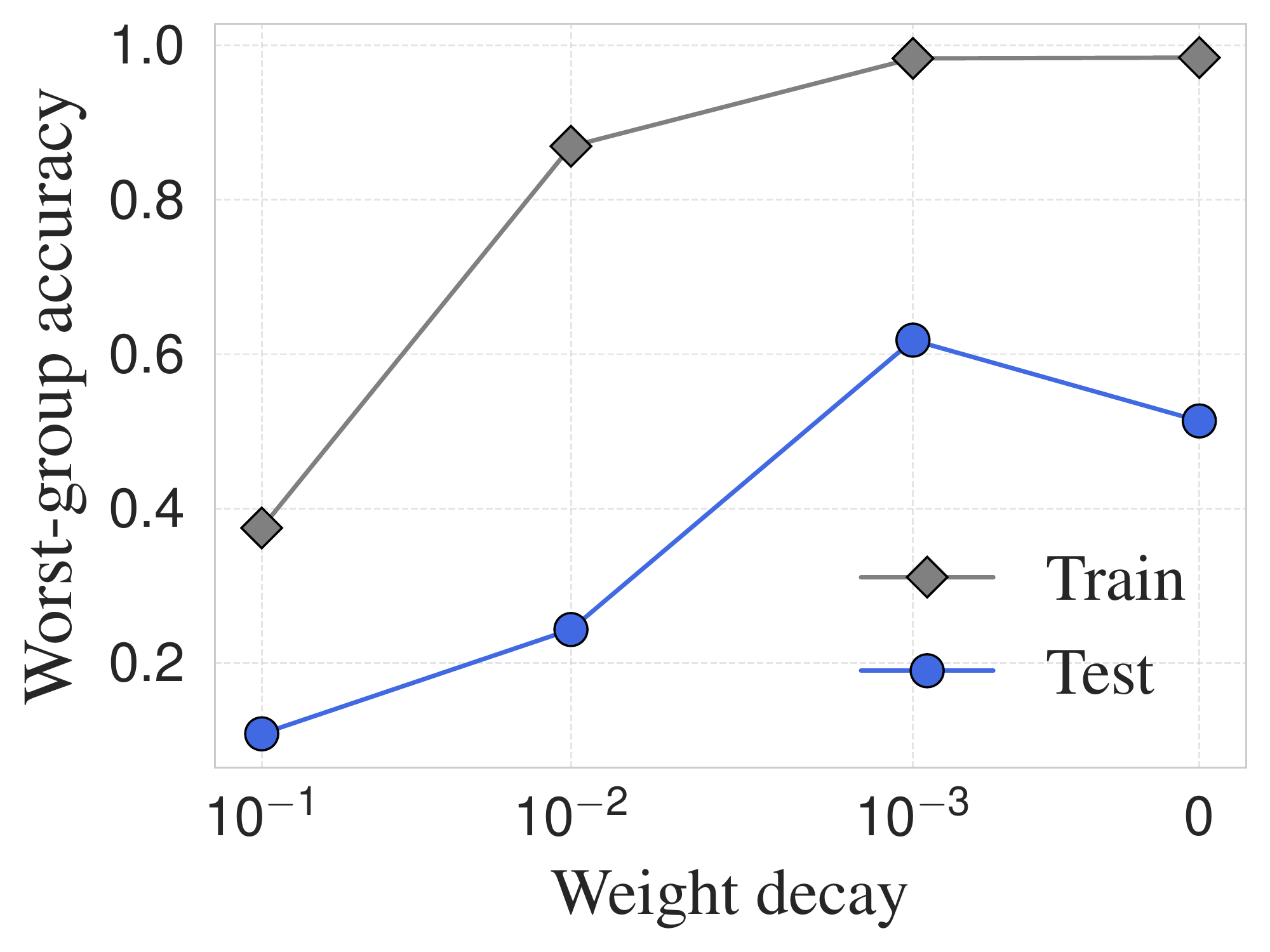}}
    
    \subfigure[CelebA (no reg.)]{
      \includegraphics[width=.3\textwidth]{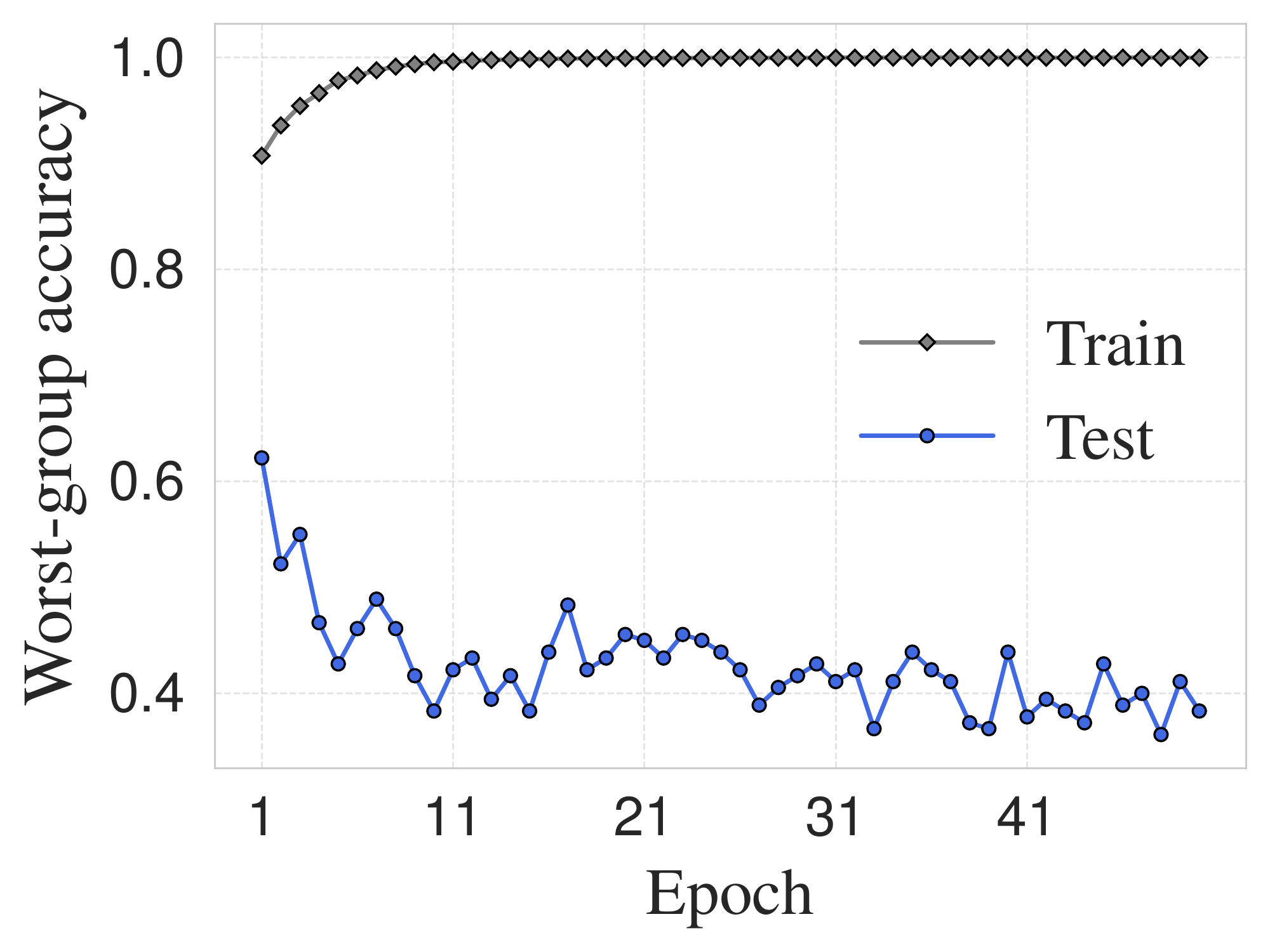}}
    \subfigure[UTKFace (no reg.)]{
      \includegraphics[width=.3\textwidth]{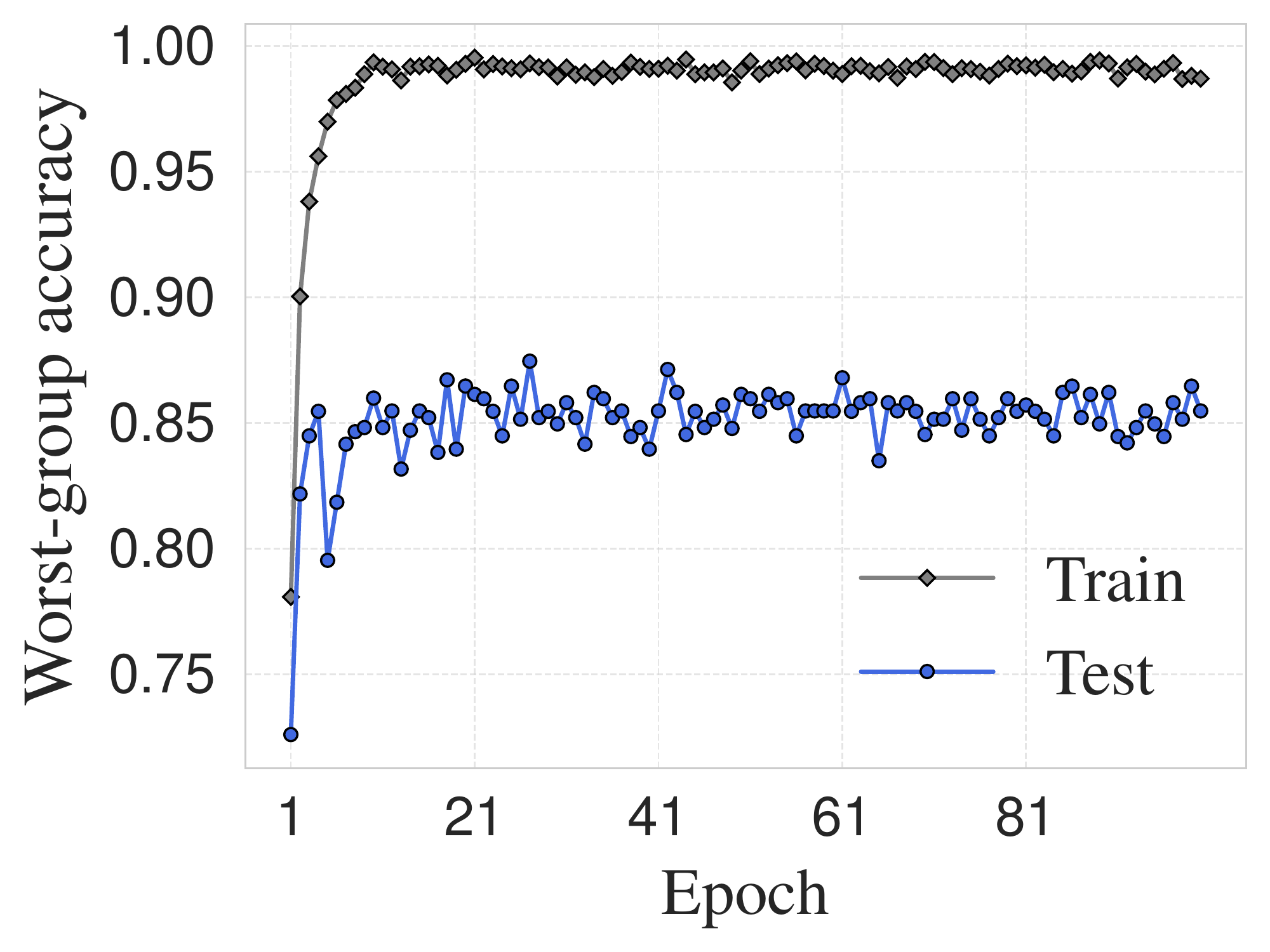}}
    \subfigure[iNat (no reg.)]{
      \includegraphics[width=.3\textwidth]{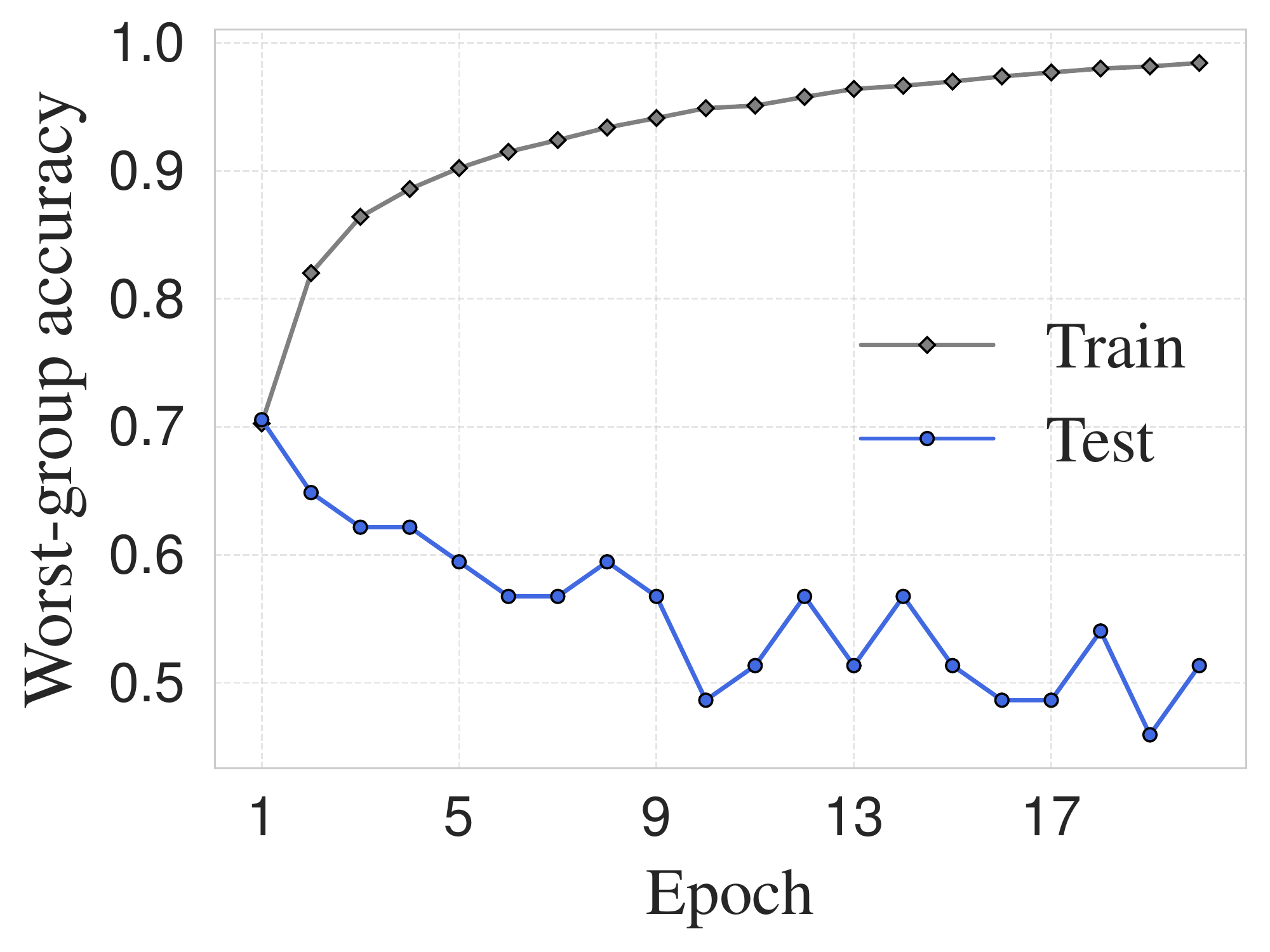}}
    
    \subfigure[CelebA ($\ell_2$ reg.)]{
      \includegraphics[width=.3\textwidth]{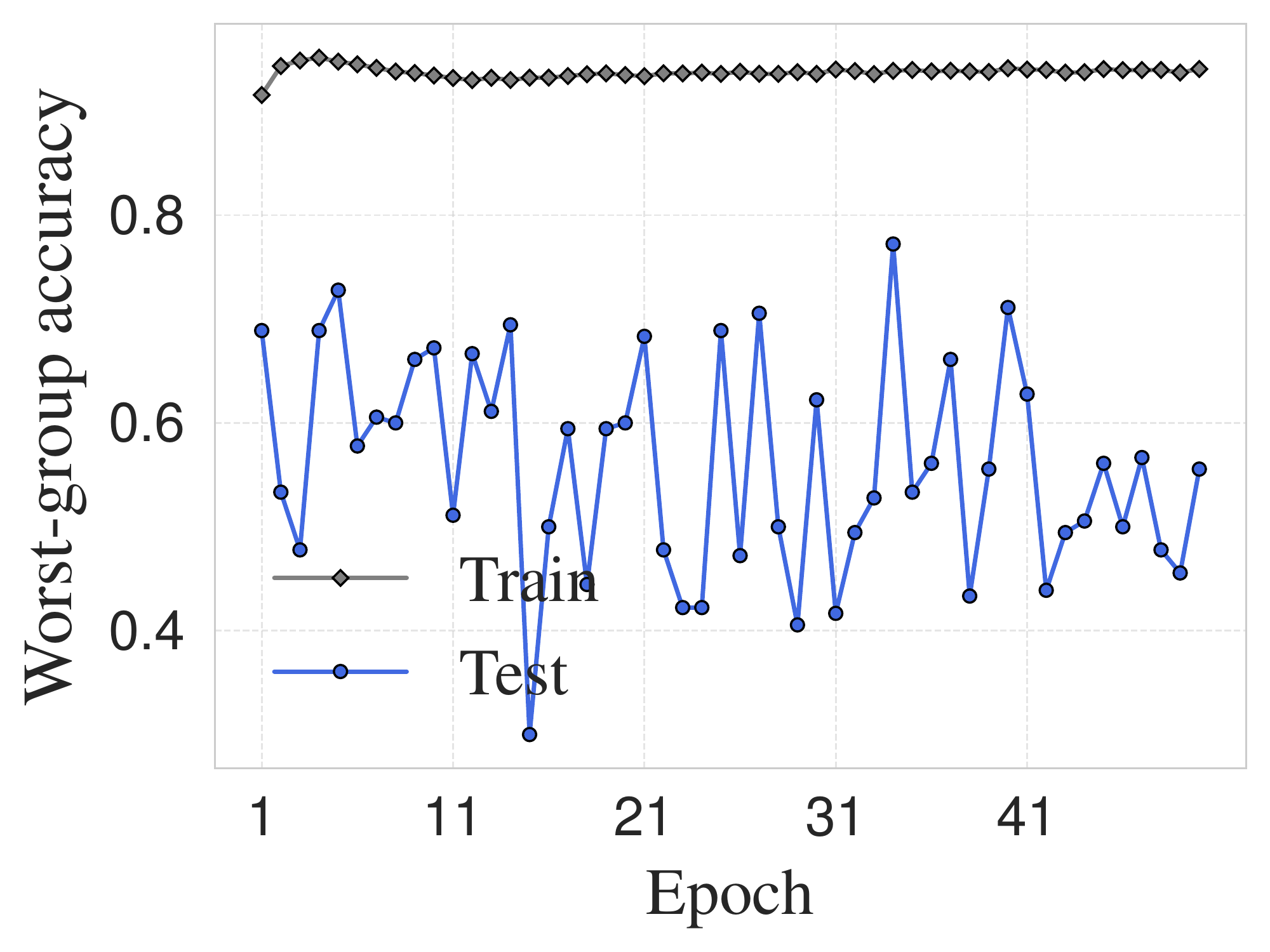}}
    \subfigure[UTKFace ($\ell_2$ reg.)]{
      \includegraphics[width=.3\textwidth]{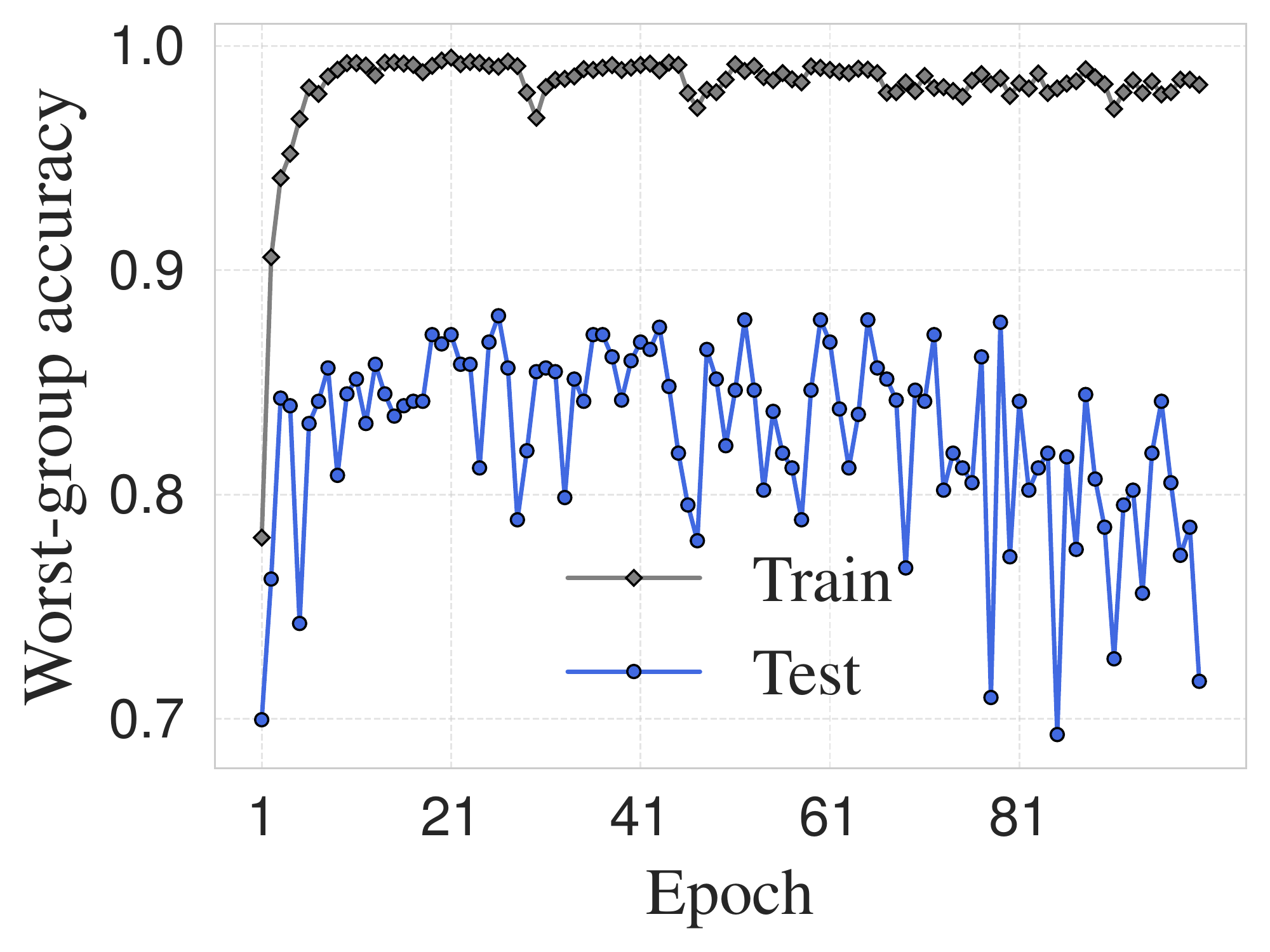}}
    \subfigure[iNat ($\ell_2$ reg.)]{
      \includegraphics[width=.3\textwidth]{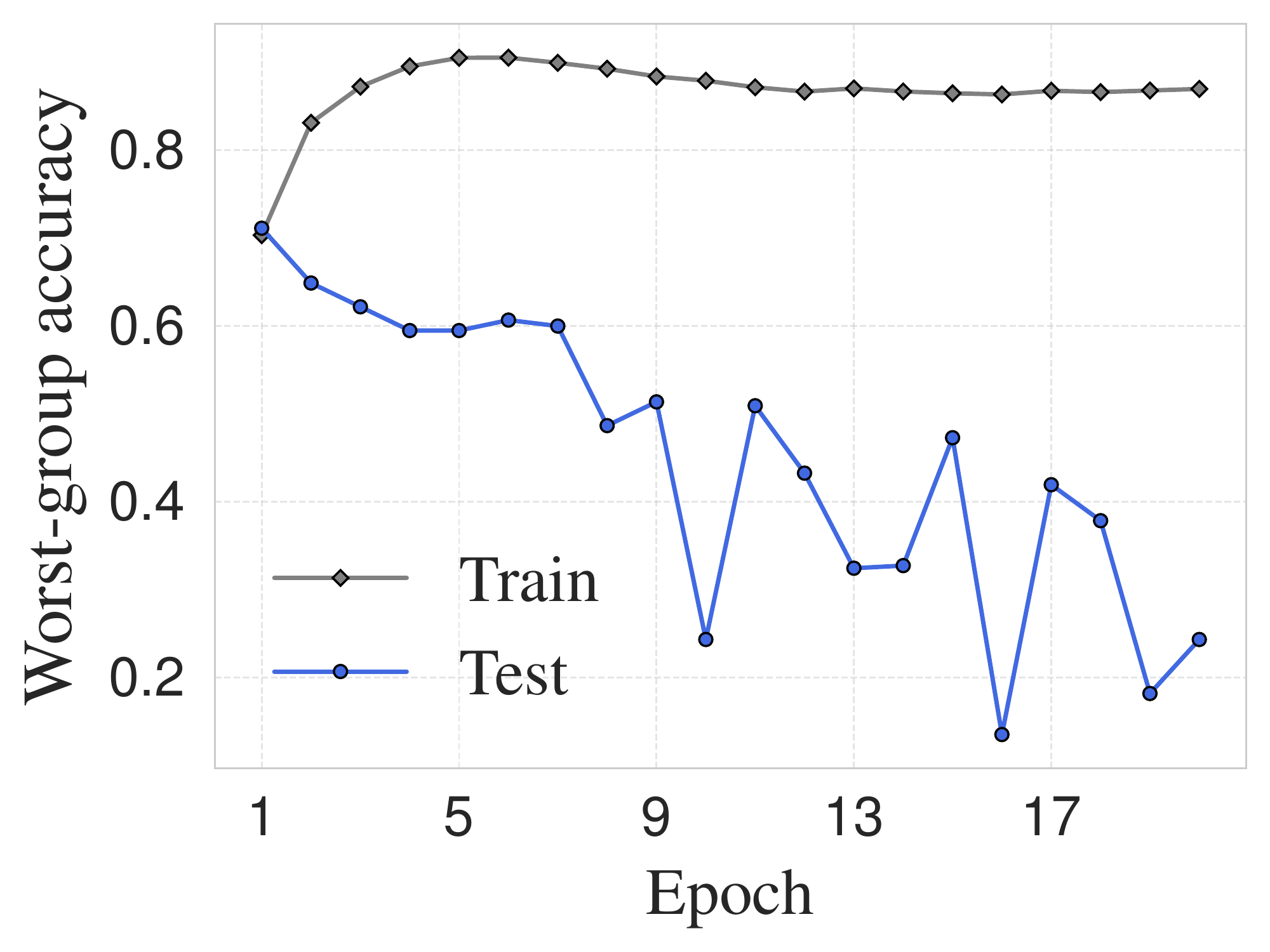}}
    
    \caption{
    We show the training and testing worst-group accuracy with different strength of $\ell_2$ regularization and on different epochs (w/ and w/o $\ell_2$ regularization).
    The network is trained with \iserm on CelebA, UTKFace, and iNat.
    For (a), (b), and (c), we show the result of the last epoch.
    For (g), (h), and (i), we set weight decay to $0.01$.
    }
    \label{fig:other_dg_methods}
\end{figure}

\subsection{Additional Details for \cref{sec:exp-disparate}}
\label{app:exp-disparate-extra}

\cref{fig:disparity-app} shows the accuracy disparity, test accuracy, and worst-group accuracy for CelebA, UTKFace, and iNat on \dpsgd and \dpis.

The reason that UTKFace has a similar disparity between \dpsgd and \dpis is likely because UTKFace has a relatively small difference in the number of training examples between the largest group and the smallest group.
In UTKFace, the majority group has around seven times more examples than in the minority group, whereas in CelebA, this difference is $52\times$.

\begin{figure}[h!]
    \centering
    \subfigure[CelebA]{\includegraphics[trim=0 0 0 42,clip,width=.3\textwidth]{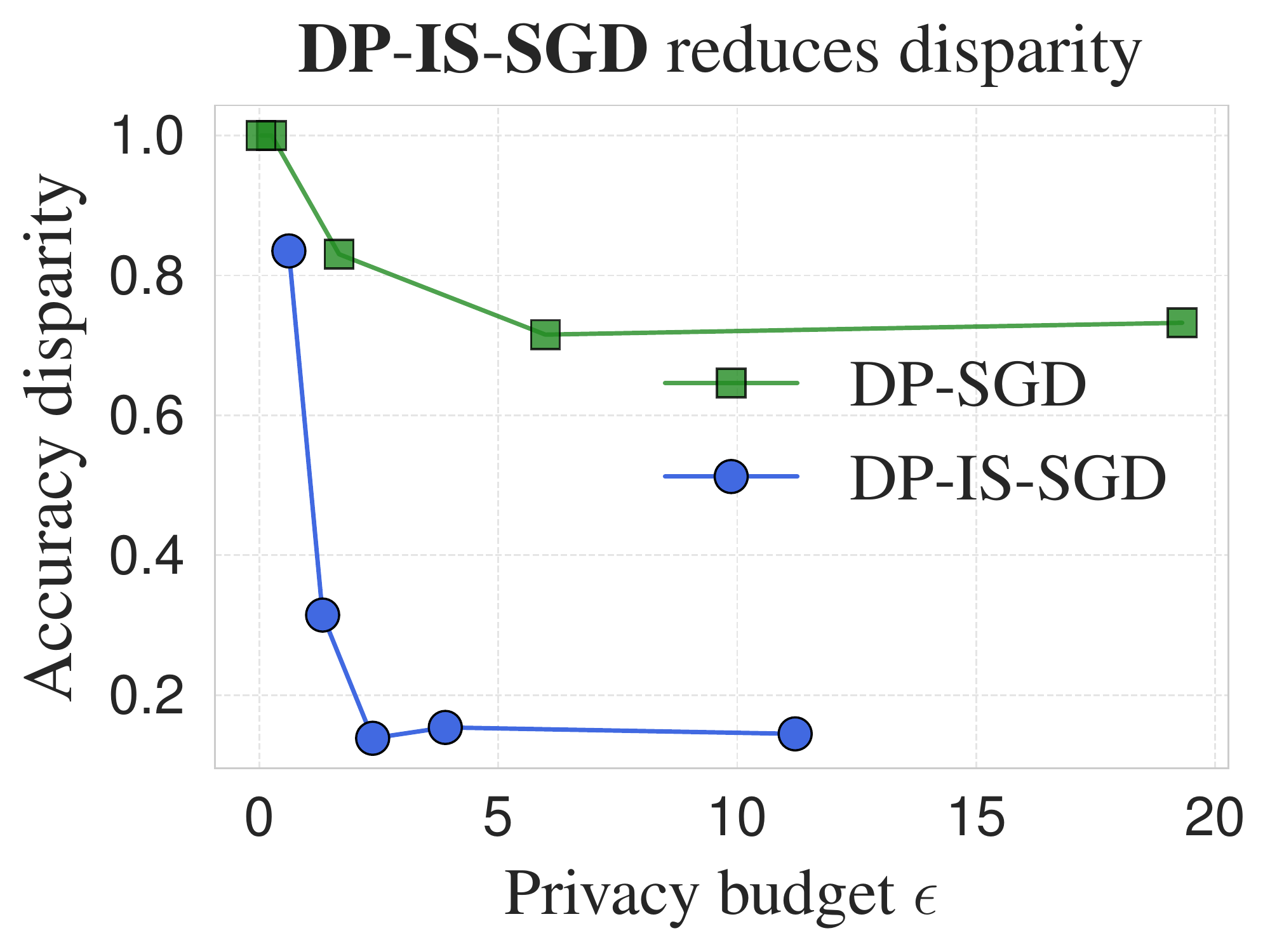}}
    \subfigure[UTKFace]{\includegraphics[trim=0 0 0 42,clip,width=.3\textwidth]{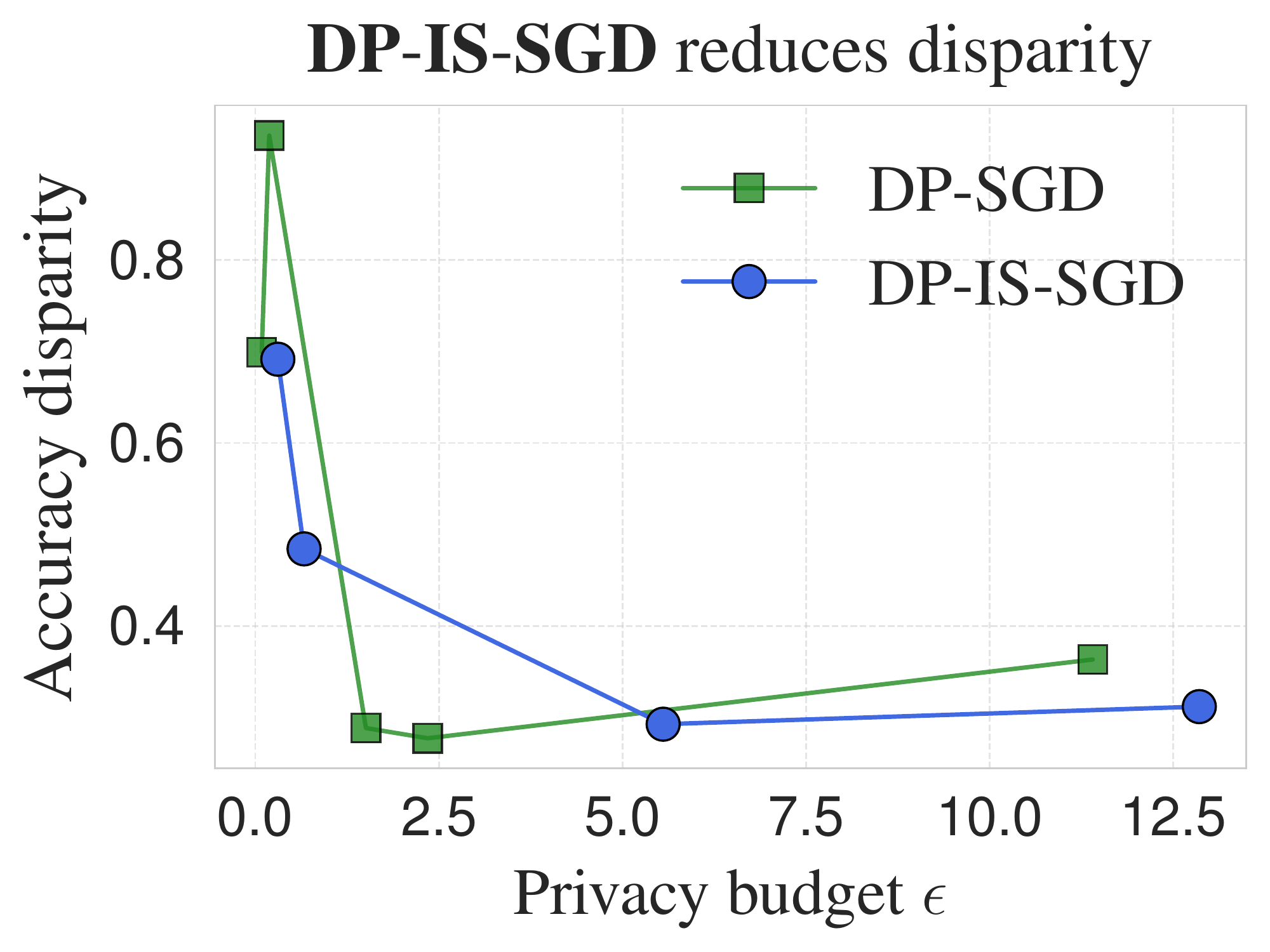}}
    \subfigure[iNat]{\includegraphics[trim=0 0 0 42,clip,width=.3\textwidth]{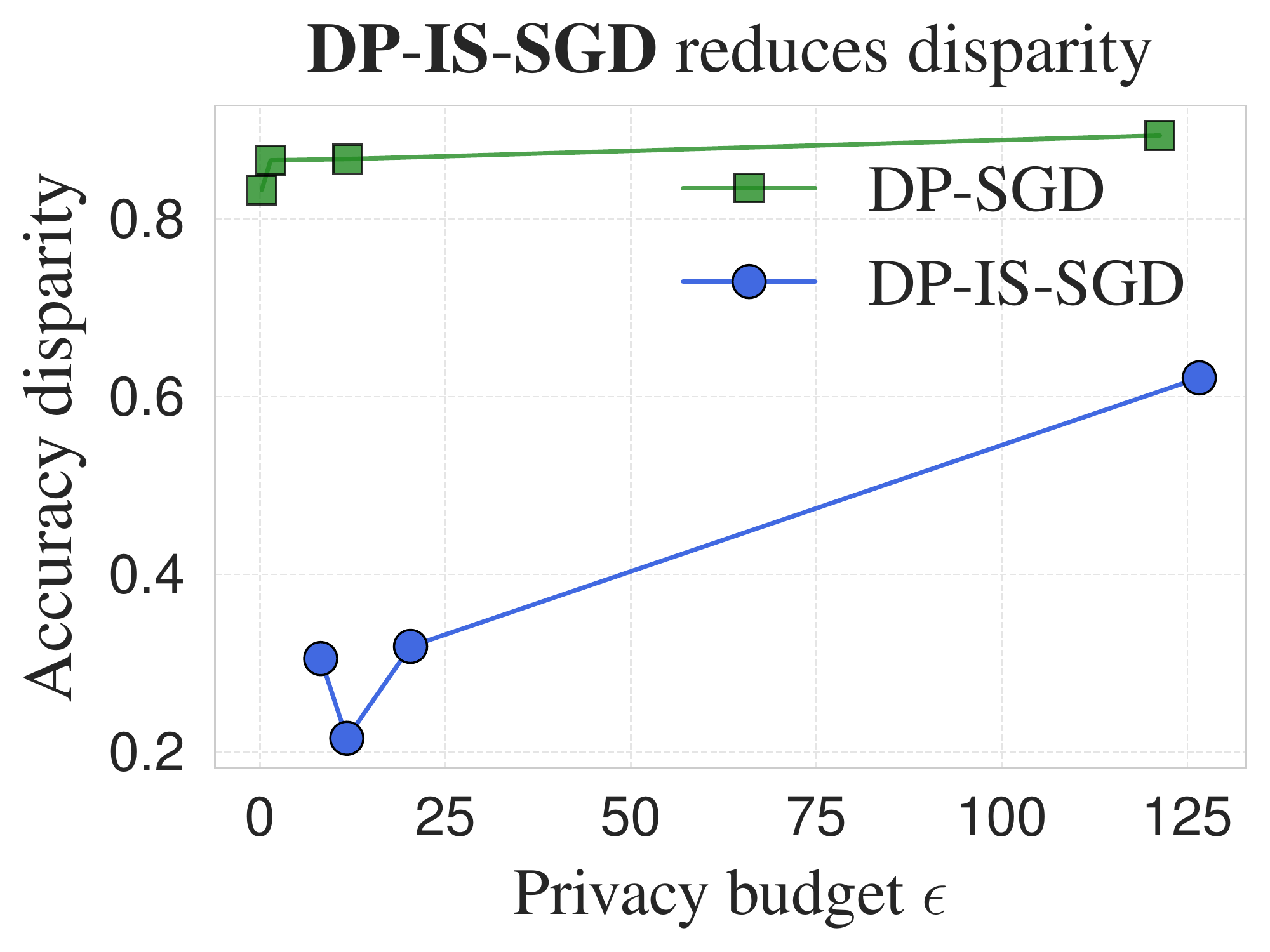}}
    
    \subfigure[CelebA]{\includegraphics[width=.3\textwidth]{images/dpdg/disparity_tstacc_celebA.pdf}}
    \subfigure[UTKFace]{\includegraphics[width=.3\textwidth]{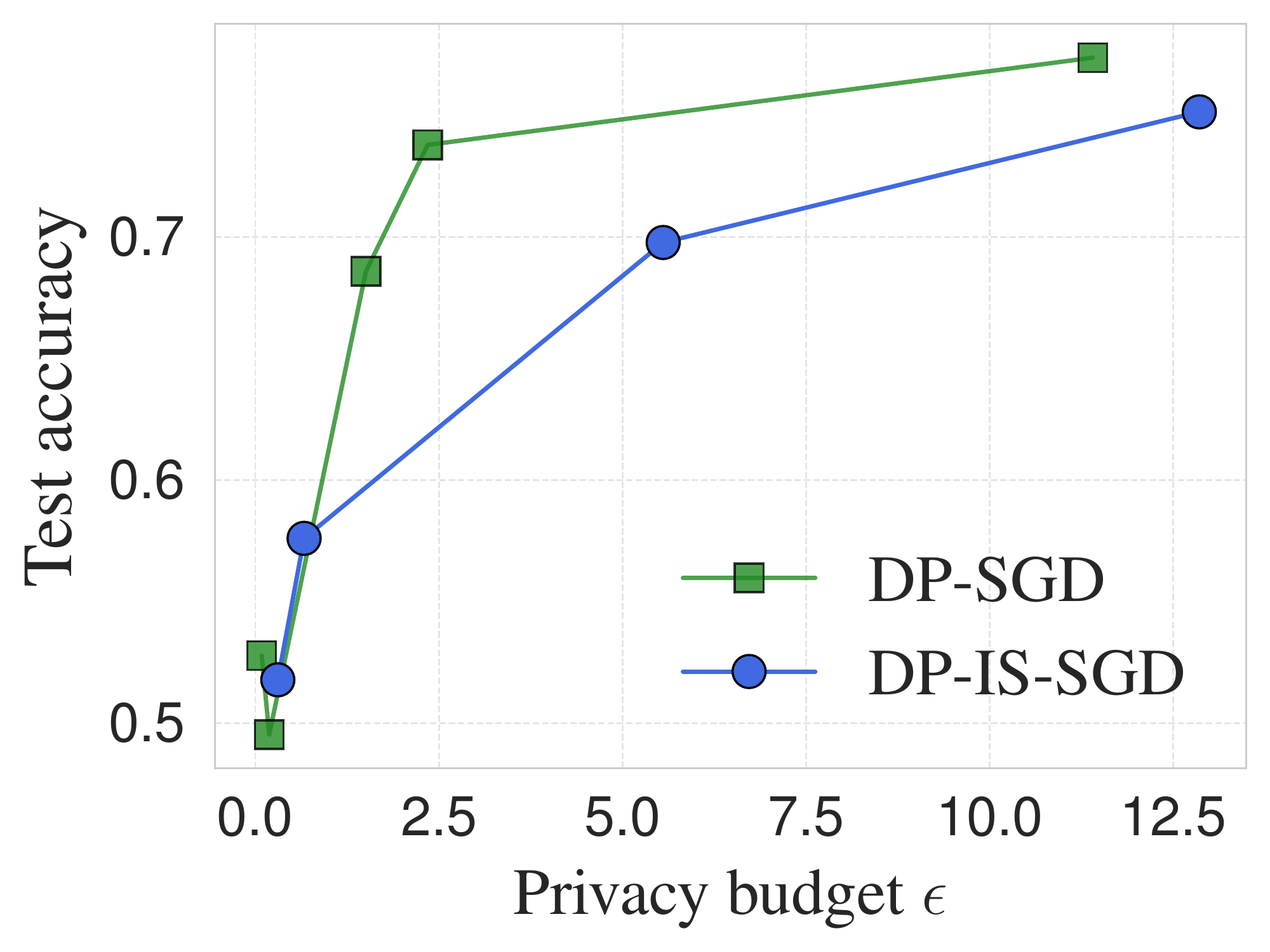}}
    \subfigure[iNat]{\includegraphics[width=.3\textwidth]{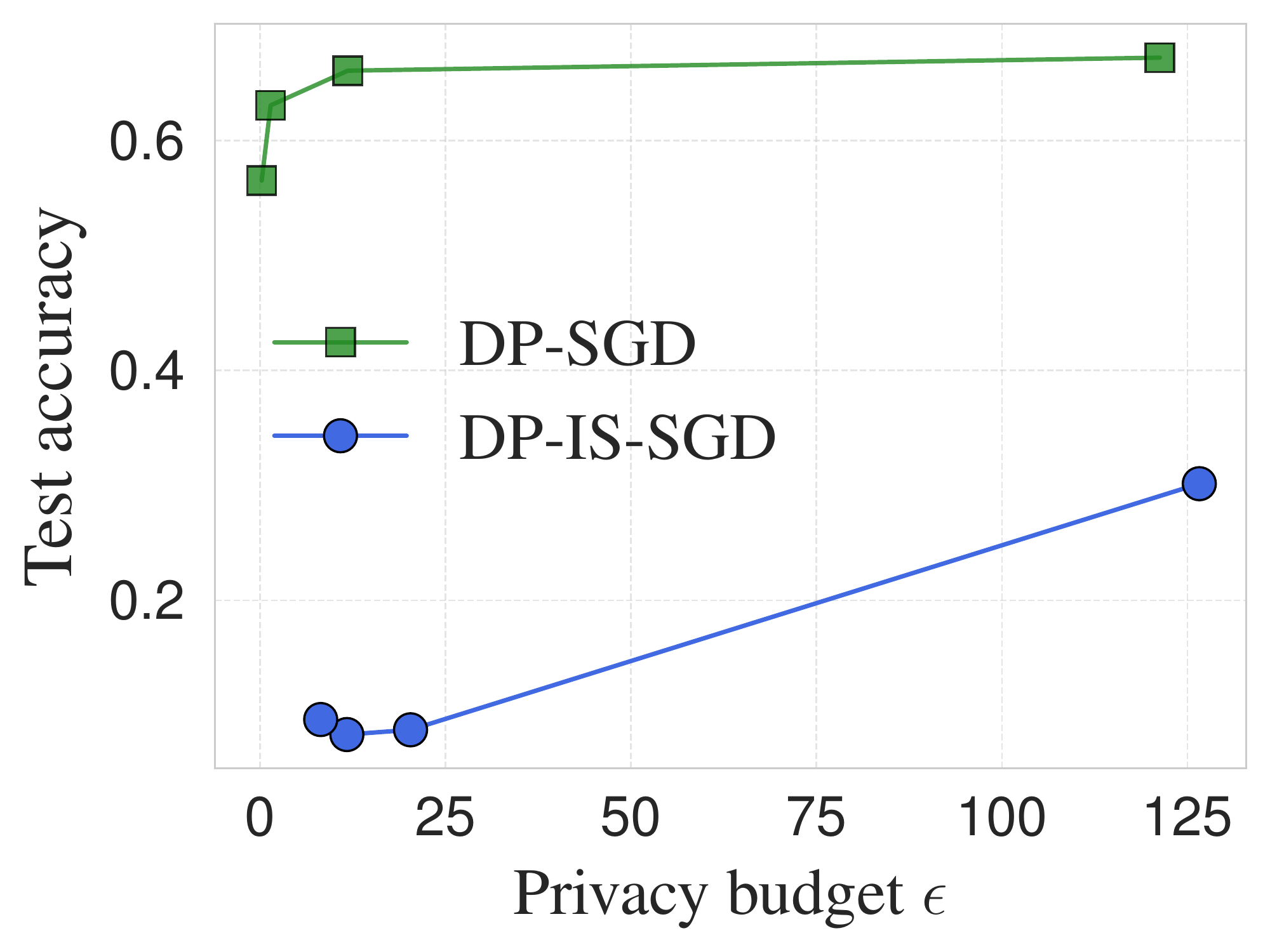}}
    
    \subfigure[CelebA]{\includegraphics[width=.3\textwidth]{images/dpdg/disparity_wgacc_celebA.pdf}}
    \subfigure[UTKFace]{\includegraphics[width=.3\textwidth]{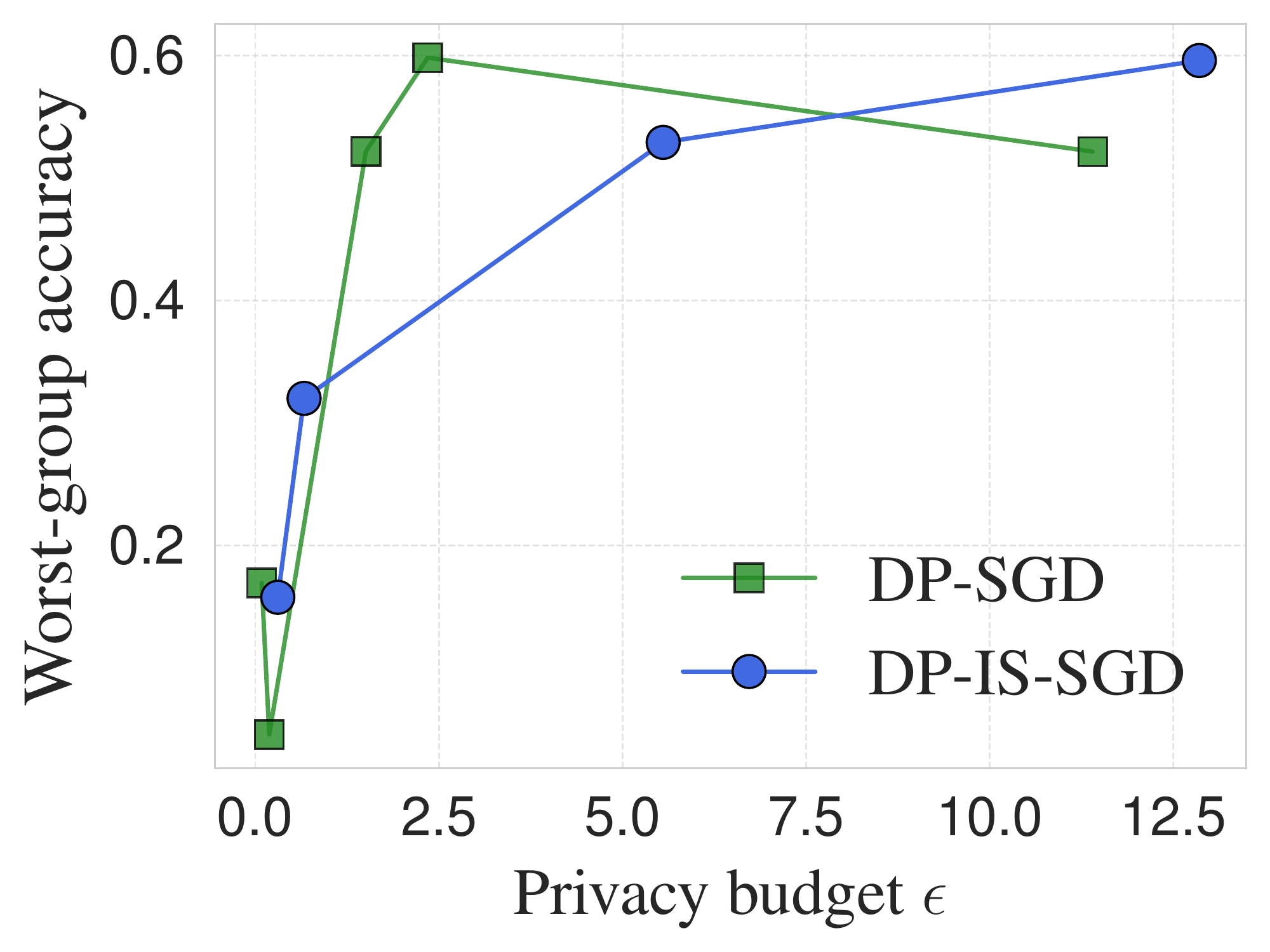}}
    \subfigure[iNat]{\includegraphics[width=.3\textwidth]{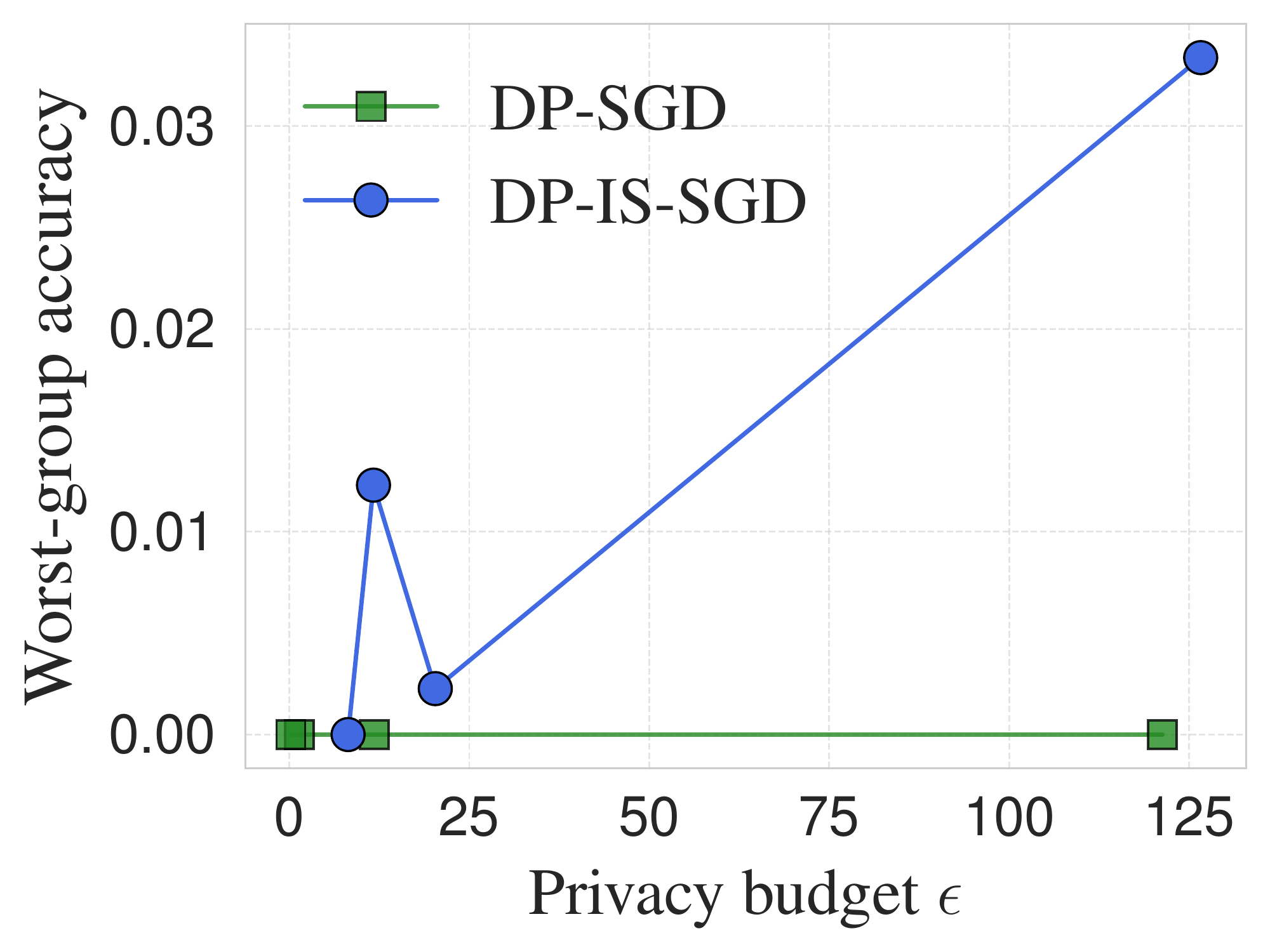}}
    
    \caption{The disparity (lower the better) and test accuracies of the models trained with \dpsgd and \dpiw on three datasets. If we care about privacy, \dpis improves disparate impact at most privacy budgets. For CelebA, we train the model for $30$ epochs. For UTKFace, we train for $100$ epochs. For iNat, we train for $20$ epochs. The GDP accountant is used to compute the privacy budget.}
    \label{fig:disparity-app}
\end{figure}

\paragraph{Comparison with DP-SGD-F~\cite{xu2021removing}.} We did not manage to obtain good performance from DP-SGD-F on CelebA, UTKFace, and iNat, possibly because of the different domain---images---than tabular data considered by \citet{xu2021removing}. To proceed with the comparison, we evaluate the algorithms on the census data---ADULT dataset~\cite{kohavi1996scaling} (see \cref{tab:adult_detail} for dataset statistics)---that \citet{xu2021removing} used in their work. As subgroups, we consider four intersectional groups composed of all possible values of the “sex” attribute and prediction class (an income higher/lower than 50k).

We show the results in \cref{tab:dpsgdf}.
For a comparable epsilon value (0.69 for DP-SGD-F, and 0.7 for our DP-IS-SGD), we see that our method has smaller accuracy disparity (Eq. 2) across the groups, although also lower overall accuracy.

\begin{table}[t]
    \caption{\textbf{\dpis has lower disparity \dpsgd-F on ADULT and better accuracy at the same privacy level. } The table shows the privacy level, maximum accuracy disparity across groups, and overall accuracy for all algorithms.}
    \label{tab:dpsgdf}
    \centering
    \begin{tabular}{lccc}
    \toprule
    Algorithm 	& $\epsilon$ & Accuracy disparity & Overall accuracy \\
    \midrule
    SGD                & -  & $0.660 \pm 0.000$ &  $0.836 \pm 0.000$ \\
    DP-SGD      & $0.6573$ & $0.852 \pm 0.005$ &  $0.802 \pm 0.001$ \\
    DP-SGD-F             & $0.6964$ & $0.657 \pm 0.023$ &  $0.832 \pm 0.001$ \\
    DP-IS-SGD  & $0.7059$ & $0.246 \pm 0.034$ &  $0.766 \pm 0.010$ \\
    \bottomrule
    \end{tabular}
\end{table}

\subsection{Additional Details for \cref{sec:exp-dro}}
\label{app:exp-dro-extra}

We compare different algorithms, including
\erm-$\ell_2$ and \iwerm-$\ell_2$ as baselines, and two other algorithms, \iserm-$\ell_2$~\cite{idrissi2022simple} and \gdro-$\ell_2$~\cite{sagawa2019distributionally} in terms of the group robustness.
We set the learning rate as $0.001$ for CelebA, UTKFace, and iNat, $0.00002$ for MNLI, and $0.00001$ for CivilComments.
We use the validation set to select the hyperparameters:
\begin{enumerate}
    \item For \erm-$\ell_2$, \iwerm-$\ell_2$, \iserm-$\ell_2$, and \gdro-$\ell_2$, we select the weight decay from $0.0001$, $0.01$, $0.1$, and $1.0$.
     \item For \dpis, we fix the gradient clipping to $1.0$ (except for iNat, where we set the value to $10.0$ as $1.0$ does not converge). We select the noise parameter from $1.0$, $0.1$, $0.01$, $0.001$ on CelebA and UTKFace, select the noise parameter from $0.0000001$, $0.000001$, $0.00001$, and $0.0001$ on iNat and select the noise parameter from $0.01$ and $0.001$ on CivilComments and MNLI.
    \item For \noisyiwerm, \noisyiserm, and \noisygdro, we select the standard deviation of the random noise from $0.001$, $0.01$, $0.1$, and $1.0$ on CelebA, UTKFace, and iNat, and we select standard deviation of the random noise from $0.00001$, $0.0001$, and  $0.001$ on CivilComments and MNLI.
\end{enumerate}

\paragraph{Statistical Concerns.}
Although our results appear to be comparable to or better than SOTA, we caution readers about the exact ordering of methods due to high estimation variance:
these benchmarks have small validation and test sets (e.g., CelebA has $182$ validation examples), 
and so hyperparameter tuning is subject to both overfitting and estimation error.
For example, 
we observe validation accuracies which differ from their test accuracies by up to $5\%$
in our experiments.
We attempt to mitigate this using three random train/val/test splits on CelebA, and avoid large hyperparameter sweeps\footnote{For example, we do not tune the ``group adjustments'' parameter for gDRO, using the default from~\citet{koh2021wilds} instead.}, but this is not done
in prior work.

\subsection{Additional Details for \cref{sec:exp-rob-overfitting}}
\label{app:exp-rob-overfitting-extra}
We use the CIFAR-10 dataset~\cite{krizhevsky2009learning}, and ResNet-18~\cite{he2016deep} as the network architecture. We train the model to be robust against $L_\infty$ perturbations of at most $\gamma = 8/255$ bound, which is a standard setup for adversarial training on this dataset. We vary $\sigma$ (noise parameter) from 0.0 (regular adversarial training without gradient noise) to 0.01. In addition, we compare the performance of noisy gradient to \textit{adversarial training with early stopping}---a simple but effective approach for mitigating  overfitting in adversarial training~\citep{rice2020overfitting}.

In this experiment, we measure robust accuracy and its respective generalization gap, thus setting $\ell((x, y), \theta)~\define~\1[f_\theta(x) = y]$ to be the 0/1 loss. 

\section{Author Contributions}

BK proposed to leverage the connection between differential privacy and distribution generalization,
and led the theory developments and algorithm analysis.
Yao-Yuan Yang led the experiments, designed the experimental settings, and implemented the algorithms. Yaodong Yu contributed to the experiments, especially in the NLP settings, and helped benchmark baselines. The noisy-gradient algorithm is proposed by Yao-Yuan Yang and Yaodong Yu.
JB contributed to the theory, in particular
the algorithmic analysis, initial proof of the tightness of the bound, and formal connections to robust generalization and calibration.
PN organized the team and managed the project,
contributing in parts to the theory and experiments.
All authors participated in framing the results and writing the paper.

\end{document}